\documentclass[final,12pt]{clear2026} %
\usepackage{hyperref}
\usepackage{url}
\usepackage{xcolor}
\usepackage{amsmath, amssymb}
\usepackage{amssymb}
\usepackage{tikz}
\usetikzlibrary{arrows.meta, positioning}
\usepackage{graphicx} %
\usepackage{wrapfig}  %
\usepackage{lipsum}   %
\usepackage{tcolorbox}
\usepackage{float}
\usepackage{fontawesome5}
\usepackage{subcaption}
\usepackage{multicol}
\usepackage{subcaption}
\usepackage{lineno}
\usepackage{bm}
\usepackage{etoc}
\usepackage{lettrine}
\usepackage{enumitem}

\definecolor{oxfordblue}{RGB}{0, 33, 71}
\definecolor{oxfordred}{RGB}{255, 90, 95}
\definecolor{oxfordgrey}{RGB}{143, 145, 162}
\definecolor{oxfordlight}{RGB}{225, 240, 255}
\definecolor{oxfordcream}{RGB}{210, 200, 185}
\definecolor{oxfordmist}{RGB}{235, 238, 245}
\definecolor{oxfordaqua}{RGB}{210, 245, 250}
\definecolor{black-cherry}{RGB}{87, 0, 0}
\definecolor{tea-green}{RGB}{215, 242, 186}
\definecolor{jungle-green}{RGB}{38, 169, 108}

\usepackage{hyperref}
\newif\ifdraft
\draftfalse  %

\usepackage[dvipsnames]{xcolor}
\newcommand{\new}[1]{#1}

\usepackage{thmtools}
\usepackage{thm-restate}

\newcommand{\defineusercomment}[3]{%
  \ifdraft
    \expandafter\newcommand\csname #1\endcsname[1]{%
      \begin{footnotesize}%
        \textcolor{#2}{\textit{\textbf{#3:} ##1}}%
      \end{footnotesize}%
    }%
  \else
    \expandafter\newcommand\csname #1\endcsname[1]{}%
  \fi
}

\defineusercomment{mwb}{red}{Markus}
\defineusercomment{anson}{magenta}{Anson}
\defineusercomment{joe}{blue}{Joe}
\defineusercomment{alex}{teal}{Alex}
\defineusercomment{hip}{violet}{Ingmar}
\defineusercomment{frederik}{oxfordgrey}{Frederik}
\defineusercomment{bradley}{oxfordred}{Bradley}

\newtheorem{assumption}{Assumption}
\newtheorem{cdefinition}{Definition}
\newtheorem{clemma}{Lemma}
\newtheorem{ctheorem}{Theorem}

\newcommand{\bt}{\boldsymbol{\theta}}
\newcommand{\bht}{\hat{\boldsymbol{\theta}}}
\newcommand{\gG}{\mathcal{G}}
\newcommand{\gGh}{\hat{\mathcal{G}}}
\newcommand{\btheta}{\boldsymbol{\theta}}
\newcommand{\bx}{\boldsymbol{x}}
\renewcommand{\bf}{\bm{f}}

\newcommand{\bh}{\bm{h}}
\newcommand{\bv}{\bm{v}}
\newcommand{\bC}{\bm{C}}
\newcommand{\btx}{\boldsymbol{x}}
\newcommand{\btau}{\boldsymbol{\tau}}

\newcommand{\ft}{\bf_{\boldsymbol{\theta}}}
\newcommand{\mg}{\mathcal{G}}

\usepackage{xcolor}
\DeclareRobustCommand{\rv}{{\bm{v}^{-1}}}

\usepackage[table]{xcolor} %
\usepackage{colortbl}      %

\makeatletter
\newcommand{\github}[1]{%
   \href{#1}{\textcolor{gray}{\faGithubSquare}}%
}
\makeatother

\usepackage{totcount}
\newtotcounter{citnum} %
\def\oldbibitem{} \let\oldbibitem=\bibitem
\def\bibitem{\stepcounter{citnum}\oldbibitem}

\title[Disentangling Dynamical Systems]{Disentangling Dynamical Systems: Causal Representation Learning Meets Local Sparse Attention}
\usepackage{times}

\clearauthor{%
 \Name{Markus W. Baumgartner} \Email{markusb@robots.ox.ac.uk}\\
 \Name{Anson Lei} \Email{anson@robots.ox.ac.uk}\\
 \Name{Joe Watson} \Email{joewatson@robots.ox.ac.uk}\\
 \Name{Ingmar Posner} \Email{ingmar@robots.ox.ac.uk} \\
 \addr Applied Artificial Intelligence Lab, Oxford Robotics Institute, Oxford, UK
}

\begin{document}
\maketitle
\vspace{-0.1cm}
\begin{abstract}%
Parametric system identification methods estimate the parameters of explicitly defined physical systems from data. Yet, they remain constrained by the need to provide an explicit function space, typically through a predefined library of candidate functions chosen via available domain knowledge.
In contrast, deep learning can demonstrably model systems of broad complexity with high fidelity, but black-box function approximation typically fails to yield explicit descriptive or disentangled representations revealing the structure of a system.
We develop a novel identifiability theorem, leveraging causal representation learning, to uncover disentangled representations of system parameters without structural assumptions.
We derive a graphical criterion specifying when system parameters can be uniquely disentangled from raw trajectory data, up to permutation and diffeomorphism.
Crucially, our analysis demonstrates that global causal structures provide a lower bound on the disentanglement guarantees achievable when considering local state-dependent causal structures.
We instantiate system parameter identification as a variational inference problem, leveraging a sparsity-regularised transformer to uncover state-dependent causal structures.
We empirically validate our approach across four synthetic domains, demonstrating its ability to recover highly disentangled representations that baselines fail to recover. Corroborating our theoretical analysis, our results confirm that enforcing local causal structure is often necessary for full identifiability.

\end{abstract}%
\begin{keywords}%
  Causal representation learning, dynamical systems identification, sparse attention
\end{keywords}%
\section{Introduction}
Parametric system identification methods aim to estimate the \textit{parameters} that govern a complex system by observing the system's behaviour. 
While these methods have proven successful in finding interpretable and physically meaningful parameters that describe the underlying dynamical systems, classical methods often require substantial domain knowledge, such as
a predetermined function space \citep{sindy,sindy-shred}.
In stark contrast, deep learning methods adopt a domain-agnostic and fully data-driven approach to modelling dynamics.
By bypassing explicit parameter estimation and directly predicting system trajectories, learning-based methods excel in highly complex settings where providing domain knowledge is prohibitively difficult, such as climate systems \citep{weather-fgn,graphcast}, biological systems \citep{AlphaFold3}, and interactive environments like video games and autonomous driving \citep{Dreamer,hu2023gaia}.
However, these methods represent the dynamics of the underlying system \emph{implicitly} within model weights, and therefore lack the insight and mechanistic understanding afforded by system identification techniques.
Motivated by the complementary strengths of these paradigms, this work investigates, both theoretically and empirically, \textit{how} and \textit{when} the underlying parameters of a physical system can be learned from observations without requiring prior functional knowledge.

The key intuition motivating our investigation is that physically meaningful parameters typically exert \textit{sparse} causal influence on system states: a given parameter does not influence all system components at all times. 
Consider a system comprised of colliding objects. System parameters such as the masses and elasticity coefficients influence only the specific objects involved in a collision, and only at the moment of impact.
We argue that this sparsity of influence serves as a practical model-selection criterion for discovering meaningful parameterisations of dynamical systems.
We formalise this intuition within the framework of \textit{parameter estimation as representation learning} \citep{yao2024marrying}, where identifying physically meaningful parametrisations from system trajectories is analogous to learning disentangled representations from observations. 
This allows us to leverage causal representation learning theory to characterise the conditions under which parameters can be disentangled. 
Specifically, we extend prior results on mechanism sparsity \citep{lachapelle2024nonparametric} to the context of dynamical systems and derive novel identifiability results that highlight the role of sparsity regularisation and elucidate the structural properties necessary for disentanglement.

To operationalise and empirically verify our theory, we propose a practical algorithm for learning parameters from observed trajectories.
We instantiate our learning method as a sparsity regularised, VAE-style model that encodes observed trajectories into latent parameters and then decodes them, together with the initial state of the system, to reconstruct future trajectories (figure \ref{fig:intro_figure}).
A key aspect of our theory is that parameter identifiability can be significantly improved if the decoder model captures \textit{local}, \textit{state-dependent} causal structures.
To reflect such structures, we employ a recently proposed sparsity-regularised transformer architecture as the decoder, designed to learn local causal dependencies.
Empirically, we validate our approach across four synthetic domains and demonstrate that the proposed method successfully isolates the underlying system parameters where baseline representation methods fail.
In sum, our core contributions are threefold:

\begin{figure}[t]
\vspace{-0.6cm}
    \centering
    \includegraphics[width=0.88\textwidth]{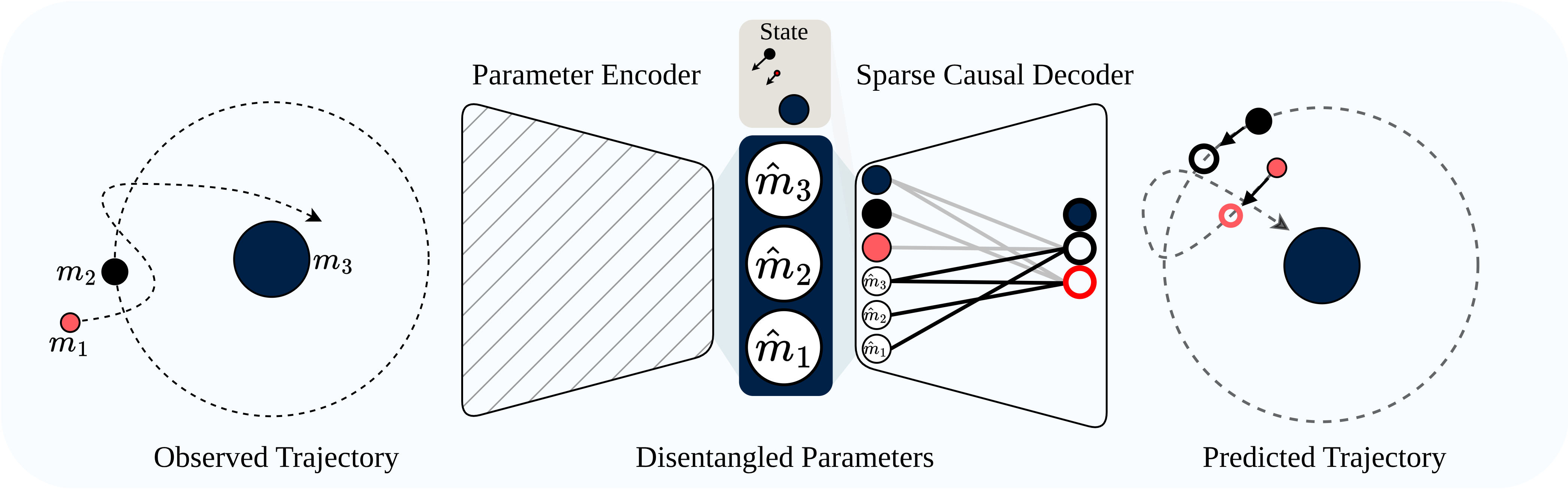}
    \caption{High-level overview of the developed theory. An observed trajectory (left) is encoded into a vector of latent system parameters (marked in dark blue). The developed theory shows that enforcing sparse causal relations between parameters and system components in the decoder (which performs one-step prediction) provably disentangles the system parameter representation.}
    \label{fig:intro_figure}
    \vspace{-0.3cm}
\end{figure}
\begin{itemize}[noitemsep]
    \item A novel theorem establishing identifiability of system parameters for local causal predictors, culminating in a graphical criterion.
    \item A practical implementation using a recent sparse transformer architecture for unsupervised local causal discovery and representation learning.
    \item An empirical validation across multiple domains, including object-centric and discontinuous systems, which corroborates the theory's predictions.
\end{itemize}

\clearpage
\section{Identifiability of System Parameters}
\label{section:theory}
In this section, we present the main identifiability theorem of our work and highlight its practical implications.
Our general approach extends prior theoretical results \citep{lachapelle2024nonparametric}, which uses mechanism sparsity to induce disentanglement, to the setting of learning parameterisations of dynamical systems from observed trajectories. In particular, we consider system parameters that change between the trajectories within a dataset. For example, consider a dataset containing planetary orbits. If each trajectory consists of a single instantiation of the system given an initial state and a particular set of parameters (e.g. planetary and solar masses), then the mass of each planet or sun strongly influences the evolution of the others.
The theory we are establishing in this section defines the conditions on the underlying system and the learning method under which mechanism sparsity leads to disentangled system parameters.
\subsection{Problem Setting}
\label{section:model_structure}
We consider nonlinear deterministic Markovian dynamical systems, which include the orbital example presented above, as well as the majority of natural systems that are not externally influenced. 
\begin{align}
\label{equation:dynamical_systems_form}
\boldsymbol{x}_{t+1}=\bf(\bx_t, \btheta) = \bf_{\theta}(\boldsymbol{x}_t),
\quad \boldsymbol{\theta} \sim \bm{p}_{\boldsymbol{\theta}}(\cdot),
\quad \btheta \in \Theta,
\quad \boldsymbol{x} \in \mathcal{X} , \quad \boldsymbol{x}_0 \in \mathcal{X}_0 \subseteq \mathcal{X},
\\
\boldsymbol{\tau}=\bh(\bx_0, \btheta)=[\btx_0,\ft(\btx_0), (\ft \circ \ft)(\btx_0),...,\ft^T (\btx_0)], \quad
\bx_0, \btheta = \bh^{-1}(\btau),
\label{equation:dynamics_model_form}
\end{align}
where $\boldsymbol{x}_t \in \mathcal{X} \subseteq \mathbb{R}^d$ denotes the state of a system at time $t$, $\bf_{\boldsymbol{\theta}}$ denotes the Markovian dynamics model, characterised by a set of parameters $\boldsymbol{\theta}\in \Theta$, where $\Theta$ has non-zero measure.
Trajectories are initialised from an initial distribution $\boldsymbol{x}_0 \in \mathcal{X}_0$. The set of systems $\bf_{\boldsymbol{\theta}} \in \mathcal{F}$ defines an environment. For instance, $\mathcal{F}$ might represent the family of orbital mechanics. 
Then $\bf_{\boldsymbol{\theta}}$ might be the instantiation of a system where a planet with mass $\theta_0$ is orbiting a sun with mass $\theta_1$. 
We assume that the parameters sparsely affect the state distribution via a causal graph $\mathcal{G}$ (assumption \ref{assumption:markov}). For example, the planet's acceleration is affected by the mass of the sun $\theta_1$ but not its own mass, and vice versa. Each trajectory is defined via a bijective trajectory decoder, $\bh(\cdot)$, which, given a parameter and starting state, auto-regressively applies the Markovian dynamics model.
Likewise, a parameter encoder, $\bh^{-1}(\cdot)$, is defined as the inverse process (equation \ref{equation:dynamics_model_form}).
This encoder-decoder setup is visualised in Figure \ref{fig:intro_figure} and forms a latent information bottleneck that forces learnt approximations to isolate the dynamical factors of variation. Let $\mathcal{S}=(\Theta,\mathcal{X}_0,\mathcal{X},\mathcal{F},\mathcal{G})$ denote the ground truth model of a dynamical family, and $\hat{\mathcal{S}}=(\hat{\Theta},\mathcal{X}_0,\mathcal{X},\hat{\mathcal{F}},\hat{\mathcal{G}})$ denote the learnt latent parameter model. For uniqueness and without loss of generality, we enforce assumptions \ref{assumption:existence} and \ref{assumption:obs_eq}, which in combination also require the bijectivity of $\bh(\cdot)$ and $\bh^{-1}(\cdot)$ with respect to their image. 
\begin{assumption}[Existence and Parameter Influence \citep{yao2024marrying}]
    \label{assumption:existence}
    For every $\bx_0 \in \mathcal{X}_0, \: \btheta\in\Theta$ there exists a unique trajectory, $\boldsymbol{\tau}$, over horizon T, satisfying $\boldsymbol{x}_{t+1}=
    \bf_\theta (\boldsymbol{x}_t)$. Equivalent to requiring $\bh (\bx_0, \bt)$ to be injective.
\end{assumption}
\begin{assumption}[Observational Equivalence]
\label{assumption:obs_eq}
Assume that the two systems, the ground-truth $\mathcal{S}$ and learnt approximation $\hat{\mathcal{S}}$, are \emph{observationally equivalent}, which posits that the modelled systems are equivalent up to invertible parameter representations. Observational equivalence requires, ${
    \forall \boldsymbol{x}_0, \boldsymbol{\theta}: \quad \exists \hat{\boldsymbol{\theta}} \text{ s.t. } \bh(\boldsymbol{x}_0, \boldsymbol{\theta}) = \hat{\bh}(\boldsymbol{x}_0,\hat{\boldsymbol{\theta}})}$ (Equivalently,  $\hat{\bh} (\cdot)$ is surjective).
\end{assumption}
\begin{assumption}[The Transition Model is Markov to a DAG \citep{lachapelle2024nonparametric}]
\label{assumption:markov}
The transition models of $\mathcal{S}$, and $\hat{\mathcal{S}}$, are causally and \emph{faithfully} \new{\citep{pearl,SpirtesBook}} related to directed acyclic graphs (DAG), $\mathcal{G}$, and $\hat{\mathcal{G}}$, which fully define the dependencies between the \textbf{output elements and latent parameters}. The edges of the graph are defined through the non-zero entries in the Jacobian of the transition model, requiring that the model and its first derivative are continuous.
Simply, the following causal equivalence is defined, where $\bt_{\text{Pa}_i^{\mathcal{G}}}$ denotes the subset of parameters that are parents of output $i$ in graph $\mathcal{G}$:
$%
    f_i (\boldsymbol{x}_t,\boldsymbol{\theta})
    \equiv
    f_i(\boldsymbol{x}_t,\boldsymbol{\theta}_{\text{Pa}_i^{\mathcal{G}}})
$.
\begin{equation}
(i,j) \notin \mathcal{G} \implies
\frac{\partial (\bf(\boldsymbol{x}_t,\boldsymbol{\theta}))_i}{\partial \theta_j} = 0 \quad \forall \boldsymbol{x}_t, \mathbf{\theta}.
\end{equation}
\end{assumption}
Assumptions \ref{assumption:existence}, \ref{assumption:obs_eq} enforce separability and uniqueness of dynamical influences, as well as ensuring that the learnt approximation sufficiently captures the system dynamics. Assumption \ref{assumption:markov} formalises the intuition that parameters affect the system components sparsely by defining parameter influence as an inherent property of a system.
All assumptions are commonly made in the non-linear ICA and causal representation learning literature \citep{oncausal,yao2024unifying}.
The fundamental investigation of this work is then to specify the conditions under which the representations of system parameters learnt by the parameter encoder, $\bh (\cdot)$, 
align with the ground-truth system parameters, or equivalently, when the representation is identified \citep{HYVARINEN2023100844,carbonneau2022measuring}.

Relative to a ground-truth representation, a learnt representation is identified (definition \ref{definition_identifiable})\footnote{All definitions can be found in appendix \ref{appendix:useful_definitions}.} 
if it is equivalent to the ground truth. Exact identifiability conditions are, in general, unobtainable in unsupervised settings \citep{disentanglement-impossible-without-structure-locatello}, due to a lack of defined scaling and bias of functional mappings \citep{lachapelle2022synergies}.
Consequently, we consider a looser form of identifiability, identifiability up to an equivalence operator (definition \ref{definition_identifiable_equivalence}). Disentanglement is defined in this setting as identifiability up to permutation and element-wise diffeomorphism.

\subsection{Graphical Criterion}
We present our main identifiability result, which shows how mechanism sparsity, when applied to the context of learning parameters of dynamical systems, can lead to identifiable parameter representations up to permutation and element-wise diffeomorphism.
\setcounter{ctheorem}{0}
\begin{ctheorem}[Disentanglement of Latent System Parameters]
    \label{theorem:1}
    Let $\mathcal{S}$, and $\hat{\mathcal{S}}$ correspond to two systems satisfying assumptions \ref{assumption:existence}, \ref{assumption:obs_eq}, and \ref{assumption:markov}.
    Further assume that,
    \begin{enumerate}
        \item The system, $\mathcal{S}$, is path connected (assumption \ref{assumption:path_connected}).
        \item $\hat{\mathcal{S}}$ satisfies $\Vert \hat{\mathcal{G}} \Vert_0 \leq \Vert \mathcal{G} \Vert_0$ (assumption \ref{assumption:sparsity_reg}). 

        \item The Jacobian of the transition model, $\nabla_{\btheta} \bf$, \emph{varies sufficiently} (assumption \ref{assumption:sufficient_variability}).
    \end{enumerate}
    Then the dynamical representations of $\mathcal{S}$, and $\hat{\mathcal{S}}$ are equivalent up to permutation and element-wise diffeomorphism if $\mathcal{G}$ satisfies the following graphical criterion:
    \begin{equation}
       \forall i: \quad  \cap_{a\in\text{Ch}_\mathcal{G}(i)}\text{Pa}_\mathcal{G}(a)=\{i\},
       \label{equation:main_graphical_criterion}
    \end{equation}
    where $\text{Ch}_\mathcal{G}(j)$, and $\text{Pa}_\mathcal{G}(j)$ denotes the set of children and parent nodes of a node $j$ in a graph $G$.
\end{ctheorem}

Here, assumption \ref{assumption:sufficient_variability}, sometimes referred to as \textit{assumption of variability} \citep{hyvarinen2019nonlinear}, is a standard assumption that requires the Jacobian of the ground truth model to depend sufficiently strongly on individual parameters such that it is possible to differentiate the parameters by observing their influence.
Assumption~\ref{assumption:path_connected} precludes the existence of disjoint partitions: in some systems, the state space can be partitioned into regions with boundaries that are not traversable by any trajectory\footnote{For an intuitive example of such a system, please refer to Appendix \ref{sec:appendix:example_path_connectedness}}. In such cases, the theory will only hold locally within each partition, guaranteeing a weaker form of conditional disentanglement. Details of the assumptions and proofs are provided in appendix \ref{appendix:proof_of_theorem}.
 
The key intuition behind the theory and its practical implications can be understood through assumption~\ref{assumption:sparsity_reg} and the graphical criterion. Assumption \ref{assumption:sparsity_reg} requires the learnt model to use a parameter representation with at most as many causal influences as the true system. To illustrate how this drives disentanglement, consider a simple 2D system with two parameters, each affecting exactly one state variable: $\theta_0 \rightarrow x_0, \theta_1 \rightarrow x_1$. Suppose an arbitrary learnt model uncovers an \emph{entangled} representation. For example, a linear mixing of the ground truth parameters.
$
\hat{\theta}_0 = \frac{1}{2}(\theta_0 + \theta_1),
$ $
\hat{\theta}_1 = \frac{1}{2}(\theta_0 - \theta_1).
$
In this case, correctly reproducing the system’s behaviour would require each state variable to depend on both mixed parameters. The causal graph would therefore need to be \emph{denser} than the true one. This illustrates a general principle: entangled representations typically require at least as many, and often more, causal edges to account for the same physical effects. By constraining how many causal relations the learned model may use, assumption \ref{assumption:sparsity_reg} restricts the space of admissible representations, pushing the model toward disentangled parameters.

However, this only works if the underlying system is sufficiently sparse. If, for example, the underlying system graph were fully connected, then every representation would be fully dense and sparsity cannot induce disentanglement. 
The presented graphical criterion (equation \ref{equation:main_graphical_criterion}) characterises exactly how sparse the underlying system needs to be: a ground-truth parameter can be identified if no other parameter also influences all of its children. In the following, we show that this criterion can be substantially relaxed by considering local causal graphs, such that a parameter is only unidentifiable if there is another parameter that influences all of its children \textit{at all time}.

\subsection{Broadened Identifiability Guarantees with State-Dependent Local Causal Graphs}
In practice, many dynamical systems exhibit interactions in which parameters exert influence \emph{sparsely} along a trajectory. For example, parameters such as an object's mass and elastic coefficients influence only the specific objects involved in a collision and only at the moment of impact. A global graph must account for every possible interaction that might occur and, as a result, fails to capture the structure of a system that can be exploited to infer disentangled representations. We consider the setting in which local subsets of a global graph are expressible as functions of the state. Imagine there are $b$ distinct local causal graphs, each being active in a non-trivial open partition of the state space $\boldsymbol{x}\in\mathcal{X}_{i\in [1,b]}$. The dynamics model, $\bf(\cdot)$, can then be assumed as Markov (assumption \ref{assumption:markov}) with respect to all local graphs within their respective partitions. 
\begin{align}
     \mathcal{G}_l(\boldsymbol{x})\subseteq\mathcal{G}, \: \forall l \in \left[ 1, 2, ... , b\right].
\end{align}
We can treat each partition as an independent representation learning problem.
Within each partition, the space of admissible parameter representations is constrained separately by each unique graph. The path-connectedness assumption ensures that the system can traverse between distinct partitions of the state space along a trajectory. Therefore, a parameter representation must be reusable across partitions and satisfy each representational constraint enforced by sparsity regularisation individually. This leads to the \emph{local graphical criterion} (equation \ref{equation:local_graphical_criterion}, derivation in appendix \ref{appendix:local_causal_graph_proof}), which formalises the conditions under which parameter reuse results in disentanglement.
 \begin{equation}
   \forall i: \quad \textstyle\bigcap_{l\in[1,b]} \cap_{a\in\text{Ch}_{\mathcal{G}_l}(i)}\text{Pa}_{\mathcal{G}_l}(a)=\{i\}.
   \label{equation:local_graphical_criterion}
\end{equation}
The local graphical criterion reveals that local sparsity constitutes a strictly stronger source of identifiability than global sparsity. As local graphs contribute to disentanglement through an intersection over graphs, the feasible space of parametrisations can only shrink as additional local graphs are introduced. One consequence is that if a parameter is disentangled in \textbf{any} individual sub-graph, then it must be disentangled globally. 
Identifiability may be attainable even when no individual local graph is sufficient to guarantee the disentanglement of any parameter on its own.

\section{Learning Parameters with Sparse Attention}
\label{section:method}
The identifiability guarantees of our theorem are expressed as constraints upon both the system, $\mathcal{S}$ and a model, $\hat{\mathcal{S}}$, where assumptions \ref{assumption:existence}, \ref{assumption:path_connected}, and \ref{assumption:sufficient_variability} apply only to the system. In fact, only two constraints are imposed on the model: observational equivalence to the system and Markovian dynamics with respect to a causal graph with $\Vert \hat{\mathcal{G}} \Vert_0 \leq \Vert \mathcal{G} \Vert_0$. In practice, these conditions are unenforceable directly. We recast these constraints as an optimisation problem by relaxing observational equivalence as a dynamical reconstruction loss and enforcing sparsity through regularisation.

\subsection{Latent Representation Learning}
In accordance with our overall framework of using representation learning techniques for latent parameter estimation, our method uses a Variational Autoencoder setup~\citep{vae}. Specifically, in the dynamical systems setting, the goal is to maximise the likelihood of observed trajectories given the initial condition, $p(\btau \mid \bx_0)$, which requires a \textit{conditional} VAE~\citep{conditionalVAE}.
The overall model consists of two modules, parametrised by $\bm{\phi}$ and $\bm{\psi}$: first, the encoder, $p_{\bm{\phi}}(\bht \vert \btau)$, encodes observed trajectories into latent parameters, and second, the decoder, $p_{\bm{\psi}}(\btau|\bht, \bx_0)$, generates the reconstructed trajectory conditioned on an initial state $\bx_0$. In practice, conditioning on the initial state is implemented by concatenating the latent parameter with the initial state.
In this work, all distributions are parameterised as Gaussian distributions with learnable mean and variance and a unit isotropic Gaussian distribution is used as the prior over the latent parameters.
The encoder and the decoder are trained jointly by maximising the Evidence Lower bound, which consists of a reconstruction term, $\mathcal{L}_{\text{rec}}(\bm{\psi}, \bm{\phi})$, and a KL divergence term, $\mathcal{L}_{\text{KL}}(\bm{\phi})$. Following standard practice, we scale strength of the KL term using a hyperparameter~\citep{higgins2017betavae}. In effect, forcing the model to reconstruct trajectories enforces the observational equivalence assumption.

\subsection{Optimising for Sparse Local Causal Structure}
The theory places a sparsity constraint on the \textit{decoder}, which generates trajectories from parameters. To this end, we leverage the SPARTAN architecture introduced by \citet{spartan} to both measure and regularise the causal relationship between input parameters and system components. SPARTAN builds upon the transformer architecture by masking the flow of information between tokens and penalising the connections between them. Starting with the key, query, and value embeddings per token, $i$, for a layer $l$ out of $L$ layers, $\{ \bm{k}_i,\bm{q}_i,\bm{v}_i\}$, SPARTAN samples an adjacency matrix, $\bm{A}^l$, treating the sigmoid, $\sigma (\cdot)$, of the attention logit as the parameter of a Bernoulli distribution. The adjacency matrix is used as a mask to restrict information flow through the softmax when computing the next hidden state $h_i$. The adjacency matrix is a function of the current state and can be interpreted as a layer wise \emph{local causal graph}. When stacking multiple transformer layers, the information flow between different tokens is tracked through the \emph{number of paths}, $\bar{\bm{A}}_{i,j}$ connecting token $j$ to $i$. 
\begin{align}
    A_{i,j}^l \sim \text{Bern}(\sigma ( \bm{q}_i^T \bm{k}_j ))
    , \quad
    \bm{h}_i &= 
    \sum_j 
    \frac{
        A_{ij} \text{exp}({
            \frac{
                \bm{q}_i^T \bm{k}_j
            }{
                \sqrt{d_k}
            }
        })\bm{v}_j
    }{
        \sum_i \text{exp}({
            \frac{
                \bm{q}_i^T \bm{k}_j
            }{
                \sqrt{d_k}
            }
            })
    }
    , \quad
    \bar{\bm{A}}(\bx_0, \bht) = 
    \prod_{l=1}^{L}
    \left(
    \bm{A}^l + \bm{\mathbb{I}}
    \right)
    , \\
    \mathcal{L}_{\text{path}}(\bm{\psi},\bm{\phi}) &= 
        \mathbb{E}_{
        \bht \sim p_{\bm{\phi}}(\cdot \mid \btau), \:
        \bx_t \in \btau\in D
    }
    \big[ 
    \Vert \bar{\bm{A}}(\bx_0, \bht)\Vert_1. 
    \big]
\end{align}
The identity matrix is added to account for residual connections. By inspecting the values of the path matrix, SPARTAN reveals a state-dependent causal graph between tokens that is differentiable and can be regularised using a path loss, $\mathcal{L}_{\text{path}}(\bm{\psi},\bm{\psi})$. 

\subsection{Training Objective}
Sparsity, reconstruction, and latent regularisation are jointly optimised in our proposed implementation of the theory. However, it is the sparsity of causal relations that drives disentanglement in our theory and not the other losses. To ensure the optimisation procedure learns the \emph{sparsest} representation that models the dynamics, training is formulated as a constrained optimisation problem, where sparsity is minimised subject to a constraint on non-causal losses set by the target loss $\mathcal{L}_\star$. Learning the \emph{sparsest} representation ensures that $\Vert \hat{\mathcal{G}} \Vert_0 \leq \Vert \mathcal{G} \Vert_0$ without explicitly defining $\Vert \mathcal{G} \Vert_0$.
\begin{align}
    \min_{\bm{\phi},\bm{\psi}} \mathcal{L}_{\text{path}}(\bm{\psi},\bm{\phi}) \text{ subject to } \mathcal{L}_{\text{rec}}(\bm{\psi}, \bm{\phi})+ \mathcal{L}_{\text{KL}}(\bm{\phi}) + \mathcal{L}_{\text{logit}}(\bm{\psi}) \leq  \mathcal{L}_\star
    .
\end{align}
Auxiliary losses are introduced in subsequent sections. The constrained optimisation problem is solved by introducing a Lagrange multiplier, $\lambda_{\text{dual}}$, and solving via a dual min-max optimisation scheme by taking alternating gradient steps on the parameters and $\lambda_{\text{dual}}$ \citep{rezende2018taming}.
\begin{align}
\label{equation:loss}
        \max_{\lambda_{\text{dual}}>0}\: \min_{\bm{\psi}, \bm{\phi}} \: 
        \mathcal{L}_{path}(\bm{\psi},\bm{\phi}) + 
\lambda_{\text{dual}}\Big( 
        \mathcal{L}_{\text{rec}}(\bm{\psi}, \bm{\phi})
        +\mathcal{L}_{\text{KL}}(\bm{\phi})+\mathcal{L}_{\text{logit}}(\bm{\psi})
        -  \mathcal{L}_\star    
        \Big)
    .
\end{align}

\paragraph{Attention Logit Regularisation.} 
The min-max optimisation scheme results in a two-phase learning process. First, the model learns to reconstruct whilst ignoring the sparsity regularisation loss ($\lambda_\text{dual}\rightarrow \infty$). At some point during training, the reconstruction threshold is reached, and the sparsity loss becomes dominant ($\lambda_\text{dual}\rightarrow 0$), which will push the entangled representation toward disentanglement. The representational \emph{switch} poses a complex optimisation problem, particularly as transformers learn increasingly peaked and low-entropy attention patterns during training \citep{zhai2023stabilizing}, leading to vanishing gradients via the softmax, thereby hindering representational plasticity during the second phase of training. Vanishing gradients are combatted through a small loss that penalises large attention logits, $\mathcal{L}_\text{logit}(\bm{\psi})$. An ablation exemplifying the necessity of the logit losses is provided in appendix \ref{appendix:attention_logit_ablation}. The functional form of the logit loss is shown below, where $l \in L$, $i \in T$, $j \in T$ represent the layer and two token dimensions respectively, and the query and key vector for each layer and token are functions of $\bx_t$ and $\bht$.
\begin{align}
    \mathcal{L}_{\text{logit}}(\bm{\psi}) &=
    \mathbb{E}_{
            \bht \sim p_{\bm{\phi}}(\cdot \mid \btau), \:
        \bx_t \in \btau\in D}
    \Bigg[
    \frac{\lambda_{logit}}{LT^2}\sum_{l,i,j=1}^{L,T,T}
    e^{
    \bm{q}_{l,i}^T \bm{k}_{l,j}
    }+ e^{
    -\bm{q}_{l,i}^T \bm{k}_{l,j}
    }
    \Bigg].
    \label{equation:logit_loss}
\end{align}

\section{Experiments}
\label{section:experiments}
Building on the local causal graph discovery method outlined in section \ref{section:method} and the conditions specified in our main theorem, this section empirically examines the feasibility of recovering disentangled representations amidst modelling and approximation errors. Additionally, by testing in controlled domains, we assess the identifiability claims in section \ref{section:theory} and validate the theoretical framework through practical experiments.
\paragraph{Environments.} We evaluate our method across four synthetic domains with known ground-truth parameters and causal structure. Disentanglement of a learnt representation up to permutation and element-wise diffeomorphism is measured by computing the \emph{mean-max correlation coefficient} (MCC)\footnote{Details on the MCC metric are provided in appendix \ref{appendix:mcc_metric}.} metric between the ground-truth and learnt parameters, as is common in related literature \citep{lachapelle2022synergies,yao2022temporally,lachapelle2022partial}.

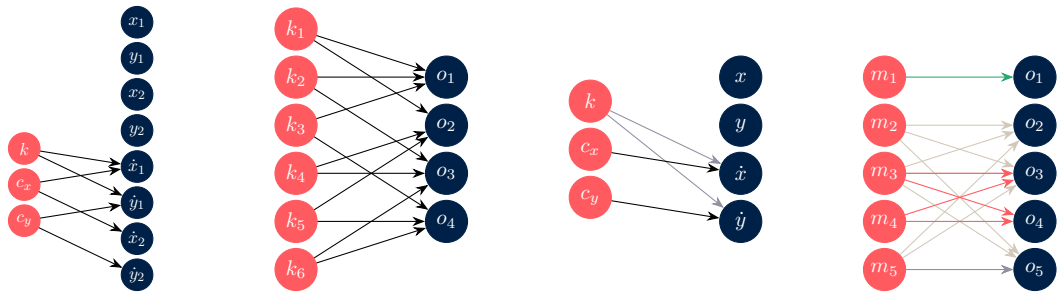
\begin{figure}[t]
    \centering
    \begin{minipage}[t]{0.23\textwidth}
        \centering
        \begin{tikzpicture}[
            scale=0.6,
            every node/.style={transform shape},
            node distance=0.5cm and 2.5cm, 
            >=Stealth,
            leftnode/.style={circle, draw=oxfordred, fill=oxfordred, text=white, minimum size=7mm, font=\bfseries, inner sep=0pt},
            rightnode/.style={circle, draw=oxfordblue, fill=oxfordblue, text=white, minimum size=7mm, font=\bfseries, inner sep=0pt}
        ]
            \node[leftnode] (k) at (0,-0.0) {$k$};
            \node[leftnode] (cx) at (0,-0.8) {$c_x$};
            \node[leftnode] (cy) at (0,-1.6) {$c_y$};
            \node[rightnode] (x1) at (2.5, 2.8) {$x_1$};
            \node[rightnode] (y1) at (2.5, 2.0) {$y_1$};
            \node[rightnode] (x2) at (2.5, 1.2) {$x_2$};
            \node[rightnode] (y2) at (2.5, 0.4) {$y_2$};
            \node[rightnode] (vx1) at (2.5, -0.4) {$\dot{x}_1$};
            \node[rightnode] (vy1) at (2.5, -1.2) {$\dot{y}_1$};
            \node[rightnode] (vx2) at (2.5, -2.0) {$\dot{x}_2$};
            \node[rightnode] (vy2) at (2.5, -2.8) {$\dot{y}_2$};

            \draw[->] (k) -- (vx1);
            \draw[->] (k) -- (vy1);
            \draw[->] (cx) -- (vx1);
            \draw[->] (cx) -- (vx2);
            \draw[->] (cy) -- (vy1);
            \draw[->] (cy) -- (vy2);

        \end{tikzpicture}
    \end{minipage}
    \hfill
    \begin{minipage}[t]{0.23\textwidth}
        \centering
        \begin{tikzpicture}[
            scale=0.8,
            every node/.style={transform shape},
            node distance=0.75cm and 2.5cm, 
            >=Stealth,
            leftnode/.style={circle, draw=oxfordred, fill=oxfordred, text=white, minimum size=7mm, font=\small, inner sep=0pt},
            rightnode/.style={circle, draw=oxfordblue, fill=oxfordblue, text=white, minimum size=7mm, font=\small, inner sep=0pt}
        ]
        \node[leftnode] (k_1) at (0, 2.0) {$k_1$};
        \node[leftnode] (k_2) at (0, 1.2) {$k_2$};
        \node[leftnode] (k_3) at (0, 0.4) {$k_3$};
        \node[leftnode] (k_4) at (0, -0.4) {$k_4$};
        \node[leftnode] (k_5) at (0, -1.2) {$k_5$};
        \node[leftnode] (k_6) at (0, -2.0) {$k_6$};

        \node[rightnode] (p_1) at (2.5, 1.2) {$o_1$};
        \node[rightnode] (p_2) at (2.5, 0.4) {$o_2$};
        \node[rightnode] (p_3) at (2.5, -0.4) {$o_3$};
        \node[rightnode] (p_4) at (2.5, -1.2) {$o_4$};

        \draw[->] (k_1) -- (p_1);
        \draw[->] (k_1) -- (p_2);

        \draw[->] (k_2) -- (p_1);
        \draw[->] (k_2) -- (p_3);

        \draw[->] (k_3) -- (p_1);
        \draw[->] (k_3) -- (p_4);

        \draw[->] (k_4) -- (p_2);
        \draw[->] (k_4) -- (p_3);

        \draw[->] (k_5) -- (p_2);
        \draw[->] (k_5) -- (p_4);

        \draw[->] (k_6) -- (p_3);
        \draw[->] (k_6) -- (p_4);

        \end{tikzpicture}
        \label{fig:environments:springs}
    \end{minipage}
    \hfill
    \begin{minipage}[t]{0.23\textwidth}
        \centering
        \begin{tikzpicture}[
            scale=0.8,
            every node/.style={transform shape},
            node distance=0.75cm and 2.5cm, 
            >=Stealth,
            leftnode/.style={circle, draw=oxfordred, fill=oxfordred, text=white, minimum size=7mm, font=\small, inner sep=0pt},
            rightnode/.style={circle, draw=oxfordblue, fill=oxfordblue, text=white, minimum size=7mm, font=\small, inner sep=0pt},
            blank/.style={circle,draw=white,fill=white,text=white,minimum size = 7mm, font=\small, inner sep = 0pt},
        ]
            \node[leftnode] (k) at (0,0.8) {$k$};
            \node[leftnode] (cx) at (0,0) {$c_x$};
            \node[leftnode] (cy) at (0,-0.8) {$c_y$};
            \node[rightnode] (x) at (2.5, 1.2) {$x$};
            \node[rightnode] (y) at (2.5, 0.4) {$y$};
            \node[rightnode] (vx) at (2.5, -0.4) {$\dot{x}$};
            \node[rightnode] (vy) at (2.5, -1.2) {$\dot{y}$};
            \node[blank]    (blank) at (2.5,-2.0) {$ $};

            \draw[->, oxfordgrey] (k) -- (vx);
            \draw[->, oxfordgrey] (k) -- (vy);
            \draw[->] (cx) -- (vx);
            \draw[->] (cy) -- (vy);

        \end{tikzpicture}
        \label{fig:environments:local}

    \end{minipage}
    \hfill
    \begin{minipage}[t]{0.23\textwidth}
        \centering
        \begin{tikzpicture}[
            scale=0.8,
            every node/.style={transform shape},
            node distance=0.75cm and 2.5cm, 
            >=Stealth,
            leftnode/.style={circle, draw=oxfordred, fill=oxfordred, text=white, minimum size=7mm, font=\small, inner sep=0pt},
            rightnode/.style={circle, draw=oxfordblue, fill=oxfordblue, text=white, minimum size=7mm, font=\small, inner sep=0pt}
        ]
            \node[leftnode] (m1) at (0, 1.6) {$m_1$};
            \node[leftnode] (m2) at (0, 0.8) {$m_2$};
            \node[leftnode] (m3) at (0, 0.0) {$m_3$};
            \node[leftnode] (m4) at (0, -0.8) {$m_4$};
            \node[leftnode] (m5) at (0, -1.6) {$m_5$};

            \node[rightnode] (o1) at (2.5, 1.6) {$o_1$};
            \node[rightnode] (o2) at (2.5, 0.8) {$o_2$};
            \node[rightnode] (o3) at (2.5, -0.0) {$o_3$};
            \node[rightnode] (o4) at (2.5, -0.8) {$o_4$};
            \node[rightnode] (o5) at (2.5, -1.6) {$o_5$};

            \draw[->, oxfordcream] (m2) -- (o2);
            \draw[->, oxfordcream] (m2) -- (o3);
            \draw[->, oxfordcream] (m2) -- (o5);
            \draw[->, oxfordcream] (m3) -- (o2);
            \draw[->, oxfordcream] (m3) -- (o3);
            \draw[->, oxfordcream] (m3) -- (o5);
            \draw[->, oxfordcream] (m5) -- (o2);
            \draw[->, oxfordcream] (m5) -- (o3);
            \draw[->, oxfordcream] (m5) -- (o5);
            
            \draw[->, jungle-green] (m1) -- (o1);
            \draw[->, oxfordgrey] (m5) -- (o5);

            \draw[->, oxfordred] (m3) -- (o4);
            \draw[->, oxfordred] (m4) -- (o3);
            \draw[->, oxfordred] (m3) -- (o3);
            \draw[->, oxfordred] (m4) -- (o4);

        \end{tikzpicture}
        \label{fig:environments:bounce}
    \end{minipage}
    \caption{
    Depiction of the ground-truth DAGs of the evaluation environments. The \emph{left} and \emph{centre-left} correspond to the dual particle and springs environments, respectively, and both satisfy the global graphical criterion for disentanglement. The \emph{centre-right} and \emph{right} graphs correspond to the local particle and bounce environments, respectively. Coloured arrows indicate causal edges that are only active in subsets of the state space, and only a limited number of subsets are shown for the bounce environment. The local particle and bounce environments satisfy only the local, not the global, graphical criterion for disentanglement.
    }    \label{fig:particles_dag_three}
\end{figure}

The \emph{dual particle} and \emph{local particle} environments consist of modified 2D mass-spring-damper systems. The former contains two point masses with independent damping on the vertical and horizontal axes, where one point mass is additionally tethered to the origin via a spring. The latter models a single point mass under independent x-y damping and a spring which becomes \emph{slack} near the origin. The \emph{springs} environment features four point masses coupled by six springs, and the \emph{bounce} environment simulates five variable mass balls that elastically collide, creating dynamic local causal graphs. Figure \ref{fig:particles_dag_three} depicts the corresponding causal structures, with expanded details on each environment provided in appendix \ref{appendix:dataset}. Together, these environments span systems in which the global graphical criterion is sufficient for identifiability, the local criterion is adequate, and the global criterion is insufficient; the parameters causally influence scalar state values, and the parameters causally influence the system at the object level (modelled using vector representations). The \emph{springs} and \emph{bounce} environments are inspired by similar variations from prior works in causal discovery and representation learning \citep{ACD,battaglia2016interaction, li2020causal}.

\paragraph{Baselines.}
We compare against baselines chosen to isolate the role of causal structure and sparsity in parameter identifiability. 
All models share the same variational autoencoder learning objective and differ only in their inductive bias on dependencies between parameters and state, i.e. the decoder.
Since the theory only places sparsity constraints on the decoder, the encoder architecture is held fixed across all baselines. An MLP is used as the encoder in all environments except the \emph{bounce} environment, where a convolutional network is found to be more efficient for longer sequences.
For the decoder, we compare against three baselines: 1. \textit{MLP}, which imposes no structural constraints and serves as a control for disentanglement arising from the variational bottleneck alone;
2. \textit{VCD}~\citep{vcd}, which uses a masked MLP decoder based on a learned global causal graph and satisfies identifiability assumptions only when the global graphical criterion holds;
3. \textit{Transformer}, which uses a transformer that does not explicitly impose sparsity constraints, but nonetheless can provide an implicit bias towards sparse information flow due to the softmax attention.
Finally, our method, labelled SPARTAN, uses a sparsity-regularised transformer as the decoder, as described in section \ref{section:method}.
\subsection{Results}
\begin{figure}[t]
    \centering
    \includegraphics[width=0.9\textwidth]{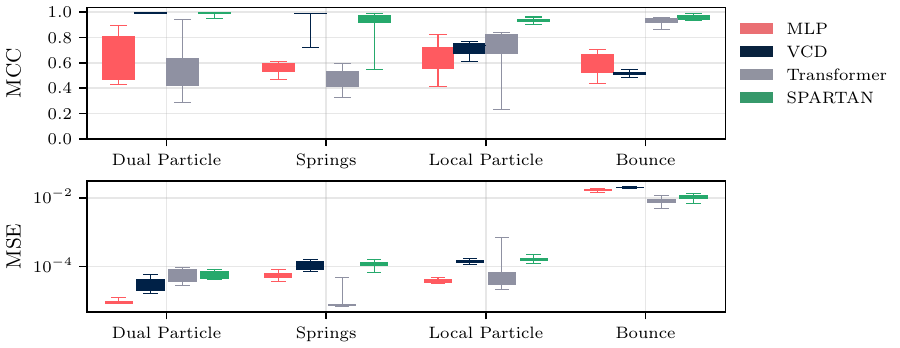}
    \caption{Comparison of disentanglement across the test environments, where an MCC of $1.0$ represents perfect disentanglement. All trials are repeated over eight random seeds. Box plots display the minimum, lower quartile, upper quartile, and maximum values. The validation reconstruction loss is shown at the bottom, indicating that all models are approximately equiperformant in these environments. The VCD baseline, which learns static graphs, strongly disentangles in the first two environments, which satisfies the global graph criterion. In contrast, only SPARTAN, which learns state-dependent graphs, consistently disentangles in all environments.}
    \label{figure:primary:results}
\end{figure}
\begin{figure}[p]
    \centering
    \subfigure{\includegraphics[width=0.49\textwidth, trim=4pt 4pt 4pt 6pt, clip]{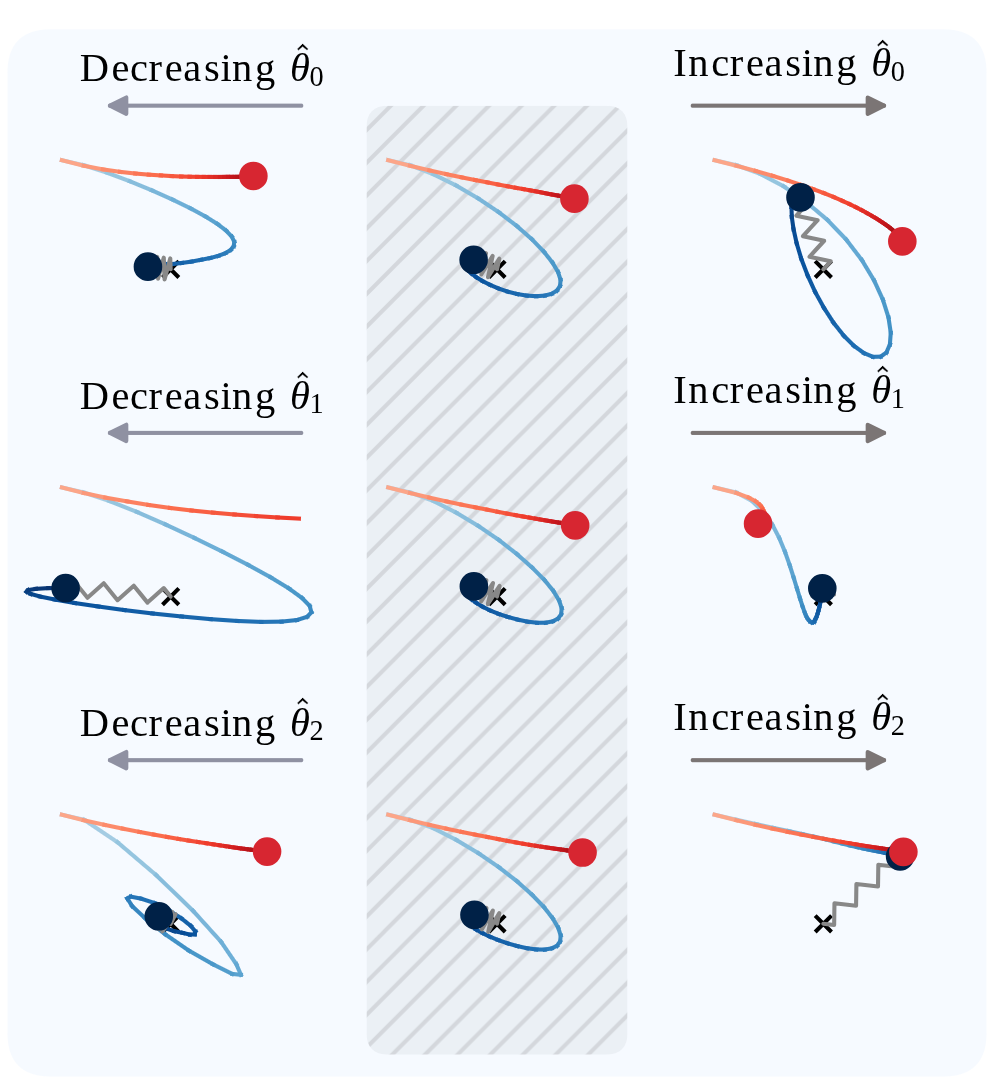}} 
    \subfigure{\includegraphics[width=0.49\textwidth, trim=4pt 4pt 4pt 6pt, clip]{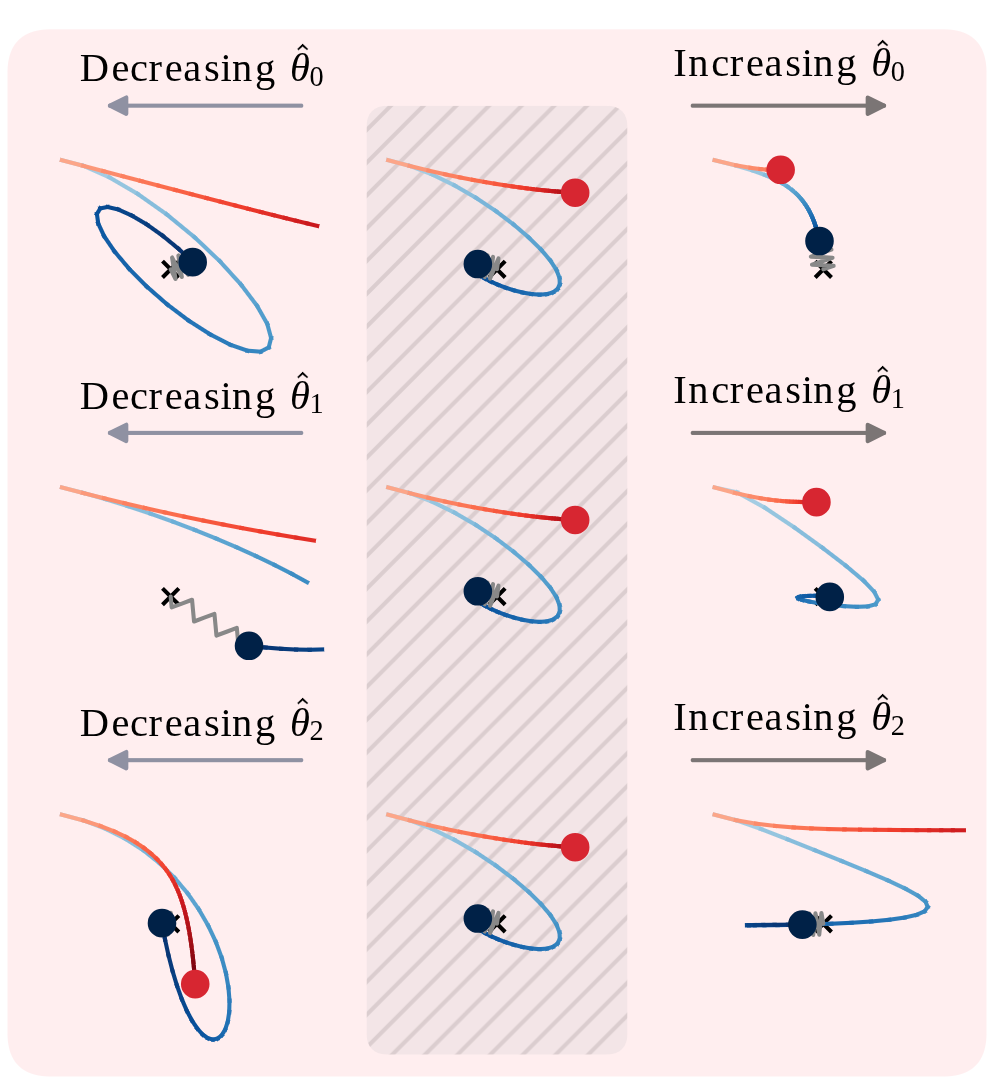}} 
    \caption{
    Autoregressive rollouts of the dual particle environment produced from the learnt (a) SPARTAN, and (b) MLP models. Starting from the same initial state, the red particle experiences resistive damping forces independently in the horizontal and vertical axes, whilst the blue particle additionally experiences a spring force to the origin. The ground-truth system parameters are the damping coefficients and the spring constant.   
    (a) shows a clear separation of parameters, increasing $\hat{\theta}_2$ evidently reduces spring strength. As the red particle is unaffected by the spring, it is invariant to changes in $\hat{\theta}_2$. Increasing $\hat{\theta}_1$ \emph{increases} horizontal damping, and increasing $\hat{\theta}_0$ \emph{reduces} vertical axis damping. In contrast, the entangled representations produced by the MLP yield no such insight.}
    \label{fig:example_rollouts}
\end{figure}
\begin{figure}[p]
    \centering
    \includegraphics[width=\textwidth]{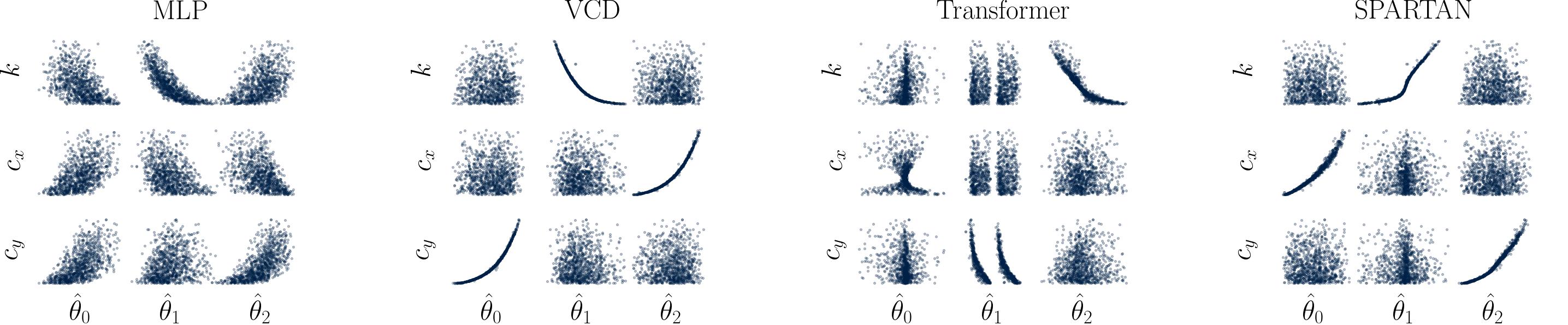}
    \caption{
    Plots of the entanglement mapping for all models in the dual particle environment. Each subplot shows the marginal distribution of a learnt parameter (x-axis) against the ground-truth (y-axis). Representative trials were chosen. Disentanglement is exemplified by the clear element-wise diffeomorphic structure learnt by both the SPARTAN and VCD architectures, whereas the MLP and Transformer architectures learn entangled representations.    
    Extended results for all environments and representative samples of the causal graphs learnt are provided in appendix \ref{appendix:extended_results}.}
    \label{fig:learnt_representations}
\end{figure}

Figure \ref{figure:primary:results} shows the distribution of MCC scores (higher is better) obtained when training each baseline in each environment across eight random seeds. Reconstruction error is reported to verify that differences in disentanglement are not attributable to failures in modelling the system's dynamics. 

In environments where the global causal graphs satisfy the graphical criterion, namely \emph{dual particle} and \emph{springs}, models enforcing global causal sparsity recover disentangled parameter representations. Both the VCD and SPARTAN baselines consistently achieve MCC scores close to one (the upper bound), corroborating our theoretical analysis. In contrast, both the MLP and transformer baselines exhibit a substantially lower and broader distribution of achieved disentanglement.
Figure \ref{fig:learnt_representations} illustrates samples of the latent representations learnt by each model in the \emph{dual particle} environment. 
Here, we observe that VCD and SPARTAN both learn latent parameters that can be mapped to the ground-truth parameters in a one-to-one manner, i.e. each ground-truth parameter can be predicted from exactly one learned parameter, qualitatively illustrating the disentanglement induced by sparsity regularisation.
Similar examples for other environments can be found in Appendix \ref{appendix:representations}.
To demonstrate the \textit{local} extension of our theory, in the \emph{local particle} and \emph{bounce} environments, the global graphical criterion is not satisfied, and local causal graphs are required for disentanglement. Consequently, the VCD baseline fails to recover disentangled parameters, while
SPARTAN consistently learns highly disentangled representations. 
In the \textit{local particle} environment, VCD is still able to attain stronger disentanglement compared to the other baselines, since the system graph still supports partial disentanglement.
In the \textit{bounce} environment, however, since the local causal graphs are significantly sparser (e.g. edges only exist for one timestep during collision), and therefore the VCD baseline is unable to recover any disentanglement.
We note that the Transformer baseline also achieves high disentanglement in the \textit{Bounce} environment. We hypothesise that this is due to the implicit bias from softmax attention, which has been shown to be sufficient for inducing local causal graphs in simple situations~\citep{Coda}.
Nonetheless, SPARTAN is the only model that is able to learn disentangled parameters robustly across all tested environments.
To further demonstrate the usefulness of the learned parameters, we perform latent \textit{interventions} on the learned dynamical systems to generate counterfactual rollout trajectories. Figure \ref{fig:example_rollouts} visualises how intervening on the learnt parameters in the \emph{dual particle} environment results in physically meaningful changes to trajectories, whereas intervening on entangled parameters lead to uninterpretable changes in system behaviours that affect all states.

Overall, the empirical results corroborate the identifiability claims of section~\ref{section:theory} and show that in settings where the relevant criteria are satisfied, disentangled representations are recoverable by realising sparse causal relations via appropriate sparsity regularisation.

\section{Related Works}
The theoretical backbone of our work draws inspiration from prior identifiability results in Causal Representation Learning (CRL) \citep{scholkopf2021toward} and Non-Linear Independent Component Analysis \citep{hyvarinen1999nonlinear}, which explore the conditions under which recovering unobserved variables from data is possible.
Works in this area show that disentanglement can be achieved via various inductive biases (e.g., additive decoders \citep{lachapelle2023additive}), assumptions on the data-generating processes (e.g. equivariances \citep{ahuja2021properties}, multi-view data \citep{yao2024multi}, and auxiliary information such as interventions \citep{lippe2022citris}). Variations exist within each subgenre; interventions can be known a priori \citep{ahuja2023}, unknown \citep{von2023nonparametric,varici2024linear}, or differ in their assumptions about intervention targets. 
Furthermore, when distinguishing between non-temporal and temporal data, temporal CRL leverages distinct assumptions, including non-stationary time series \citep{hyvarinen2016unsupervised,song2023}, interventions on the temporal transition function \citep{lippe2022citris}, or sparse relations between ground-truth latent factors \citep{li2025}.

Our approach builds specifically on the mechanism sparsity principle \citep{lachapelle2022partial} to achieve identifiability.
The key distinction between our approach and existing CRL efforts is that, whereas prior works explore \textit{state} representation for observations such as images \citep{buchholz2023learning,lachapelle2024nonparametric}, our setting concerns learning dynamical \textit{parameters}.
In this sense, our work can be viewed as an instantiation of the broader framework of \textit{parameter learning as representation learning}, as advocated in recent literature \citep{yao2024marrying}.

Central to our theory and proposed method is the notion of \textit{local causal graphs}, which capture the causal dependencies between variables at a fine-grained temporal resolution. This nascent idea has recently been explored in the context of world models \citep{zhao2025curiouscausalityseekingagentslearn,spartan} and reinforcement learning \citep{hwang2024fine,seitzer2021causal}, where state-dependent local structures are shown to improve robustness and adaptation efficiency.
To the best of our knowledge, our work is the first to theoretically demonstrate how local causal graphs can extend identifiability guarantees for parameter learning and, more broadly, representation learning.

On a conceptual level, our work shares the motivation of System Identification methods, which aim to estimate the underlying parameters of a system given observed trajectories.
Here, classical methods require significant prior knowledge of the system, such as known functional forms in PINN-based ODE learning methods \citep{PINNs}, and a pre-specified library of functions in the SINDy family of works \citep{sindy,sindy-pde}. 
In contrast, our work relaxes these requirements and relies only on the general assumption of mechanism sparsity. 
Beyond System Identification, the notion of extracting latent descriptions of systems from observed trajectories also appears in adjacent fields such as meta-RL \citep{MetaCARD,DreamToAdapt}, where latent task embeddings are inferred from past trajectories, and in-context inductive reasoning \citep{lpn}, where the system rules are inferred from context examples as latent program codes. 
We believe the insights developed in this work are readily transferable to these lines of work.

\section{Conclusion}
In this work, we analyse the identifiability of dynamical system parameters by extending prior results in causal representation learning.
Our main result shows that mechanism sparsity can induce identifiable representation of dynamical parameters and make explicit the conditions on the class of dynamical systems where this is possible.
In settings where parameters exert \textit{temporally sparse} causal influence on the system, we extend our theory to show that state-dependent local causal graphs strictly strengthen these identifiability guarantees, enlarging the class of identifiable systems.

Motivated by the theory, we propose a practical algorithm for identifying latent parameters using variational inference and sparse transformers.
Experiments across four synthetic domains corroborate the theoretical analysis: mechanism sparsity allows consistent learning of disentangled representations, and in scenarios where globally sparse methods are insufficient, local sparse attention enables consistent disentanglement. Furthermore, we empirically show that the disentangled parameters learnt under these conditions enable physically meaningful and localised interventions on system behaviour, as well as counterfactual trajectory generation. 
\paragraph{Limitations and Future work.}
The present work has several limitations that point toward promising directions for future investigation.
It has long been conjectured that disentangled representations play a central role in improving generalisation and robustness \citep{lachapelle22a,bengio2013representation}. While our experiments demonstrate that disentanglement is possible, future work should evaluate its generalisation properties in downstream tasks.
Moreover, in its current form, our theory requires a priori knowledge of the dimension of the parameter space, which can be relaxed in future work by, for example, performing model selection on the latent dimension.
Finally, the theorem makes explicit \emph{what} can be learnt from a given dataset and system, and, as such, future work can exploit the graphical criterion for active data acquisition strategies.
\newpage
\acks{
This research was supported by an EPSRC Programme Grant (EP/V000748/1). The authors would like to acknowledge the use of the University of Oxford Advanced Research Computing (ARC) facility (http://dx.doi.org/10.5281/zenodo.22558) and the SCAN computing cluster in carrying out this work. Ingmar Posner holds concurrent appointments as a Professor of Applied AI at the University of Oxford and as an Amazon Scholar. This paper describes work performed at the University of Oxford and is not associated with Amazon.

The authors would also like to thank Alexander Mitchell and Frederik Nolte, whose keen insight and invaluable discussions were instrumental in making this work possible.
}

\bibliography{sections/citations.bib}

\begin{thebibliography}{76}
\providecommand{\natexlab}[1]{#1}
\providecommand{\url}[1]{\texttt{#1}}
\expandafter\ifx\csname urlstyle\endcsname\relax
  \providecommand{\doi}[1]{doi: #1}\else
  \providecommand{\doi}{doi: \begingroup \urlstyle{rm}\Url}\fi

\bibitem[Abramson et~al.(2024)Abramson, Adler, Dunger, Evans, Green, Pritzel, Ronneberger, Willmore, Ballard, Bambrick, Bodenstein, Evans, Hung, O'Neill, Reiman, Tunyasuvunakool, Wu, {\v{Z}}emgulyt{\.{e}}, Arvaniti, Beattie, Bertolli, Bridgland, Cherepanov, Congreve, Cowen-Rivers, Cowie, Figurnov, Fuchs, Gladman, Jain, Khan, Low, Perlin, Potapenko, Savy, Singh, Stecula, Thillaisundaram, Tong, Yakneen, Zhong, Zielinski, {\v{Z}}{\'i}dek, Bapst, Kohli, Jaderberg, Hassabis, and Jumper]{AlphaFold3}
Josh Abramson, Jonas Adler, Jack Dunger, Richard Evans, Tim Green, Alexander Pritzel, Olaf Ronneberger, Lindsay Willmore, Andrew~J. Ballard, Joshua Bambrick, Sebastian~W. Bodenstein, David~A. Evans, Chia-Chun Hung, Michael O'Neill, David Reiman, Kathryn Tunyasuvunakool, Zachary Wu, Akvil{\.{e}} {\v{Z}}emgulyt{\.{e}}, Eirini Arvaniti, Charles Beattie, Ottavia Bertolli, Alex Bridgland, Alexey Cherepanov, Miles Congreve, Alexander~I. Cowen-Rivers, Andrew Cowie, Michael Figurnov, Fabian~B. Fuchs, Hannah Gladman, Rishub Jain, Yousuf~A. Khan, Caroline M.~R. Low, Kuba Perlin, Anna Potapenko, Pascal Savy, Sukhdeep Singh, Adrian Stecula, Ashok Thillaisundaram, Catherine Tong, Sergei Yakneen, Ellen~D. Zhong, Michal Zielinski, Augustin {\v{Z}}{\'i}dek, Victor Bapst, Pushmeet Kohli, Max Jaderberg, Demis Hassabis, and John~M. Jumper.
\newblock Accurate structure prediction of biomolecular interactions with alphafold 3.
\newblock \emph{Nature}, 630\penalty0 (8016):\penalty0 493--500, Jun 2024.
\newblock ISSN 1476-4687.
\newblock \doi{10.1038/s41586-024-07487-w}.
\newblock URL \url{https://doi.org/10.1038/s41586-024-07487-w}.

\bibitem[Ahuja et~al.(2022)Ahuja, Hartford, and Bengio]{ahuja2021properties}
Kartik Ahuja, Jason Hartford, and Yoshua Bengio.
\newblock Properties from mechanisms: an equivariance perspective on identifiable representation learning.
\newblock In \emph{International Conference on Learning Representations}. Proceedings of Machine Learning Research, 2022.

\bibitem[Ahuja et~al.(2023)Ahuja, Mahajan, Wang, and Bengio]{ahuja2023}
Kartik Ahuja, Divyat Mahajan, Yixin Wang, and Yoshua Bengio.
\newblock Interventional causal representation learning.
\newblock In \emph{International conference on machine learning}, pages 372--407. PMLR, 2023.

\bibitem[Alet et~al.(2025)Alet, Price, El-Kadi, Masters, Markou, Andersson, Stott, Lam, Willson, Sanchez-Gonzalez, et~al.]{weather-fgn}
Ferran Alet, Ilan Price, Andrew El-Kadi, Dominic Masters, Stratis Markou, Tom~R Andersson, Jacklynn Stott, Remi Lam, Matthew Willson, Alvaro Sanchez-Gonzalez, et~al.
\newblock Skillful joint probabilistic weather forecasting from marginals.
\newblock \emph{arXiv preprint arXiv:2506.10772}, 2025.

\bibitem[Battaglia et~al.(2016)Battaglia, Pascanu, Lai, Jimenez~Rezende, et~al.]{battaglia2016interaction}
Peter Battaglia, Razvan Pascanu, Matthew Lai, Danilo Jimenez~Rezende, et~al.
\newblock Interaction networks for learning about objects, relations and physics.
\newblock \emph{Advances in neural information processing systems}, 29, 2016.

\bibitem[Bengio et~al.(2013)Bengio, Courville, and Vincent]{bengio2013representation}
Yoshua Bengio, Aaron Courville, and Pascal Vincent.
\newblock Representation learning: A review and new perspectives.
\newblock \emph{IEEE transactions on pattern analysis and machine intelligence}, 35\penalty0 (8):\penalty0 1798--1828, 2013.

\bibitem[Bronstein et~al.(2021)Bronstein, Bruna, Cohen, and Veli{\v{c}}kovi{\'c}]{bronstein2021geometric}
Michael~M Bronstein, Joan Bruna, Taco Cohen, and Petar Veli{\v{c}}kovi{\'c}.
\newblock Geometric deep learning: Grids, groups, graphs, geodesics, and gauges.
\newblock \emph{arXiv preprint arXiv:2104.13478}, 2021.

\bibitem[Brunton et~al.(2016)Brunton, Proctor, and Kutz]{sindy}
Steven~L. Brunton, Joshua~L. Proctor, and J.~Nathan Kutz.
\newblock Discovering governing equations from data by sparse identification of nonlinear dynamical systems.
\newblock \emph{Proceedings of the National Academy of Sciences}, 113\penalty0 (15):\penalty0 3932--3937, 2016.
\newblock \doi{10.1073/pnas.1517384113}.
\newblock URL \url{https://www.pnas.org/doi/abs/10.1073/pnas.1517384113}.

\bibitem[Buchholz et~al.(2023)Buchholz, Rajendran, Rosenfeld, Aragam, Sch{\"o}lkopf, and Ravikumar]{buchholz2023learning}
Simon Buchholz, Goutham Rajendran, Elan Rosenfeld, Bryon Aragam, Bernhard Sch{\"o}lkopf, and Pradeep Ravikumar.
\newblock Learning linear causal representations from interventions under general nonlinear mixing.
\newblock \emph{Advances in Neural Information Processing Systems}, 36:\penalty0 45419--45462, 2023.

\bibitem[Burgess et~al.(2018)Burgess, Higgins, Pal, Matthey, Watters, Desjardins, and Lerchner]{burgess2018understanding}
Christopher~P Burgess, Irina Higgins, Arka Pal, Loic Matthey, Nick Watters, Guillaume Desjardins, and Alexander Lerchner.
\newblock Understanding disentangling in beta-vae.
\newblock \emph{arXiv preprint arXiv:1804.03599}, 2018.

\bibitem[Carbonneau et~al.(2022)Carbonneau, Zaidi, Boilard, and Gagnon]{carbonneau2022measuring}
Marc-Andr{\'e} Carbonneau, Julian Zaidi, Jonathan Boilard, and Ghyslain Gagnon.
\newblock Measuring disentanglement: A review of metrics.
\newblock \emph{IEEE transactions on neural networks and learning systems}, 35\penalty0 (7):\penalty0 8747--8761, 2022.

\bibitem[Chang et~al.(2017)Chang, Ullman, Torralba, and Tenenbaum]{chang2017compositional}
Michael Chang, Tomer Ullman, Antonio Torralba, and Joshua Tenenbaum.
\newblock A compositional object-based approach to learning physical dynamics.
\newblock In \emph{International Conference on Learning Representations}, 2017.

\bibitem[Chen et~al.(2018)Chen, Li, Grosse, and Duvenaud]{chen2018isolating}
Ricky~TQ Chen, Xuechen Li, Roger~B Grosse, and David~K Duvenaud.
\newblock Isolating sources of disentanglement in variational autoencoders.
\newblock \emph{Advances in neural information processing systems}, 31, 2018.

\bibitem[Esmaeili et~al.(2019)Esmaeili, Wu, Jain, Bozkurt, Siddharth, Paige, Brooks, Dy, and Meent]{esmaeili2019structured}
Babak Esmaeili, Hao Wu, Sarthak Jain, Alican Bozkurt, Narayanaswamy Siddharth, Brooks Paige, Dana~H Brooks, Jennifer Dy, and Jan-Willem Meent.
\newblock Structured disentangled representations.
\newblock In \emph{The 22nd International Conference on Artificial Intelligence and Statistics}, pages 2525--2534. PMLR, 2019.

\bibitem[Fu et~al.(2025)Fu, Huang, Li, Zheng, Ng, Hu, and Zhang]{fu2025identification}
Minghao Fu, Biwei Huang, Zijian Li, Yujia Zheng, Ignavier Ng, Yingyao Hu, and Kun Zhang.
\newblock Identification of nonparametric dynamic causal structure and latent process in climate system.
\newblock \emph{arXiv preprint arXiv:2501.12500}, 2025.

\bibitem[Gao et~al.(2025)Gao, Williams, and Kutz]{sindy-shred}
Mars~Liyao Gao, Jan~P Williams, and J~Nathan Kutz.
\newblock Sparse identification of nonlinear dynamics and koopman operators with shallow recurrent decoder networks.
\newblock \emph{arXiv preprint arXiv:2501.13329}, 2025.

\bibitem[Gerken et~al.(2023)Gerken, Aronsson, Carlsson, Linander, Ohlsson, Petersson, and Persson]{Gerken2023}
Jan~E. Gerken, Jimmy Aronsson, Oscar Carlsson, Hampus Linander, Fredrik Ohlsson, Christoffer Petersson, and Daniel Persson.
\newblock Geometric deep learning and equivariant neural networks.
\newblock \emph{Artificial Intelligence Review}, 56\penalty0 (12):\penalty0 14605--14662, Dec 2023.
\newblock ISSN 1573-7462.
\newblock \doi{10.1007/s10462-023-10502-7}.
\newblock URL \url{https://doi.org/10.1007/s10462-023-10502-7}.

\bibitem[Gresele et~al.(2021)Gresele, Von~K{\"u}gelgen, Stimper, Sch{\"o}lkopf, and Besserve]{gresele2021independent}
Luigi Gresele, Julius Von~K{\"u}gelgen, Vincent Stimper, Bernhard Sch{\"o}lkopf, and Michel Besserve.
\newblock Independent mechanism analysis, a new concept?
\newblock \emph{Advances in neural information processing systems}, 34:\penalty0 28233--28248, 2021.

\bibitem[Gu and Borovykh(2023)]{gu2023original}
Boyang Gu and Anastasia Borovykh.
\newblock On original and latent space connectivity in deep neural networks.
\newblock \emph{arXiv preprint arXiv:2311.06816}, 2023.

\bibitem[Hafner et~al.(2025)Hafner, Pasukonis, Ba, and Lillicrap]{Dreamer}
Danijar Hafner, Jurgis Pasukonis, Jimmy Ba, and Timothy Lillicrap.
\newblock Mastering diverse control tasks through world models.
\newblock \emph{Nature}, 640\penalty0 (8059):\penalty0 647--653, Apr 2025.
\newblock ISSN 1476-4687.
\newblock \doi{10.1038/s41586-025-08744-2}.
\newblock URL \url{https://doi.org/10.1038/s41586-025-08744-2}.

\bibitem[Higgins et~al.(2017)Higgins, Matthey, Pal, Burgess, Glorot, Botvinick, Mohamed, and Lerchner]{higgins2017betavae}
Irina Higgins, Loic Matthey, Arka Pal, Christopher Burgess, Xavier Glorot, Matthew Botvinick, Shakir Mohamed, and Alexander Lerchner.
\newblock beta-{VAE}: Learning basic visual concepts with a constrained variational framework.
\newblock In \emph{International Conference on Learning Representations}, 2017.
\newblock URL \url{https://openreview.net/forum?id=Sy2fzU9gl}.

\bibitem[Hu et~al.(2023)Hu, Russell, Yeo, Murez, Fedoseev, Kendall, Shotton, and Corrado]{hu2023gaia}
Anthony Hu, Lloyd Russell, Hudson Yeo, Zak Murez, George Fedoseev, Alex Kendall, Jamie Shotton, and Gianluca Corrado.
\newblock Gaia-1: A generative world model for autonomous driving.
\newblock \emph{arXiv e-prints}, pages arXiv--2309, 2023.

\bibitem[Hwang et~al.(2024)Hwang, Kwak, Choi, Zhang, and Lee]{hwang2024fine}
Inwoo Hwang, Yunhyeok Kwak, Suhyung Choi, Byoung-Tak Zhang, and Sanghack Lee.
\newblock Fine-grained causal dynamics learning with quantization for improving robustness in reinforcement learning.
\newblock \emph{Proceedings of Machine Learning Research}, 235:\penalty0 20842--20870, 2024.

\bibitem[Hyvarinen and Morioka(2016)]{hyvarinen2016unsupervised}
Aapo Hyvarinen and Hiroshi Morioka.
\newblock Unsupervised feature extraction by time-contrastive learning and nonlinear ica.
\newblock \emph{Advances in neural information processing systems}, 29, 2016.

\bibitem[Hyv{\"a}rinen and Pajunen(1999)]{hyvarinen1999nonlinear}
Aapo Hyv{\"a}rinen and Petteri Pajunen.
\newblock Nonlinear independent component analysis: Existence and uniqueness results.
\newblock \emph{Neural networks}, 12\penalty0 (3):\penalty0 429--439, 1999.

\bibitem[Hyvarinen et~al.(2019)Hyvarinen, Sasaki, and Turner]{hyvarinen2019nonlinear}
Aapo Hyvarinen, Hiroaki Sasaki, and Richard Turner.
\newblock Nonlinear ica using auxiliary variables and generalized contrastive learning.
\newblock In \emph{The 22nd international conference on artificial intelligence and statistics}, pages 859--868. PMLR, 2019.

\bibitem[Hyvärinen et~al.(2023)Hyvärinen, Khemakhem, and Morioka]{HYVARINEN2023100844}
Aapo Hyvärinen, Ilyes Khemakhem, and Hiroshi Morioka.
\newblock Nonlinear independent component analysis for principled disentanglement in unsupervised deep learning.
\newblock \emph{Patterns}, 4\penalty0 (10):\penalty0 100844, 2023.
\newblock ISSN 2666-3899.
\newblock \doi{https://doi.org/10.1016/j.patter.2023.100844}.
\newblock URL \url{https://www.sciencedirect.com/science/article/pii/S2666389923002234}.

\bibitem[Jiang et~al.(2024)Jiang, Deng, Singh, Lee, and Ahn]{jiang2024slot}
Jindong Jiang, Fei Deng, Gautam Singh, Minseung Lee, and Sungjin Ahn.
\newblock Slot state space models.
\newblock \emph{Advances in Neural Information Processing Systems}, 37:\penalty0 11602--11633, 2024.

\bibitem[Khemakhem et~al.(2020)Khemakhem, Kingma, Monti, and Hyvarinen]{khemakhem2020variational}
Ilyes Khemakhem, Diederik Kingma, Ricardo Monti, and Aapo Hyvarinen.
\newblock Variational autoencoders and nonlinear ica: A unifying framework.
\newblock In \emph{International conference on artificial intelligence and statistics}, pages 2207--2217. PMLR, 2020.

\bibitem[Kim and Mnih(2018)]{kim2018disentangling}
Hyunjik Kim and Andriy Mnih.
\newblock Disentangling by factorising.
\newblock In \emph{International conference on machine learning}, pages 2649--2658. PMLR, 2018.

\bibitem[Kingma and Welling(2013)]{vae}
Diederik~P Kingma and Max Welling.
\newblock Auto-encoding variational bayes.
\newblock \emph{arXiv preprint arXiv:1312.6114}, 2013.

\bibitem[Kipf et~al.(2018)Kipf, Fetaya, Wang, Welling, and Zemel]{NRI}
Thomas Kipf, Ethan Fetaya, Kuan-Chieh Wang, Max Welling, and Richard Zemel.
\newblock Neural relational inference for interacting systems.
\newblock In \emph{International conference on machine learning}, pages 2688--2697. Pmlr, 2018.

\bibitem[Lachapelle and Lacoste-Julien(2022)]{lachapelle2022partial}
Sebastien Lachapelle and Simon Lacoste-Julien.
\newblock Partial disentanglement via mechanism sparsity.
\newblock In \emph{UAI 2022 Workshop on Causal Representation Learning}, 2022.

\bibitem[Lachapelle et~al.(2022{\natexlab{a}})Lachapelle, Deleu, Mahajan, Mitliagkas, Bengio, Lacoste-Julien, and Bertrand]{lachapelle2022synergies}
S{\'e}bastien Lachapelle, Tristan Deleu, Divyat Mahajan, Ioannis Mitliagkas, Yoshua Bengio, Simon Lacoste-Julien, and Quentin Bertrand.
\newblock Synergies between disentanglement and sparsity: Generalization and identifiability in multi-task learning.
\newblock \emph{arXiv preprint arXiv:2211.14666}, 2022{\natexlab{a}}.

\bibitem[Lachapelle et~al.(2022{\natexlab{b}})Lachapelle, Rodriguez, Sharma, Everett, PRIOL, Lacoste, and Lacoste-Julien]{lachapelle22a}
Sebastien Lachapelle, Pau Rodriguez, Yash Sharma, Katie~E Everett, R{\'e}mi~LE PRIOL, Alexandre Lacoste, and Simon Lacoste-Julien.
\newblock Disentanglement via mechanism sparsity regularization: A new principle for nonlinear {ICA}.
\newblock In Bernhard Schölkopf, Caroline Uhler, and Kun Zhang, editors, \emph{Proceedings of the First Conference on Causal Learning and Reasoning}, volume 177 of \emph{Proceedings of Machine Learning Research}, pages 428--484. PMLR, 11--13 Apr 2022{\natexlab{b}}.
\newblock URL \url{https://proceedings.mlr.press/v177/lachapelle22a.html}.

\bibitem[Lachapelle et~al.(2023)Lachapelle, Mahajan, Mitliagkas, and Lacoste-Julien]{lachapelle2023additive}
S{\'e}bastien Lachapelle, Divyat Mahajan, Ioannis Mitliagkas, and Simon Lacoste-Julien.
\newblock Additive decoders for latent variables identification and cartesian-product extrapolation.
\newblock \emph{Advances in Neural Information Processing Systems}, 36:\penalty0 25112--25150, 2023.

\bibitem[Lachapelle et~al.(2024)Lachapelle, L{\'o}pez, Sharma, Everett, Priol, Lacoste, and Lacoste-Julien]{lachapelle2024nonparametric}
S{\'e}bastien Lachapelle, Pau~Rodr{\'\i}guez L{\'o}pez, Yash Sharma, Katie Everett, R{\'e}mi~Le Priol, Alexandre Lacoste, and Simon Lacoste-Julien.
\newblock Nonparametric partial disentanglement via mechanism sparsity: Sparse actions, interventions and sparse temporal dependencies.
\newblock \emph{arXiv preprint arXiv:2401.04890}, 2024.

\bibitem[Lam et~al.(2023)Lam, Sanchez-Gonzalez, Willson, Wirnsberger, Fortunato, Alet, Ravuri, Ewalds, Eaton-Rosen, Hu, Merose, Hoyer, Holland, Vinyals, Stott, Pritzel, Mohamed, and Battaglia]{graphcast}
Remi Lam, Alvaro Sanchez-Gonzalez, Matthew Willson, Peter Wirnsberger, Meire Fortunato, Ferran Alet, Suman Ravuri, Timo Ewalds, Zach Eaton-Rosen, Weihua Hu, Alexander Merose, Stephan Hoyer, George Holland, Oriol Vinyals, Jacklynn Stott, Alexander Pritzel, Shakir Mohamed, and Peter Battaglia.
\newblock Learning skillful medium-range global weather forecasting.
\newblock \emph{Science}, 382\penalty0 (6677):\penalty0 1416--1421, 2023.
\newblock \doi{10.1126/science.adi2336}.
\newblock URL \url{https://www.science.org/doi/abs/10.1126/science.adi2336}.

\bibitem[Lee(2012)]{intro_to_smooth_manifolds}
John~M. Lee.
\newblock \emph{Introduction to smooth manifolds}.
\newblock Graduate texts in mathematics, 218. Springer, New York, second edition. edition, 2012.
\newblock ISBN 9781489994752.

\bibitem[Lei et~al.(2023)Lei, Sch{\"o}lkopf, and Posner]{vcd}
Anson Lei, Bernhard Sch{\"o}lkopf, and Ingmar Posner.
\newblock Variational causal dynamics: Discovering modular world models from interventions.
\newblock \emph{Transactions on Machine Learning Research}, 2023.

\bibitem[Lei et~al.(2025)Lei, Posner, and Sch{\"o}lkopf]{spartan}
Anson Lei, Ingmar Posner, and Bernhard Sch{\"o}lkopf.
\newblock Spartan: A sparse transformer learning local causation.
\newblock \emph{Advances in Neural Information Processing Systems}, 2025.

\bibitem[Li et~al.(2024)Li, Pan, and Bareinboim]{li2024disentangled}
Adam Li, Yushu Pan, and Elias Bareinboim.
\newblock Disentangled representation learning in non-markovian causal systems.
\newblock In \emph{The Thirty-eighth Annual Conference on Neural Information Processing Systems}, 2024.
\newblock URL \url{https://openreview.net/forum?id=uLGyoBn7hm}.

\bibitem[Li et~al.(2019)Li, Hooi, and Lee]{liidentifying}
Shen Li, Bryan Hooi, and Gim~Hee Lee.
\newblock Identifying through flows for recovering latent representations.
\newblock In \emph{International Conference on Learning Representations}, 2019.

\bibitem[Li et~al.(2025{\natexlab{a}})Li, Zheng, Xiong, Wang, and Huang]{li2025understanding}
Yuke Li, Yujia Zheng, Tianyi Xiong, Zhenyi Wang, and Heng Huang.
\newblock Understanding catastrophic interference on the identifibility of latent representations.
\newblock \emph{arXiv preprint arXiv:2509.23027}, 2025{\natexlab{a}}.

\bibitem[Li et~al.(2020)Li, Torralba, Anandkumar, Fox, and Garg]{li2020causal}
Yunzhu Li, Antonio Torralba, Anima Anandkumar, Dieter Fox, and Animesh Garg.
\newblock Causal discovery in physical systems from videos.
\newblock \emph{Advances in Neural Information Processing Systems}, 33:\penalty0 9180--9192, 2020.

\bibitem[Li et~al.(2025{\natexlab{b}})Li, Fu, Huang, Shen, Cai, Sun, Chen, and Zhang]{li2025towards}
Zijian Li, Minghao Fu, Junxian Huang, Yifan Shen, Ruichu Cai, Yuewen Sun, Guangyi Chen, and Kun Zhang.
\newblock Towards identifiability of hierarchical temporal causal representation learning.
\newblock In \emph{The Thirty-ninth Annual Conference on Neural Information Processing Systems}, 2025{\natexlab{b}}.
\newblock URL \url{https://openreview.net/forum?id=0J0y9vSCWf}.

\bibitem[Li et~al.(2025{\natexlab{c}})Li, Shen, Zheng, Cai, Song, Gong, Chen, and Zhang]{li2025}
Zijian Li, Yifan Shen, Kaitao Zheng, Ruichu Cai, Xiangchen Song, Mingming Gong, Guangyi Chen, and Kun Zhang.
\newblock On the identification of temporal causal representation with instantaneous dependence.
\newblock International Conference on Learning Representations, ICLR, 2025{\natexlab{c}}.

\bibitem[Lippe et~al.(2022)Lippe, Magliacane, L{\"o}we, Asano, Cohen, and Gavves]{lippe2022citris}
Phillip Lippe, Sara Magliacane, Sindy L{\"o}we, Yuki~M Asano, Taco Cohen, and Stratis Gavves.
\newblock Citris: Causal identifiability from temporal intervened sequences.
\newblock In \emph{International Conference on Machine Learning}, pages 13557--13603. PMLR, 2022.

\bibitem[Locatello et~al.(2019)Locatello, Bauer, Lučić, Rätsch, Gelly, Schölkopf, and Bachem]{disentanglement-impossible-without-structure-locatello}
Francesco Locatello, Stefan Bauer, Mario Lučić, Gunnar Rätsch, Sylvain Gelly, Bernhard Schölkopf, and Olivier~Frederic Bachem.
\newblock Challenging common assumptions in the unsupervised learning of disentangled representations.
\newblock In \emph{International Conference on Machine Learning}, 2019.
\newblock URL \url{http://proceedings.mlr.press/v97/locatello19a.html}.
\newblock Best Paper Award.

\bibitem[L{\"o}we et~al.(2022)L{\"o}we, Madras, Zemel, and Welling]{ACD}
Sindy L{\"o}we, David Madras, Richard Zemel, and Max Welling.
\newblock Amortized causal discovery: Learning to infer causal graphs from time-series data.
\newblock In \emph{Conference on Causal Learning and Reasoning}, pages 509--525. PMLR, 2022.

\bibitem[Macfarlane and Bonnet(2025)]{lpn}
Matthew~V Macfarlane and Clement Bonnet.
\newblock Searching latent program spaces, 2025.
\newblock URL \url{https://arxiv.org/abs/2411.08706}.

\bibitem[Mathieu et~al.(2019)Mathieu, Rainforth, Siddharth, and Teh]{mathieu2019disentangling}
Emile Mathieu, Tom Rainforth, Nana Siddharth, and Yee~Whye Teh.
\newblock Disentangling disentanglement in variational autoencoders.
\newblock In \emph{International conference on machine learning}, pages 4402--4412. PMLR, 2019.

\bibitem[Pearl(2000)]{pearl}
Judea Pearl.
\newblock \emph{Causality : models, reasoning, and inference}.
\newblock Cambridge University Press, Cambridge, 2nd edition edition, 2000.
\newblock ISBN 1-139-63780-0.

\bibitem[Pervez et~al.(2024)Pervez, Locatello, and Gavves]{pervezmechanistic}
Adeel Pervez, Francesco Locatello, and Stratis Gavves.
\newblock Mechanistic neural networks for scientific machine learning.
\newblock In Ruslan Salakhutdinov, Zico Kolter, Katherine Heller, Adrian Weller, Nuria Oliver, Jonathan Scarlett, and Felix Berkenkamp, editors, \emph{Proceedings of the 41st International Conference on Machine Learning}, volume 235 of \emph{Proceedings of Machine Learning Research}, pages 40484--40501. PMLR, 21--27 Jul 2024.
\newblock URL \url{https://proceedings.mlr.press/v235/pervez24a.html}.

\bibitem[Pitis et~al.(2020)Pitis, Creager, and Garg]{Coda}
Silviu Pitis, Elliot Creager, and Animesh Garg.
\newblock Counterfactual data augmentation using locally factored dynamics.
\newblock \emph{Advances in Neural Information Processing Systems}, 33:\penalty0 3976--3990, 2020.

\bibitem[Raissi et~al.(2019)Raissi, Perdikaris, and Karniadakis]{PINNs}
M.~Raissi, P.~Perdikaris, and G.E. Karniadakis.
\newblock Physics-informed neural networks: A deep learning framework for solving forward and inverse problems involving nonlinear partial differential equations.
\newblock \emph{Journal of Computational Physics}, 378:\penalty0 686--707, 2019.
\newblock ISSN 0021-9991.
\newblock \doi{https://doi.org/10.1016/j.jcp.2018.10.045}.
\newblock URL \url{https://www.sciencedirect.com/science/article/pii/S0021999118307125}.

\bibitem[Rezende and Viola(2018)]{rezende2018taming}
Danilo~Jimenez Rezende and Fabio Viola.
\newblock Taming vaes.
\newblock \emph{arXiv preprint arXiv:1810.00597}, 2018.

\bibitem[Rudy et~al.(2017)Rudy, Brunton, Proctor, and Kutz]{sindy-pde}
Samuel~H. Rudy, Steven~L. Brunton, Joshua~L. Proctor, and J.~Nathan Kutz.
\newblock Data-driven discovery of partial differential equations.
\newblock \emph{Science Advances}, 3\penalty0 (4):\penalty0 e1602614, 2017.
\newblock \doi{10.1126/sciadv.1602614}.
\newblock URL \url{https://www.science.org/doi/abs/10.1126/sciadv.1602614}.

\bibitem[Sch\"{o}lkopf et~al.(2012)Sch\"{o}lkopf, Janzing, Peters, Sgouritsa, Zhang, and Mooij]{oncausal}
Bernhard Sch\"{o}lkopf, Dominik Janzing, Jonas Peters, Eleni Sgouritsa, Kun Zhang, and Joris Mooij.
\newblock On causal and anticausal learning.
\newblock In \emph{Proceedings of the 29th International Coference on International Conference on Machine Learning}, ICML'12, page 459–466, Madison, WI, USA, 2012. Omnipress.
\newblock ISBN 9781450312851.

\bibitem[Sch{\"o}lkopf et~al.(2021)Sch{\"o}lkopf, Locatello, Bauer, Ke, Kalchbrenner, Goyal, and Bengio]{scholkopf2021toward}
Bernhard Sch{\"o}lkopf, Francesco Locatello, Stefan Bauer, Nan~Rosemary Ke, Nal Kalchbrenner, Anirudh Goyal, and Yoshua Bengio.
\newblock Toward causal representation learning.
\newblock \emph{Proceedings of the IEEE}, 109\penalty0 (5):\penalty0 612--634, 2021.

\bibitem[Schug et~al.(2024)Schug, Kobayashi, Akram, Wolczyk, Proca, von Oswald, Pascanu, Sacramento, and Steger]{schug2024discovering}
Simon Schug, Seijin Kobayashi, Yassir Akram, Maciej Wolczyk, Alexandra~Maria Proca, Johannes von Oswald, Razvan Pascanu, Joao Sacramento, and Angelika Steger.
\newblock Discovering modular solutions that generalize compositionally.
\newblock In \emph{The Twelfth International Conference on Learning Representations}, 2024.
\newblock URL \url{https://openreview.net/forum?id=H98CVcX1eh}.

\bibitem[Seitzer et~al.(2021)Seitzer, Sch{\"o}lkopf, and Martius]{seitzer2021causal}
Maximilian Seitzer, Bernhard Sch{\"o}lkopf, and Georg Martius.
\newblock Causal influence detection for improving efficiency in reinforcement learning.
\newblock \emph{Advances in Neural Information Processing Systems}, 34:\penalty0 22905--22918, 2021.

\bibitem[Sohn et~al.(2015)Sohn, Lee, and Yan]{conditionalVAE}
Kihyuk Sohn, Honglak Lee, and Xinchen Yan.
\newblock Learning structured output representation using deep conditional generative models.
\newblock \emph{Advances in neural information processing systems}, 28, 2015.

\bibitem[Song et~al.(2023)Song, Yao, Fan, Dong, Chen, Niebles, Xing, and Zhang]{song2023}
Xiangchen Song, Weiran Yao, Yewen Fan, Xinshuai Dong, Guangyi Chen, Juan~Carlos Niebles, Eric Xing, and Kun Zhang.
\newblock Temporally disentangled representation learning under unknown nonstationarity.
\newblock In \emph{Thirty-seventh Conference on Neural Information Processing Systems}, 2023.
\newblock URL \url{https://openreview.net/forum?id=V8GHCGYLkf}.

\bibitem[Spirtes et~al.(2001)Spirtes, Glymour, and Scheines]{SpirtesBook}
Peter Spirtes, Clark Glymour, and Richard Scheines.
\newblock \emph{Causation, Prediction, and Search}.
\newblock The MIT Press, 01 2001.
\newblock ISBN 9780262284158.
\newblock \doi{10.7551/mitpress/1754.001.0001}.
\newblock URL \url{https://doi.org/10.7551/mitpress/1754.001.0001}.

\bibitem[Van~Steenkiste et~al.(2018)Van~Steenkiste, Greff, Chang, and Schmidhuber]{van2018relational}
Sjoerd Van~Steenkiste, Klaus Greff, Michael Chang, and J{\"u}rgen Schmidhuber.
\newblock Relational neural expectation maximization: Unsupervised discovery of objects and their interactions.
\newblock In \emph{6th International Conference on Learning Representations, ICLR 2018-Conference Track Proceedings}. International Conference on Learning Representations, ICLR, 2018.

\bibitem[Var{\i}c{\i} et~al.(2024)Var{\i}c{\i}, Acart{\"u}rk, Shanmugam, and Tajer]{varici2024linear}
Burak Var{\i}c{\i}, Emre Acart{\"u}rk, Karthikeyan Shanmugam, and Ali Tajer.
\newblock Linear causal representation learning from unknown multi-node interventions.
\newblock \emph{Advances in Neural Information Processing Systems}, 37:\penalty0 111614--111648, 2024.

\bibitem[von K{\"u}gelgen et~al.(2023)von K{\"u}gelgen, Besserve, Wendong, Gresele, Keki{\'c}, Bareinboim, Blei, and Sch{\"o}lkopf]{von2023nonparametric}
Julius von K{\"u}gelgen, Michel Besserve, Liang Wendong, Luigi Gresele, Armin Keki{\'c}, Elias Bareinboim, David Blei, and Bernhard Sch{\"o}lkopf.
\newblock Nonparametric identifiability of causal representations from unknown interventions.
\newblock \emph{Advances in Neural Information Processing Systems}, 36:\penalty0 48603--48638, 2023.

\bibitem[Wang et~al.(2024)Wang, Li, Zhang, and Wang]{MetaCARD}
Min Wang, Xin Li, Leiji Zhang, and Mingzhong Wang.
\newblock Metacard: Meta-reinforcement learning with task uncertainty feedback via decoupled context-aware reward and dynamics components.
\newblock \emph{Proceedings of the AAAI Conference on Artificial Intelligence}, 38\penalty0 (14):\penalty0 15555--15562, Mar. 2024.
\newblock \doi{10.1609/aaai.v38i14.29482}.
\newblock URL \url{https://ojs.aaai.org/index.php/AAAI/article/view/29482}.

\bibitem[Wen et~al.(2024)Wen, Tseng, Peng, and Zhang]{DreamToAdapt}
Lu~Wen, Eric~H. Tseng, Huei Peng, and Songan Zhang.
\newblock Dream to adapt: Meta reinforcement learning by latent context imagination and mdp imagination.
\newblock \emph{IEEE Robotics and Automation Letters}, 9\penalty0 (11):\penalty0 9701--9708, 2024.
\newblock \doi{10.1109/LRA.2024.3417114}.

\bibitem[Yao et~al.(2024{\natexlab{a}})Yao, Muller, and Locatello]{yao2024marrying}
Dingling Yao, Caroline Muller, and Francesco Locatello.
\newblock Marrying causal representation learning with dynamical systems for science.
\newblock \emph{Advances in Neural Information Processing Systems}, 37:\penalty0 71705--71736, 2024{\natexlab{a}}.

\bibitem[Yao et~al.(2024{\natexlab{b}})Yao, Rancati, Cadei, Fumero, and Locatello]{yao2024unifying}
Dingling Yao, Dario Rancati, Riccardo Cadei, Marco Fumero, and Francesco Locatello.
\newblock Unifying causal representation learning with the invariance principle.
\newblock In \emph{38th Conference on Neural Information Processing Systems}, volume~37, 2024{\natexlab{b}}.

\bibitem[Yao et~al.(2024{\natexlab{c}})Yao, Xu, Lachapelle, Magliacane, Taslakian, Martius, K{\"u}gelgen, and Locatello]{yao2024multi}
Dingling Yao, Danru Xu, S{\'e}bastien Lachapelle, Sara Magliacane, Perouz Taslakian, Georg Martius, Julius~von K{\"u}gelgen, and Francesco Locatello.
\newblock Multi-view causal representation learning with partial observability.
\newblock In \emph{12th International Conference on Learning Representations}, 2024{\natexlab{c}}.

\bibitem[Yao et~al.(2022)Yao, Chen, and Zhang]{yao2022temporally}
Weiran Yao, Guangyi Chen, and Kun Zhang.
\newblock Temporally disentangled representation learning.
\newblock \emph{Advances in Neural Information Processing Systems}, 35:\penalty0 26492--26503, 2022.

\bibitem[Zhai et~al.(2023)Zhai, Likhomanenko, Littwin, Busbridge, Ramapuram, Zhang, Gu, and Susskind]{zhai2023stabilizing}
Shuangfei Zhai, Tatiana Likhomanenko, Etai Littwin, Dan Busbridge, Jason Ramapuram, Yizhe Zhang, Jiatao Gu, and Joshua~M Susskind.
\newblock Stabilizing transformer training by preventing attention entropy collapse.
\newblock In \emph{International Conference on Machine Learning}, pages 40770--40803. PMLR, 2023.

\bibitem[Zhao et~al.(2025)Zhao, Li, Zhang, Wang, Faccio, Schmidhuber, and Yang]{zhao2025curiouscausalityseekingagentslearn}
Zhiyu Zhao, Haoxuan Li, Haifeng Zhang, Jun Wang, Francesco Faccio, Jürgen Schmidhuber, and Mengyue Yang.
\newblock Curious causality-seeking agents learn meta causal world, 2025.
\newblock URL \url{https://arxiv.org/abs/2506.23068}.

\end{thebibliography}
\appendix
\newpage

\section{Useful Definitions}
\label{appendix:useful_definitions}

\begin{cdefinition}[Identifiable Models \citep{khemakhem2020variational}]
    \label{definition_identifiable}
    Two systems $\mathcal{S}$, $\hat{\mathcal{S}}$, are said to be identifiable if and only if
    \begin{equation}
        \forall (\btheta, \hat{\btheta}, \bx_0 ): \quad 
        \btau = \hat{\btau} \implies \btheta = \hat{\btheta}.
    \end{equation}
\end{cdefinition}

\begin{cdefinition}[Identifiable up to Equivalence]
    \label{definition_identifiable_equivalence}
    Two systems $\mathcal{S}$, $\hat{\mathcal{S}}$, are said to be identifiable up to some equivalence operator if and only if
    \begin{equation}
        \forall (\btheta, \hat{\btheta}, \bx_0): \quad 
        \btau = \hat{\btau} \implies \btheta \sim_{equiv} \hat{\btheta}.
    \end{equation}
\end{cdefinition}

\begin{cdefinition}[Diffeomorphism]
\label{definition:diffeomorphism}
A map $\bf: \mathbb{R}^M \to \mathbb{R}^N$ between two differentiable manifolds $\mathcal{M}$ and $\mathcal{N}$ is called a \emph{diffeomorphism} if it satisfies the following equivalent conditions:
\begin{enumerate}
    \item $\bf$ is a one-to-one continuously differentiable mapping of $\mathcal{M}$ onto $\mathcal{N}$ with a continuously differentiable inverse mapping $\bf^{-1}$.
    \item $\bf$ is a $C^1$ bijection (see definition \ref{definition:ck_functions}), meaning that:
    \begin{itemize}
        \item $\bf$ is injective (one-to-one): $\bf(\bm{z}_1) = f(\bm{z}_2)$ implies $\bm{z}_1 = \bm{z}_2$ for all $\bm{z}_1, \bm{z}_2 \in \mathcal{M}$,
        \item $\bf$ is surjective: for every $\bm{x} \in \mathcal{N}$, there exists a $\bm{z} \in \mathcal{M}$ such that $\bm{f}(\bm{z}) = \bm{x}$.
    \end{itemize}
\end{enumerate}
\end{cdefinition}

\begin{cdefinition}[Surjectivity]
\label{definition:surjective}
    A function, $\bf$, with codomain and domain $\bm{y}\in \mathcal{Y}, \bm{x} \in \mathcal{X}$, is surjective if for every $\bm{y} \in \mathcal{Y}$ there exists $\bm{x} \in \mathcal{X}$ such that $\bm{y}=\bf(\bm{x})$. 
\end{cdefinition}

\begin{cdefinition}[$C^k$ Functions]
\label{definition:ck_functions}
A function $\bf: U \to \mathbb{R}^m$ defined on an open subset $U \subset \mathbb{R}^n$ is said to be of class $C^k$ if it is $k$-times continuously differentiable on $U$. In other words, the partial derivatives of $\bf$ up to order $k$ on $U$ are continuous functions.
\end{cdefinition}

\begin{cdefinition}[Equiv. up to Elementwise Perm. \& Diffeomorphism \citep{lachapelle2022synergies}]
    \label{definition_disentanglement_diffeomorphism}
    A representation, $\hat{\boldsymbol{\theta}}$, is said to be equivalent \textit{with respect to} a ground truth representation, $\boldsymbol{\theta}$ up to permutation and elementwise diffeomorphism (see definition \ref{definition:diffeomorphism}) if and only if there exists a permutation matrix, $P$, and a set of one-to-one diffeomorphisms, $c_i(\cdot)$, such that 
    \begin{equation}
        \boldsymbol{\theta} \sim_\text{diff} \hat{\boldsymbol{\theta}} \implies
        \forall (i, \hat{\boldsymbol{\theta}},\boldsymbol{\theta}): \theta_i = c_i([P\hat{\boldsymbol{\theta}}]_i).
    \end{equation} 
\end{cdefinition}

\begin{cdefinition}[Disentanglement]
    Two representations, $\boldsymbol{\theta}$, $\hat{\boldsymbol{\theta}}$, are provably disentangled if they are identifiable up to elementwise diffeomorphism and permutation equivalence. 
    \label{definition_disentanglement}
\end{cdefinition}

\begin{cdefinition}[Subsets of Graphs]
\begin{align}
    \mathcal{G}_a \subseteq \mathcal{G}_b \Leftrightarrow \big[A_{\mathcal{G}_b}\big]_{i,j}=0 \implies \big[A_{\mathcal{G}_a}\big]_{i,j}=0.
\end{align}
Graphs subsets are defined as follows. Every zero in the adjacency matrix of the superset graph must also be a zero in the adjacency matrix of the subset graph.
\label{definition:graph_subset}
\end{cdefinition}

\begin{cdefinition}[Graph Arithmetic]

    Graph arithmetic allows the product and sum of Jacobians (and consequently their respective graphs) to be expressed in the adjacency matrix domain. Graph operations are defined via Boolean matrix algebra. All summation operations are overloaded by element-wise OR operations, and all scalar multiplication with the AND operation. 
    \label{definition:graph_arithmetic}
\end{cdefinition}

\begin{cdefinition}[Right and Left Graph Consistency]
    Any graph $\mathcal{C}$ is said to be left graph consistent with another graph $\mathcal{G}$ if and only if the graph is invariant to adjacency matrix multiplication:
    \begin{equation}
        A_\mathcal{G}=A_\mathcal{C} A_\mathcal{G}.
    \end{equation}
    The same applies for right consistency:
    \begin{equation}
        A_\mathcal{G}= A_\mathcal{G}A_\mathcal{C}.
    \end{equation}
    \label{definition_graph_consistency}
\end{cdefinition}

\begin{cdefinition}[Real Subset]
    \label{definition:real_subset} The real subset relative to a directed acyclic graph, $\mathcal{G}$, with $n$ input nodes and $m$ output nodes
    is denoted as $\mathbb{R}_\mathcal{G}$. It is defined as the set of real matrices where all elements and only the elements with edges in $\mathcal{G}$ can take any non-zero value, and all elements with no edges in $\mathcal{G}$ are zero.
    \begin{align}
        M \in \mathbb{R}_\mathcal{G} \implies \Big( M_{i,j} \neq 0 \iff (i,j) \in \mathcal{G} \Big).
    \end{align}
    An alternative definition defines $\mathbb{R}_\mathcal{G}$ as the set of matrices for which there exists a $\bm{\lambda}$ which satisfies:
\begin{align}
    \mathbb{R}_\gG : \quad \{\ M \vert \exists \boldsymbol{\lambda} \in \mathbb{R}^{\Vert \mathcal{G} \Vert_0}  \text{ s.t. } \bm{M} = 
    \sum_{(i,j) \in \gG} \lambda_{i,j} \bm{e}_{i,j} \}.
    \label{equation:set_of_all_jacobians}
\end{align}
Where $\bm{e}_{i,j}$ is an indicator matrix.
\end{cdefinition}

\newpage
\section{Formal Proofs and Derivation of Graphical Criteria}
\label{appendix:proof_of_theorem}
\new{\localtableofcontents}
\vspace{0.5cm}

\new{%
This appendix provides the formal mathematical foundations and proofs for the identifiability of system parameters in latent dynamical systems. We recapitulate the main theorem and problem setting from the main text to keep the proof self-contained. The proof of the main theorem is found in section \ref{sec:proof:main}. Below, we provide the precise setting and the required lemmas. We consider a non-linear deterministic Markovian dynamical system defined as the tuple $\mathcal{S}=(\Theta,\mathcal{X}_0,\mathcal{X},\mathcal{F},\mathcal{G})$ with state variables $\bx_t$ and dynamical parameters $\btheta$, following the evolution:
\begin{align}
\label{equation:dynamical_systems_form}
\boldsymbol{x}_{t+1}=\bf(\bx_t, \btheta) = \bf_{\theta}(\boldsymbol{x}_t),
\quad \btheta \in \Theta,
\quad \boldsymbol{x} \in \mathcal{X} , \quad \boldsymbol{x}_0 \in \mathcal{X}_0 \subseteq \mathcal{X},
\\
\boldsymbol{\tau}=\bh(\bx_0, \btheta)=[\btx_0,\ft(\btx_0), (\ft \circ \ft)(\btx_0),...,\ft^T (\btx_0)], \quad
\bx_0, \btheta = \bh^{-1}(\btau),
\label{equation:dynamics_model_form}
\end{align}
Where $\bm{f}(\cdot)$ is the dynamics model which describes the system's evolution, $\btau$ is a trajectory generated by the autoregressive application of the dynamics model. The dynamics model is faithful to a directed acyclic graph (DAG) $\mathcal{G}$ in modelling how parameters influence the state variables. $\bm{h}(\cdot)$ is a trajectory \emph{decoder} which encapsulates the autoregressive process, and $\bm{h}^{-1}(\cdot)$ is the trajectory \emph{encoder} which recovers a latent parameter representation from a trajectory.
Our primary objective is to establish the conditions under which a learnt representation of dynamics $\bht$ (belonging to a learnt approximation $\hat{\mathcal{S}}$) is equivalent to the ground truth parameters $\btheta$ up to permutation and element-wise diffeomorphism.}
\setcounter{assumption}{0}
\new{%
For uniqueness and without loss of generality, we enforce assumptions \ref{assumption:existence} and \ref{assumption:obs_eq}, which in combination also require the bijectivity of $\bh(\cdot)$ and $\bh^{-1}(\cdot)$ with respect to their image. We further assume that the parameters sparsely affect the state distribution via a causal graph $\mathcal{G}$.
\begin{assumption}[Existence and Parameter Influence \citep{yao2024marrying}]
    \label{assumption:existence}
    For every $\bx_0 \in \mathcal{X}_0, \: \btheta\in\Theta$ there exists a unique trajectory, $\boldsymbol{\tau}$, over horizon T, satisfying $\boldsymbol{x}_{t+1}=
    \bf_\theta (\boldsymbol{x}_t)$. Equivalent to requiring $\bh (\bx_0, \bt)$ to be injective.
\end{assumption}
\begin{assumption}[Observational Equivalence]
\label{assumption:obs_eq}
Assume that the two systems, the ground-truth $\mathcal{S}$ and learnt approximation $\hat{\mathcal{S}}$, are \emph{observationally equivalent}, which posits that the modelled systems are equivalent up to invertible parameter representations. Observational equivalence requires, ${
    \forall \boldsymbol{x}_0, \boldsymbol{\theta}: \quad \exists \hat{\boldsymbol{\theta}} \text{ s.t. } \bh(\boldsymbol{x}_0, \boldsymbol{\theta}) = \hat{\bh}(\boldsymbol{x}_0,\hat{\boldsymbol{\theta}})}$ (Equivalently,  $\hat{\bh} (\cdot)$ is surjective).
\end{assumption}
\begin{assumption}[The Transition Model is Markov to a DAG \citep{lachapelle2024nonparametric}]
\label{assumption:markov}
The transition models of $\mathcal{S}$, and $\hat{\mathcal{S}}$, are causally and \emph{faithfully} \new{\citep{pearl,SpirtesBook}} related to directed acyclic graphs (DAG), $\mathcal{G}$, and $\hat{\mathcal{G}}$, which fully define the dependencies between the \textbf{output elements and latent parameters}. The edges of the graph are defined through the non-zero entries in the Jacobian of the transition model, requiring that the model and its first derivative are continuous.
Simply, the following causal equivalence is defined, where $\bt_{\text{Pa}_i^{\mathcal{G}}}$ denotes the subset of parameters that are parents of output $i$ in graph $\mathcal{G}$:
$%
    f_i (\boldsymbol{x}_t,\boldsymbol{\theta})
    \equiv
    f_i(\boldsymbol{x}_t,\boldsymbol{\theta}_{\text{Pa}_i^{\mathcal{G}}})
$.
\begin{equation}
(i,j) \notin \mathcal{G} \implies
\frac{\partial (\bf(\boldsymbol{x}_t,\boldsymbol{\theta}))_i}{\partial \theta_j} = 0 \quad \forall \boldsymbol{x}_t, \mathbf{\theta}.
\end{equation}
\end{assumption}
}
These assumptions establish a framework in which multiple parametrisations can yield identical observations, allowing the study of the conditions under which additional assumptions lead to disentanglement.
We begin by introducing the concept of an entanglement mapping and an entanglement graph as necessary prerequisites to understand the theory. Subsequently, we build intermediate results in forms of lemmas. Section \ref{sec:proof:main} combines the lemmas into the proof of Theorem 1. Finally, section \ref{sec:proof:local_graphs} extends the main result to state-dependent local causal graphs and demonstrates that state-dependent sparsity improves identifiability guarantees.

\subsection{Entanglement Map and Entanglement Graph}

The trajectory encoder, $\btheta = \bh^{-1}(\btau)$, allows us to define an \textit{entanglement map}, $\bv(\cdot)$, which relates the latent representations learnt by $\hat{\mathcal{S}}$, with the parameters in $\hat{\mathcal{S}}$. Given that the two systems are modelling the same underlying system, it follows that for each parameter of the ground truth system, $\btheta$, there must exist a learnt representation $\bht$ such that the modelled process is equivalent $\bf(\boldsymbol{x}_0,\boldsymbol{{\theta}})=\hat{\bf}(\boldsymbol{x}_0,\bv(\btheta,\boldsymbol{x}_0))$. The equivalence of the system is formalised in the next section as assumption \ref{assumption:obs_eq}. Observational equivalence requires the injectivity of the \emph{entanglement map}.
\begin{align}
    \hat{\boldsymbol{\theta}}=\bv(\boldsymbol{\theta},\boldsymbol{x}_0)= (\hat{\bh}^{-1} \circ \bh)
(\boldsymbol{\theta},\boldsymbol{x}_0).
\end{align}

The dependence of the entanglement mapping upon a trajectory's starting state, $\bm{x}_0$, is a consequence of the application to dynamical domains. This dependence allows the dynamical representation to become entangled with the state representation.
The entanglement map is also associated with an \textit{entanglement graph}, $\mathcal{V}$, which quantifies the dependencies between the two representations using the same fundamental definition used for the system graphs $\mathcal{G}$, and $\hat{\mathcal{G}}$. The \emph{entanglement graph} denotes which system parameters, $\btheta$, are required to map to the learnt representation. Equivalently, $\bht_i = v_i( \btheta_{\text{Pa}_i^{\mathcal{V}}}, \bx_0)$ where $\btheta_{\text{Pa}_i^{\mathcal{V}}}$ is the subset of parameters that are parents of node $i$ in the \emph{entanglement graph}.
\begin{equation}
(i,j) \notin \mathcal{V} \implies
\frac{\partial [\bv(\btheta,\bx_0)]_i}{\partial \theta_j} = 0 \quad \forall \boldsymbol{\theta}, \bx_0.
\end{equation}
Given, $\mathbf{h}(\cdot)$ is bijective (and therefore so is its inverse), the entanglement map must also be bijective, as the composition of two bijective functions is also bijective. Consequently, $\btheta$ and $\hat{\btheta}$ are reparametrisations of the same underlying manifold. %

\subsection{Disentanglement and Entanglement Graphs
}
The entanglement graph provides a mathematical description of the relationship between a known ground-truth parametrisation and a representation learnt by a neural network. 
If we show that the adjacency matrix of the entanglement graph must be a permutation of the identity matrix, then we have also equivalently demonstrated that the model has disentangled up to elementwise diffeomorphism.
To illustrate the point, consider $\mathcal{V}=\mathbf{I}$ and a two-dimensional parametrisation. This entanglement graph forces the entanglement mapping dependencies to be element-wise, since each output latent can depend on at most one input latent. Consequently, $\hat{\theta}_0 = v_0(\theta_0),\hat{\theta}_1 = v_1(\theta_1)$, or some permutation thereof.
Note that the entanglement mapping $\bm{v}(\cdot)$ can be highly non-linear, and constraining the entanglement graph does not allow us to make any assumptions about the functional form of the mapping, only that it is smooth, differentiable, and invertible, or in other words, a diffeomorphism.

Hence, we have established that showing an entanglement graph is either the identity matrix or a permutation of the identity matrix is sufficient to establish that a learnt representation is disentangled up to an element-wise diffeomorphism (see definition \ref{definition:diffeomorphism}) and permutation.

\new{
Without further assumptions upon the systems $\mathcal{S}$ and $\hat{\mathcal{S}}$, the entanglement graph connecting them is unconstrained. The following section introduces the necessary assumptions and the theorem under which the entanglement graph is constrained to be a permutation of the identity.
}

\subsection{
Assumptions
}
\setcounter{assumption}{3}
We formalize three core assumptions which in conjunction sufficiently constrain the entanglement graph and form the primary identifiability result (restated below). These assumptions are applied on top of the system preliminaries (assumptions \ref{assumption:existence}, \ref{assumption:obs_eq}, and \ref{assumption:markov}). Assumption \ref{assumption:path_connected} ensures that the representation of dynamics does not entangle with the state representation, by precluding possible sources of entanglement from existing\footnote{ For an intuition on path connectivity, refer to Appendix \ref{sec:appendix:example_path_connectedness}}. Assumption \ref{assumption:sparsity_reg} restricts the causal complexity of the learnt representation, thereby reducing the admissible representation space. Assumption \ref{assumption:sufficient_variability} enforces information diversity: if two parameters were to always affect the system in the exact same way, then no amount of data can tell them apart, preventing degenerate cases where the number of parameters is not representative of the true dimension of variation of the system.

\begin{assumption}[Path Connectedness]
    \label{assumption:path_connected}
We assume the state variable domain $\mathcal{X}$ is path-connected, such that there exists no disjoint decomposition of the state space separated by boundaries that are fundamentally unreachable or untraversable by the system's trajectories.
\begin{align}
\forall \mathcal{X}_A, \mathcal{X}_B \subset \mathcal{X} \text{ s.t. } \mathcal{X}_A \cup \mathcal{X}_B = \mathcal{X}, \mathcal{X}_A \cap \mathcal{X}_B = \emptyset, \exists (\bm{x}_0, \bm{\theta}, t) \text{ s.t. } \bm{x}_0 \in \mathcal{X}_A \land \bm{x}_t \in \mathcal{X}_B
\end{align}
\end{assumption}
\begin{assumption}[Sufficient Sparsity]
    \label{assumption:sparsity_reg}
    We assume that $\hat{\mathcal{S}}$ learns a representation and a system model that is at most as sparse as the underlying system, which in practise is enforced by assuming $\hat{\mathcal{S}}$    
    is a minimiser of $\Vert \hat{\mathcal{G}}\Vert_0$. This implies that $\Vert \hat{\mathcal{G}}\Vert_0 \leq \Vert \mathcal{G}\Vert_0$.
\end{assumption}
\begin{assumption}[Sufficient Variability of the Jacobian \citep{lachapelle22a}]
    \label{assumption:sufficient_variability}
    There must exist at least one $\btheta \in \Theta$ such that there exists a set of states $\{ \bx_p\}_{p=1}^{\Vert \mathcal{G} \Vert_0}, \bx \in \mathcal{X}$  such that 
    \begin{align}
        \text{span}\left\{ 
        \nabla_{\btheta} \bf ( \bx_p, \btheta )
        \right\}_{p=1}^{\Vert \mathcal{G} \Vert_0}
        &=
        \mathbb{R}^{\mathcal{G}}
    \end{align}
The real graph subset is defined in definition \ref{definition:real_subset}. This common formulation \citep{lachapelle22a,hyvarinen2019nonlinear} requires that the effect of each causal edge in the system graph be independently observable from the state space.
\end{assumption}

\subsection{
Definition and Proofs of Supporting Lemmas
}
\new{
The proof of Theorem \ref{theorem:1} is built from several smaller supporting lemmas defined here. The final proof combining these lemmas is stated in section \ref{sec:proof:main}. Lemma \ref{lemma:entanglement_map_independence} combines assumption \ref{assumption:path_connected}
}
\begin{clemma}[Entanglement Map Independence]
If a system satisfies path connectivity (assumption \ref{assumption:path_connected}), the entanglement mapping is independent of the starting state: $\bv(\boldsymbol{\theta},\boldsymbol{x}_0)=\bv(\boldsymbol{\theta})$.
\label{lemma:entanglement_map_independence}
\end{clemma}

\begin{proof}
The trajectory decoder function is composed of one-step Markovian predictions from a dynamical model, $\bf(\cdot)$.

\begin{align}
\boldsymbol{\tau}=\bh(\bx_0, \btheta)=[\btx_0,\ft(\btx_0), (\ft \circ \ft)(\btx_0),...,\ft^T (\btx_0)].
\end{align}

The proof follows two steps. First, we show that the entanglement mapping must be invariant to states that can originate from the same trajectory, and second, we show that it must be invariant between any states that could have originated from a chain of trajectories. 
For a trajectory $\btau = (\bx_0, \bx_1, \cdots, \bx_T )$ generated by a fixed ground truth parameter $\btheta$. Under autoregressive rollouts, the model reuses the same inferred latent parameter $\bht$ to decode the dynamics at every step along the trajectory. Hence, for the rollout to remain consistent with the underlying system (as is guaranteed by the uniqueness of each trajectory in assumption \ref{assumption:existence}), the entanglement mapping, $\bv (\cdot)$, cannot depend on which state along the same trajectory is used for inference. If the rollout were to be reinitialised from the middle of some trajectory, $\tilde{\bx}_0=\bx_t, \: \tilde{\btau}=\bh(\tilde{\bx}_0, \btheta) $, while keeping $\btheta$ unchanged, the inferred latent parameter must agree with the original $\hat{\bh}^{-1}(\btau)= \hat{\bh}^{-1}(\tilde{\btau})$. Otherwise, the model would assign different latent parameters to the same physical system depending on the temporal offset, thereby violating the Markovian assumption of the transition model.

Consequently, this implies $\bv(\btheta, \bx_0)=\bv(\btheta, \bx_1)=v(\btheta, \bx_i) \forall \btheta,\bx_i \in \btau$.
The corollary of this is that the entanglement mapping must be invariant to states within the same trajectory; therefore, defining a trajectory subset of the state space:

    \begin{equation}
        \mathcal{X}_{\btau}(\boldsymbol{x}_0,\boldsymbol{\theta}) = \left\{\tilde{\boldsymbol{x}}\mid \exists \boldsymbol{x}' \in \mathcal{X}_{\btau} \text{ s.t. } \tilde{\boldsymbol{x}}=\bf(\boldsymbol{x}', \boldsymbol{\theta}) \text{ and } \boldsymbol{x}_0 \in \mathcal{X}_{\btau} \right\}.
    \end{equation}

    Set $\mathcal{X}_{\btau}$ represents the set of all positions that could have led to or have been produced from $\bm{x}_0$ given a particular parametrisation of a system.  In the geometric deep learning literature, this concept is referred to as manifold invariance \citep{bronstein2021geometric, Gerken2023}. Furthermore, as the entanglement map is a function of state, then instantaneously evaluated at a specific state, $\bm{x}'$, the representation becomes independent of state. The invariance manifold can consequently be extended by considering local stationarity of the representation at fixed states. \new{For any two trajectories, $\btau_a= \bh(\bx_a, \btheta_a)$, $\btau_b= \bh(\bx_b,\btheta_b)$, such that at some point in their trajectories, both contain a state $\bx_c \in \btau_a \cap \btau_b$.
    }

From the definition of $\mathcal{X}_{\btau}$ and the consequential state invariance of the entanglement mapping, 
    \begin{align}
        \bv(\bx_a,\btheta) &= \bv(\bx_c,\btheta) \quad\forall \btheta, \bx_a,\bx_c \in \mathcal{X}_{\btau} (\bx_a, \btheta_a) 
        \\
        \bv(\bx_b,\btheta) &= \bv(\bx_c,\btheta) \quad\forall \btheta, \bx_b,\bx_c \in \mathcal{X}_{\btau} (\bx_b, \btheta_b)
        \\
        \bx_c \in \btau_a \cap \btau_b &\implies \mathcal{X}_{\btau} (\bx_a, \btheta_a) \cap \mathcal{X}_{\btau} (\bx_b, \btheta_b) \neq \emptyset
    \end{align}
Through the transitive property, whenever two trajectory manifolds overlap in state variable space, they must share the same inferred representation, implying that $\bv$ is invariant on $\mathcal{X}_{\btau_a} \cup \mathcal{X}_{\btau_b}$ (and, by transitive closure, on any union of manifolds connected through overlaps).

    \begin{equation}
        \mathcal{X}_P(\boldsymbol{x}_0) = \left\{\tilde{\boldsymbol{x}}\mid \exists \boldsymbol{x}' \in \mathcal{X}_{\btau}, \boldsymbol{\theta} \text{ s.t. } \tilde{\boldsymbol{x}}=\bf(\boldsymbol{x}', \boldsymbol{\theta}) \text{ and } \boldsymbol{x}_0 \in \mathcal{X}_{\btau} \right\}.
    \end{equation}

    The set $\mathcal{X}_P$ defines a partition of the state space. The set $\mathcal{X}_P$ lifts $\mathcal{X}_{\btau}$; the latter is defined as the set of states that are trajectory-connected to a specific starting point and dynamical parameter. The former extends $\mathcal{X}_{\btau}$ to include every dynamical system that is either dynamically or trajectory-connected to a point through any arrangement.
    A corollary of this is that the set $\mathcal{X}_P \subseteq \mathcal{X}$ defines a partition of the state space for which $\bv(\boldsymbol{\theta},\boldsymbol{x}')=\bv(\boldsymbol{\theta},\tilde{\boldsymbol{x}}), \forall \boldsymbol{x}',\tilde{\boldsymbol{x}}\in \mathcal{X}_P$.

     If the latent parameter is surjective (definition \ref{definition:surjective}) with respect to the dynamical model, then trivially $\mathcal{X}_P$ must span the whole domain $\mathcal{X}$ and global manifold invariance follows. If $\mathcal{X}=\mathcal{X}_P$, then the entanglement mapping must be invariant to the starting state, $\bv(\boldsymbol{x},\boldsymbol{\theta})=\bv(\boldsymbol{\theta})$. 
    \label{proof_entanglement_map_independence}
\end{proof}

Topological assumptions upon the data-generating distribution in representation learning are commonplace. Literature usually enforces stronger assumptions about the bijectivity of the latent space rather than a more relaxed connectivity requirement \citep{lachapelle2024nonparametric, fu2025identification}, although \citet{schug2024discovering} implements a discrete analogue to our path connectivity. Similar concepts and methods have been used more generally to study emergent structures in latent representational space \citep{gu2023original,li2025understanding}. \citet{lachapelle2023additive} applies a similar path-connectedness assumption on the mapping between observational and representational space without the constraint of connectedness through a dynamical model. 

\new{The independence of the entanglement mapping allows proof techniques from \citet{lachapelle2024nonparametric} to be imported and applied to analyse the constraints upon the entanglement mapping.}

\begin{clemma}[Causal Inclusion] Under assumption \ref{assumption:sufficient_variability} (Sufficient Variability), there must exist a permutation matrix, $\bm{P}$, such that every edge contained in the ground truth graph, $\mathcal{G}$, must also be contained within the permuted version of the learnt graph $\hat{\mathcal{G}}$.\footnote{Note all products and sums between adjacency matrices follow graph arithmetic (definition \ref{definition:graph_arithmetic}).}
\begin{align}
    \exists \bm{P} \text{ s.t. } \quad \bm{A}_{\gG} \subseteq \bm{A}_{\gGh} \bm{P}.
\end{align}
\label{lemma:inclusion}
\end{clemma}
\begin{proof}
    We begin from the observational equivalence assumption (assumption \ref{assumption:obs_eq}) but consider the inverse of $\bm{v}(\cdot)$ which must exist as $\bm{v}(\cdot)$ is a diffeomorphism.  
\begin{align}
    \bf(\bx_{t}, \rv(\bht)) &= \hat{\bf}(\bx_t, \bht),
    \\
     \nabla_{\bht} \hat{\bf} &= \nabla_{\bt} \bf \nabla_{\bht} \rv \quad \forall \bx_t,\: \bht,
     \\
     [ \nabla_{\bht} \hat{\bf}]_{i,j} &= \sum_k [\nabla_{\bt} \bf]_{i,k} [\nabla_{\bht} \rv]_{k,j},
     \\
     \nabla_{\bht} \hat{\bf}_{i,j} &= (\nabla_{\bt} \bf_{i,\cdot})^T \nabla_{\bht} \rv_{\cdot,j}.
\end{align}

From here, we show that given the invertibility of $\rv$ and following assumption \ref{assumption:sufficient_variability}, every edge contained in $\gG$ must also be contained within some permutation of the learnt graph $\hat{\gG}$. First, consider, rather than $\rv$, a permutation $\rv$. As $\rv$ is invertible, there exists a permutation thereof such that there are no zeros on the diagonal. This fact follows simply from the \textit{Inverse Function Theorem} \citep{intro_to_smooth_manifolds}. The \textit{Inverse Function Theorem} proves that the Jacobian of the entanglement mapping must be of full row rank. Consequently, it must be possible to rearrange the entanglement graph such that there are no zeros on the diagonal. The permuted composition will be denoted as $\nabla_{\bht} \rv \bm{P} = \bC$. Formally, $\bC$ is a function of $\bht$, but the inputs of $\bC$, $\bf$, and $\hat{\bf}$ are omitted for notational clarity.
\begin{align}
    [ \nabla_{\bht} \hat{\bf}\bm{P}]_{i,j} &= \sum_k [\nabla_{\bt} \bf]_{i,k} [\nabla_{\bht} \rv \bm{P}]_{k,j},
    \\ &=\sum_k [\nabla_{\bt} \bf]_{i,k} [\bC]_{k,j}, \label{eq:36_hfp=fc}
    \\
    &= [\nabla_{\bt} \bf]_{i,j} [\bC]_{j,j} + \sum_{k\neq j} [\nabla_{\bt} \bf]_{i,k} [\bC]_{k,j}
    .
    \label{equation:41}
\end{align}
From the inverse function theorem, $\bm{C}_{j,j}\neq 0 \forall j$. This implies that if $[\nabla_{\bt} \bf]_{i,j} \neq 0$ then ${[\nabla_{\bt} \bf]_{i,j} [\bC]_{j,j} \neq 0}$. This is not sufficient to imply that edge $(i,j)$ is in $\bm{A}_{\gGh} \bm{P}$ since the domains $\bx \in \mathcal{X}$, $\btheta \in \Theta$ may constrain $\nabla_{\bt} \bf_{i,\cdot}$ to the subspace of $\bm{C}_{\cdot,j}$. To combat this, we introduce assumption \ref{assumption:sufficient_variability}. The sufficient variability assumption ensures that there exists at least one input for which the Jacobian of the ground truth system is not in the subspace of the permuted inverse entanglement map. Formally:
\begin{align}
    \forall (i,j) \in \{i,j \mid  [ \nabla_{\bht} \hat{\bf}\bm{P}]_{i,j} \neq 0\}, \quad \exists (\btheta, \bx_t) \text{ s.t } \sum_k [\nabla_{\bt} \bf]_{i,k} [\bC]_{k,j} \neq 0.
\end{align}

 There must exist a permutation $\bm{P}$ such that there exists any point in the input space where $[\nabla_{\bt} \bf]_{i,j}\neq 0$ necessarily implies $[ \nabla_{\bht} \hat{\bf}\bm{P}]_{i,j} \neq 0$. Consequently every edge in $\bm{A}_\mathcal{G}$ must exist in $\bm{A}_{\hat{\mathcal{G}}}\bm{P}$. This is not necessarily true in the other direction; not every edge in $\gGh$ must exist in a permuted form of $\gG$.\footnote{Toy example to showcase this is that $\gGh$ can be fully connected, and $\gG$ can be sparse. Since only the sufficient influence of the true dynamics model is constrained and not that of the learnt model.} Formally,
 \begin{align}
    \exists \bm{P} \text{ s.t. } \quad \bm{A}_{\gG} \subseteq \bm{A}_{\gGh} \bm{P}.
 \end{align}
\end{proof}
\begin{clemma}[Graph Equivalence] Under assumptions \ref{assumption:sparsity_reg} and \ref{assumption:sufficient_variability}, the learnt graph $\mathcal{G}$ is equivalent to the ground truth graph up to permutation. 
\begin{align}
\exists \bm{P} \text{ s.t. }\quad \bm{A}_{\hat{\mathcal{G}}}\bm{P} \equiv \bm{A}_\mathcal{G}.    
\end{align}
    \label{lemma:equivalence}
\end{clemma}
\begin{proof}
    Trivially true. From lemma \ref{lemma:inclusion} it is established that there exists a permutation such that every edge contained in $\gG$ must also be contained within the permuted $\gGh$. Assumption \ref{assumption:sparsity_reg} states that the learnt representation's graph must contain at most as many edges as the underlying graph. Therefore, they must be the same.
\end{proof}
\begin{clemma}[The Inverse Entanglement Graph is Right Graph Consistent] Given all prior assumptions, the permuted 
        Given lemma \ref{lemma:equivalence}, the permuted inverse entanglement mapping must be right $\gG$ consistent (definition \ref{definition_graph_consistency}). Which, in the form of adjacency matrix arithmetic, requires 
        \begin{align}
            \exists \bm{P} \text{ s.t. }\quad\bm{A}_\gG = \bm{A}_\gG \bm{A}_{\mathcal{V}^{-1}} \bm{P}. 
        \end{align}
    \label{lemma:graph_consistency}
\end{clemma}
\begin{proof}
    Beginning from the full form of equation \ref{eq:36_hfp=fc}, 
    \begin{align}
    \underbrace{\nabla_{\bht} \hat{\bf}}_{\subseteq \gGh = \gG \bm{P}^T}
    \bm{P}
    &=
     \underbrace{\nabla_{\bt}\bf}_{\subseteq \gG} 
     \underbrace{\nabla_{\bht} \rv}_{\subseteq \mathcal{V}^{-1}} 
     \bm{P},
     \label{eq:42}
    \end{align}
    From Lemma \ref{lemma:inclusion}, we know that there exists a point in the input space such that no cancellations occur in the product between the ground truth system Jacobian and the Jacobian of the inverse entanglement map (the Jacobian of the ground truth graph varies sufficiently and consequently it is not constrained to the null-space of the Jacobian of the entanglement mapping). The relation below can be rewritten in adjacency matrix arithmetic by considering the set of non-zero values each element of equation \ref{eq:42} can take. The set of non-zero values is given by the elements of the graph in each underbrace.
    \begin{align}
        \bm{A}_{\gGh} \bm{P} &\subseteq \bm{A}_\gG \bm{A}_{\mathcal{V}^{-1}} \bm{P}. 
    \end{align}
    The lack of cancellations allows the subset to be tightened to equality:
    \begin{align}
        \bm{A}_{\gGh} \bm{P} &= \bm{A}_\gG \bm{A}_{\mathcal{V}^{-1}} \bm{P},
        \\
        \bm{A}_{\gG} &= \bm{A}_\gG \bm{A}_{\mathcal{V}^{-1}} \bm{P}.
    \end{align}
    
\end{proof}
\begin{clemma}[Graphical Criterion] If $\gG$ satisfies the following graphical criterion: 
\begin{align}
   \forall i: \quad  \cap_{a\in\text{Ch}(i)}\text{Pa}(a)=\{i\}.
\label{equation:graphical_criterion}
\end{align}
Then the only valid right consistent graph on $\mathcal{G}$, $\mathcal{R}: \: \bm{A}_\gG = \bm{A}_\gG \bm{A}_\mathcal{R}$ is the identity.
\label{lemma:graphical_criterion}
\end{clemma}

\begin{proof}
We will now formalise the allowed structures on $\mathcal{R}$, given $\mathcal{G}$. Consider a generic 3x3 example, where a $\star$ symbol denotes that any possible value could be taken. 
\[
\underbrace{\begin{bmatrix}
\star & \star & \star \\
\star & \star & \star \\
\star & \star & \star
\end{bmatrix}}_\mathcal{G}
=
\underbrace{\begin{bmatrix}
\star & \star & \star \\
\star & \star & \star \\
\star & \star & \star
\end{bmatrix}}_\mathcal{G}
\underbrace{
\begin{bmatrix}
R_{11} & R_{12} & R_{13} \\
R_{21} & R_{22} & R_{23} \\
R_{31} & R_{32} & R_{33}
\end{bmatrix}}_\mathcal{R}.
\]

In adjacency matrix arithmetic, cancellations of edges are impossible; consequently, the product can be written out in Boolean logic form as:
\begin{equation}
    \mg_{ij}=0 \implies (\mg_{i1} \land R_{1j}=0) \lor (\mg_{i2} \land R_{2j}=0)\lor...\lor(\mg_{ik}\land R_{kj})=0 .
\end{equation}

By iterating over all appearances of $R_{ij}$, it can be shown through boolean algebra that $R_{ij}$ must be $0$ if and only if:
\begin{equation}
    \bigvee_{k} \left ( \mg_{kj}=0 \land \mg_{ik} \neq 0  \right ) = 1.
\end{equation}

Further rearrangement shows that the set of nonzero elements in each row of the entanglement mapping is given by $\mathcal{R}_{i,:}$:
\begin{equation}
    \cap_{a\in\text{Ch}(i)}\text{Pa}(a).
\end{equation}
Where $\text{Ch}(\cdot)$ and $\text{Pa}(\cdot)$ are the children and parents of a node in each graph, respectively.

Therefore, the set of permissible elements in each row of $\bm{A}_\mathcal{R}$ is restricted to the diagonal (equivalently the elements of the identity), by iterating over row elements in the graphical criterion.

\begin{align}
   \forall i: \quad  \cap_{a\in\text{Ch}(i)}\text{Pa}(a)=\{i\}.
\label{equation:graphical_criterion}
\end{align}
\end{proof}

\subsection{Proof of Theorem 1}
\label{sec:proof:main}

\new{\setcounter{ctheorem}{0}
\begin{ctheorem}[Disentanglement of Latent System Parameters]
    \label{theorem:1}
    Let $\mathcal{S}$, and $\hat{\mathcal{S}}$ correspond to two systems satisfying assumptions \ref{assumption:existence}, \ref{assumption:obs_eq}, and \ref{assumption:markov}.
    Further assume that,
    \begin{enumerate}
        \item The system, $\mathcal{S}$, is path connected (assumption \ref{assumption:path_connected}).
        \item $\hat{\mathcal{S}}$ satisfies $\Vert \hat{\mathcal{G}} \Vert_0 \leq \Vert \mathcal{G} \Vert_0$ (assumption \ref{assumption:sparsity_reg}). 

        \item The Jacobian of the transition model, $\nabla_{\btheta} \bf$, \emph{varies sufficiently} (assumption \ref{assumption:sufficient_variability}).
    \end{enumerate}
    Then the dynamical representations of $\mathcal{S}$, and $\hat{\mathcal{S}}$ are equivalent up to permutation and element-wise diffeomorphism if $\mathcal{G}$ satisfies the following graphical criterion:
    \begin{equation}
       \forall i: \quad  \cap_{a\in\text{Ch}_\mathcal{G}(i)}\text{Pa}_\mathcal{G}(a)=\{i\},
       \label{equation:main_graphical_criterion}
    \end{equation}
    where $\text{Ch}_\mathcal{G}(j)$, and $\text{Pa}_\mathcal{G}(j)$ denotes the set of children and parent nodes of a node $j$ in a graph $G$.
\end{ctheorem}
}
\begin{proof}

\paragraph{Step 1.} %
We define an entanglement map $\bv: \Theta \times \mathcal{X} \to \hat{\Theta}$ such that $\bht = \bv(\bt, \bx)$. This mapping describes how the learned parameters relate to the ground truth at any given state. The dependencies in this mapping are captured by an entanglement graph $\mathcal{V}$, where an edge exists if a ground-truth parameter influences a learned latent dimension:
\begin{align}
    \hat{\boldsymbol{\theta}}=\bv(\boldsymbol{\theta},\boldsymbol{x}_0)= (\hat{\bh}^{-1} \circ \bh)
(\boldsymbol{\theta},\boldsymbol{x}_0).
\end{align}
\paragraph{Step 2.} %
By assumption \ref{assumption:path_connected}, any two states in the domain are connected through some sequence of trajectories. Lemma \ref{lemma:entanglement_map_independence} uses the assumption to show that the entanglement map must be invariant to the starting state:
\begin{align}
    \bv(\boldsymbol{\theta},\boldsymbol{x}_0)=\bv(\boldsymbol{\theta}) \quad\forall \bx_0, \bt
\end{align}

\paragraph{Step 3.} We establish the conditions under which the entanglement map $\bm{v}$ implies representational disentanglement. By definition, $\bm{v} = \hat{\bm{h}}^{-1} \circ \bm{h}$. From Assumption \ref{assumption:existence}, $\bm{h}(\cdot)$ must be injective, and from Assumption \ref{assumption:obs_eq}, $\hat{\bm{h}}(\cdot)$ must be surjective; together with the requirement that both models are defined over the same image, this ensures that $\bm{v}$ is a bijection. Furthermore, Assumption \ref{assumption:markov} requires that the transition model (and consequently $\bm{h}$) and its first derivative are continuous, rendering $\bm{v}$ a diffeomorphism. If the entanglement graph is shown to be a permutation of the identity, $\bm{A}_\mathcal{V} = \bm{P}$, it implies that each learned parameter $\hat{\theta}_i$ is a function of exactly one unique ground-truth parameter $\theta_j$. Because the requirements of bijectivity and differentiability must hold, such a structural constraint necessitates that $\bm{v}$ is an element-wise diffeomorphism, satisfying the criteria for disentanglement (Definition \ref{definition_disentanglement}). By the same logic, if the inverse entanglement graph is a permutation of the identity ($\bm{A}_{\mathcal{V}^{-1}} = \bm{P}$), then each ground-truth parameter is mapped to a single learned latent, which similarly forces an element-wise diffeomorphic relationship between the two representations.

\paragraph{Step 4.} %
From observational equivalence (assumption \ref{assumption:obs_eq}), we relate the Jacobians of the ground truth and learnt transition models via the inverse entanglement mapping,
\begin{align}
    \nabla_{\bht} \hat{\bf}(\bx_t, \bht) = \nabla_{\bt} \bf(\bx_t, \bv^{-1}(\bht)) \nabla_{\bht} \bv^{-1}(\bht).
\end{align}

\paragraph{Step 5.} %
We use lemma \ref{lemma:inclusion} to show that given sufficient variability of the Jacobian of the transition model (assumption \ref{assumption:sufficient_variability}), there must exist a point in the input space such that every non-zero entry in the Jacobian of the ground truth transition model must also be non-zero up to permutation in the learnt transition model. In terms of the structural graphs, this implies that:
\begin{align}
    \exists \bm{P} \text{ s.t. } \quad \bm{A}_{\gG} \subseteq \bm{A}_{\gGh} \bm{P}.
\end{align}

\paragraph{Step 6.} %
Additionally, introducing assumption 5 (Sparsity), which posits that the learned graph's cardinality is bounded by the ground truth ($\|\gGh\|_0 \le \|\gG\|_0$). By lemma \ref{lemma:equivalence}, since $\bm{A}_\gG \subseteq \bm{A}_{\gGh} \bm{P}$ and the number of edges in $\gGh$ cannot exceed those in $\gG$, the two graphs must be equivalent up to permutation: $\bm{A}_{\gGh} \bm{P} \equiv \bm{A}_{\gG}$.

\paragraph{Step 7.} %
Using the graph equivalence, we return to the Jacobian formulation from step 4. The sufficient-variability assumptions require that the transition model not be constrained to the nullspace of the Jacobian of the inverse entanglement mapping. Consequently, lemma \ref{lemma:graph_consistency} shows that the inverse entanglement mapping must be right consistent with the ground truth graph:
    \begin{align}
        \bm{A}_{\gG} &= \bm{A}_\gG \bm{A}_{\mathcal{V}^{-1}} \bm{P}.
    \end{align}

\paragraph{Step 8.} %
The requirement of right consistency means that the ground truth graph, $\gG$, constrains the allowed structure upon $\bm{A}_{\mathcal{V}^{-1}}\bm{P}$. Lemma \ref{lemma:graphical_criterion} proves that if $\gG$ satisfies the graphical criterion:
\begin{align}
   \forall i: \quad  \cap_{a\in\text{Ch}(i)}\text{Pa}(a)=\{i\}.
\label{equation:graphical_criterion}
\end{align}
Then the only matrix $\bm{A}_\mathcal{R}$ that satisfies the consistency equation, $\bm{A}_\gG = \bm{A}_\gG \bm{A}_\mathcal{R}$ is the identity matrix. Therefore $\bm{A}_{\mathcal{V}^{-1}}\bm{P}=\bm{I}$.

\paragraph{Step 9.} Since $\bm{A}_{\mathcal{V}^{-1}}\bm{P} = \bm{I}$, it follows that the entanglement graph is a permutation of the identity. This forces the mapping $\bv(\bt)$ to be an element-wise diffeomorphism, meaning the system parameters have been disentangled.
This concludes the proof.
\end{proof}

\newpage
\subsection{Extension to Local Causal Graphs}
\label{sec:proof:local_graphs}
\label{appendix:local_causal_graph_proof}

Many dynamical systems exhibit state-dependent causal structures, where only a subset of the edges in the global graph are active at a given state. Let the state space admit a finite cover by non-zero measure partitions $\{ \mathcal{X}_l\}_{l=1}^b$, such that for $\boldsymbol{x}_t \in \mathcal{X}_l$, the one-step transition $\boldsymbol{x}_{t+1}=\bf(\boldsymbol{x}_t,\boldsymbol{\theta})$ is Markov not only with respect to the global DAG, but also the local DAG $\mathcal{G}_l \subseteq \mathcal{G}$ associated with that partition. The partition must have non-zero measure so the sufficient variability assumption can be applied. In practice, this means that when dealing with continuous systems, instantaneous events such as collisions can not be factored into discrete local graph partitions. However, when operating in discrete time, collisions occupy a non-zero measure.

\begin{equation}
\big[\bm{A}_{\mathcal{G}_l}\big]_{i,j} = 0 \iff \frac{\partial [\bf (\bx_t, \btheta )]_i}{\partial \theta_j} = 0 \quad \forall \boldsymbol{x}_t \in \mathcal{X}_l, \mathbf{\theta}.
\end{equation}

We further assume:
\begin{itemize}
    \item Assumptions \ref{assumption:existence}, and \ref{assumption:obs_eq} hold globally. This is required to prevent trivial solutions. 
    \item Assumption \ref{assumption:path_connected} must also hold globally. In practice, this imposes the additional constraint that the system must be able to navigate between state space partitions. Otherwise, a separate representation might be learnt for each partition. Global path connectedness ensures the same representation is used for each partition and allows theoretical reasoning about the reuse of representations across partitions.
    \item Assumption \ref{assumption:sparsity_reg}-\ref{assumption:sufficient_variability} must hold \textit{within} each partition $\mathcal{X}_l$. The former allows local representations to be sparser than global representations, whilst the latter ensures that the effect of each edge in a local causal graph is measurable.
\end{itemize}

\subsubsection{Pointwise Consistency of the Entanglement Mapping.}

Recall the global \textit{right-consistency} relation from lemma \ref{lemma:graph_consistency}, which forces the inverse entanglement map to respect the system graph in adjacency arithmetic: $\bm{A}_{\gG} = \bm{A}_\gG \bm{A}_{\mathcal{V}^{-1}} \bm{P}$, obtained from the Jacobian factorisation, $\nabla_{\bht}\hat{\bf} = \nabla_{\bt}\bf \nabla_{\bht} \bv^{-1}$, and $\nabla_{\bt} \bf = \nabla_{\bht} \hat{\bf} \nabla_{\bt} \bv$. Building the same calculation, restricted to a partition $\mathcal{X}_l$, gives, pointwise in $\bx_t \in \mathcal{X}_l$:
\begin{align}
    \bm{A}_{\gG_l} &= \bm{A}_{\gG_l}  \bm{A}_{\mathcal{V}^{-1}}\bm{P}.
    \label{equation:local_consistency}
\end{align}
Intuitively, whenever an edge is active in $\gG_l$, the corresponding influence must be representable through the inverse entanglement mapping, $\bv^{-1}$, at that state.

\subsubsection{Allowed Structure of $V$ under local graphs.}
In the global case, Boolean adjacency arithmetic shows that the set of indices allowed to be non-zero in row $i$ of $\mathcal{V}^{-1}$ is:
\begin{align}
    \text{supp}(\mathcal{V}^{-1}_{i,:}) &\subseteq \bigcap_{a \in \text{Ch}(i)} \text{Pa}(a).
\end{align}

Which yields the global graph criterion for disentanglement $\bigcap_{a \in \text{Ch}(i)} \text{Pa}(a) = \{i\} \forall i $. Note that the same applies to each and every partition:
\begin{align}
    \text{supp}(\mathcal{V}^{-1}_{i,:}) &\subseteq \bigcap_{a \in \text{Ch}_{\gG_l}(i)} \text{Pa}_{\gG_l}(a) \quad \forall l.
\end{align}

Since the representation is reused globally and throughout partitions, the allowed supports must satisfy \textbf{all} local constraints simultaneously. Hence, the allowed support upon the entanglement graph is given as:
\begin{align}
    \text{supp}(\mathcal{V}^{-1}_{i,:}) &\subseteq\bigcap_{l=1}^b  \bigcap_{a \in \text{Ch}_{\gG_l}(i)} \text{Pa}_{gG_l}(a)
\end{align}

Therefore, the local-graph graphical criterion guaranteeing disentanglement is:
\begin{align}
        \forall i: \bigcap_{l=1}^b  \bigcap_{a \in \text{Ch}_{\gG_l}(i)} \text{Pa}_{gG_l}(a) = \{ i \}.
\end{align}

\subsubsection{Corollaries and Edge Cases.}
\begin{itemize}
    \item \textbf{Monotonicity}. Adding more local graphs can only shrink the admissible set for each row of the entanglement graph, so incorporating additional partitions can never worsen disentanglement and may only help it.
    \item \textbf{Sufficient Subgraph}. If full latent parameter disentanglement holds in \textit{any} one local graph, the representation must be entangled globally.
    \item \textbf{Insufficient Subgraph}. If full latent parameter disentanglement holds in \textit{none} of the individual sub-graphs, or the global graph, disentanglement throughout the intersections is still feasible given the right structure.  
    \item \textbf{Overspecification}. The number of dynamical parameters that can be disentangled can be significantly larger than the number of state parameters considered. Theroretically in the limit it can be infinite as long as all assumptions are still met, and the full set of dynamical parameters is inferable from a trajectory.
    \item \textbf{Nodes without children}. If $\text{Ch}_{\gG_l} = \emptyset$ in some sub-graph, the local constraint does not restrict row $i$. When evaluating the intersection, these should be ignored. In practice, it is easiest to replace such a case with the set of all nodes, as this best aligns with the allowed elements of the entanglement mapping. Suppose a node has no children in any subset of the global causal graph. In that case, the invertibility of dynamics is unmet, and the latents are not deterministically encodable from a trajectory. The theory does not apply in cases where the number of learnt parameters exceeds the minimal set of parameters required to describe a system. 
\end{itemize}

\newpage
\section{
An Example of Path Connectedness
}
\label{sec:appendix:example_path_connectedness}

\begin{tcolorbox}[colback=red!5!white,colframe=white,arc=5pt]
\textbf{Example.}
\begin{figure}[H] %
    \centering
    \includegraphics[width=0.4\linewidth]{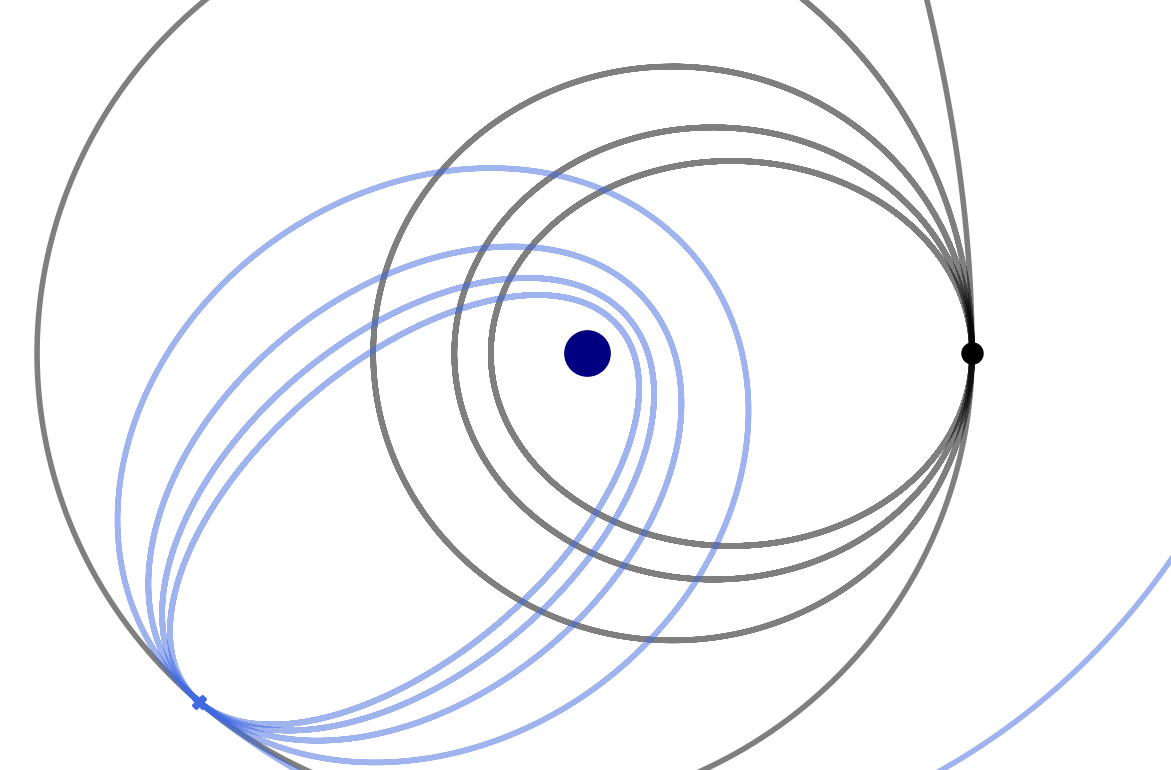}
    \caption{Illustration of how partition sets are defined using a planetary orbit example. Consider this system where the dynamical factor of variation is the (central) sun's mass. From an initial starting position and velocity, a large set of trajectories can unfold depending on the sun's mass. Throughout every point in black, the learnt representation must be independent of the system state. Each black line represents one trajectory that could have originated from the black point by varying the dynamical parameter. This process can be repeated iteratively at any point along any of the black lines to expand the set. In the figure, we choose the blue point from the set of black trajectories and, from the same starting position and velocity, change the dynamical parameter again. The result is a set of states, which can be expanded to cover a partition of the state-space.}
    \label{fig:planet_example}
\end{figure}
    Path connectivity uncovers trajectory-invariant symmetries in dynamical families. This means that the usefulness of a representation to our theorem is highly sensitive to the available state representation. Trajectory invariant symmetries can be challenging to spot. Consider the example from figure \ref{fig:planet_example}. It is evident that, given a broad enough set of allowable starting conditions and sun masses, the entire position space can be covered by one singular partition, as there exist trajectories between any two positions. However, if the state space were to include the planet's velocity, then two partitions would exist in the system. One for clockwise circular motion and one for counter-clockwise circular motion, since no dynamical parameter or state can transfer the system from one partition to another. Therefore, we expect any dynamical model trained on this system to learn up to two state-dependent distinct representations, one for clockwise circular motion and one for counter-clockwise circular motion.
    This also means that invariant or redundant system information cannot be part of the state space. For example, imagine a dynamical model predicting the behaviour of cubes on a table, except sometimes the cube is red and sometimes the cube is blue. If this property were encoded into the state space, it would be invariant along a trajectory, and any dynamical parameter can be a function of the cube colour.
\end{tcolorbox}

\clearpage

\newpage
\new{
\section{Discussion}
\subsection{Comparison with the Recent Works \citet{lachapelle2024nonparametric} and \citet{yao2024marrying}}
}
The proof presented in Appendix \ref{appendix:proof_of_theorem} builds upon the proposed concepts by \citet{yao2024marrying} and the proof technique of \citet{lachapelle2024nonparametric}. \citet{lachapelle2024nonparametric} introduces \emph{Mechanism Sparsity} a principle for nonlinear ICA that achieves disentanglement by assuming latent factors depend only on a small subset of past factors or auxiliary variables. They derive an identifiability theorem showing that latent variables can be recovered up to permutation and element-wise diffeomorphism if the ground-truth causal structure satisfies a graphical criterion. To implement this in practice, they propose a method that simultaneously learns latent representations and their underlying causal graph structure via $L_1$ regularisation. 
We adapt the framework of \citet{lachapelle2024nonparametric} to a novel setting, necessitating three key theoretical extensions: the incorporation of auxiliary variable representation learning, the introduction of state-dependent local-causal graphs, and the application of a path-connectedness assumption.
\paragraph{Learning the Representation of Auxiliary Variables.} 
Prior works, such as \citep{lachapelle2024nonparametric,hyvarinen2019nonlinear}, are primarily focused on learning state representations while assuming auxiliary variables are directly observed. We show that in the framework of mechanism sparsity, these results can be extended to learning representations of auxiliary variables, which are of particular importance in settings of modelling dynamical systems, as dynamical parameters can be seen as auxiliary variables.
\paragraph{Path Connectedness Assumption.} Switching to auxiliary representational learning does introduce additional challenges not addressed by \citet{lachapelle2024nonparametric} and not discussed in \citet{yao2024marrying}. The representation of parameters can become entangled with the representations of state. We formalise path-connectedness, an assumption that resolves this issue. To our knowledge, we are the first to implement such an assumption in the context of a mechanism sparsity identifiability theorem. 
\paragraph{Local Graphs.} Finally, we introduce the notion of local graphs, which are state-dependent sub-graphs of a system's global graph. To the best of our knowledge, ours is the first work to establish a representational disentanglement identifiability theorem incorporating local graphs.
Critically, we show that incorporating subgraphs must be equiperformant with global graph methods or improve upon them. We anticipate that the shift from global to local intersections is a generalizable insight applicable to other graphical criteria within the mechanism sparsity framework, although we leave the formal proof of this extension to future work.

\citet{yao2024marrying} propose an isomorphism between Causal Representation Learning (CRL) and dynamical systems, arguing that integrating the two fields equips system identification with theoretical guarantees and improves the real-world applicability of CRL. Empirically, they rely on \emph{Mechanistic Neural Networks} \citep{pervezmechanistic} to model the parameters of an Ordinary Differential Equation (ODE) via differentiable solvers. This reliance on ODEs fundamentally distinguishes their domain from ours. While ODEs model continuous vector fields, our theorem allows for object-centric or factored representations that are not bound by global continuity. Consequently, our method is inherently capable of capturing sparse (both spatially and temporally), discontinuous interactions, such as collisions, that lie outside the standard expressivity of continuous ODE formulations. \\
~\\
Regarding identifiability, \citet{yao2024marrying} presents two key propositions. The first, applicable when the functional form is known (see their Corollary 3.1), aligns with the framework of sparse identification methods such as SINDy \citep{sindy}. By assuming a fixed dictionary of terms, this approach maps predicted parameters to pre-defined equation components; thus, identifiability arises from the imposed structural priors rather than through representation learning of latent variables.
\\
Their second proposition (Corollary 3.2) addresses the case of unknown functional forms and relies on \emph{multi-view} assumptions. They show that optimising an alignment loss across paired trajectories partially identifies shared parameters. In contrast, our work does not consider multi-view data or paired counterfactuals. Instead, we prove that \emph{mechanism sparsity} is sufficient in some settings to disentangle parameters from raw trajectories.

\subsection{Discussion on the Sufficient Influence Assumption}

The sufficient variability assumption (alternatively termed sufficient influence, assumption \ref{assumption:sufficient_variability}) emerges as a theoretical necessity. In this section, we demonstrate that this condition possesses a clear, intuitive interpretation within our framework. We support this claim through illustrative examples and a comparative analysis of how analogous constraints are utilised in the broader causal representation learning literature (section \ref{sec:appendix:D:related}). The assumption is restated below for convenience.

\setcounter{assumption}{5}
\begin{assumption}[Sufficient Variability of the Jacobian \citep{lachapelle22a}]
    There must exist at least one $\btheta \in \Theta$ such that there exists a set of states $\{ \bx_p\}_{p=1}^{\Vert \mathcal{G} \Vert_0}, \bx \in \mathcal{X}$  such that 
    \begin{align}
        \text{span}\left\{ 
        \nabla_{\bt} \bf ( \bx_p, \btheta )
        \right\}_{p=1}^{\Vert \mathcal{G} \Vert_0}
        &=
        \mathbb{R}^{\mathcal{G}}
    \end{align}
The real graph subset is defined in definition \ref{definition:real_subset}.
\end{assumption}

In essence, the assumption requires that every causal edge exerts a unique influence that is separable from the influence of every other edge. Practically, this requires that the graph $\mathcal{G}$ and our model of the ground-truth environment $\bf$ are irreducible.
Consequently, the assumption is mild and additionally imposes that the assumed cardinality of the parameter space is correct.

In the setting of physical parameters, irreducibility can be expressed as a lack of redundancy. For example, given a system under the influence of two springs with strengths $\theta_0$, and $\theta_1$, as per equation \ref{eq:appendix:D:example:A:1}.

\paragraph{Example A. Redundant Causal Edge}
\begin{align}
    \bx_{t+1}&= \bf(\bx_t, \bt) = -(\theta_0+\theta_1)\bx_t
    \label{eq:appendix:D:example:A:1}
    \\
    \nabla_{\btheta} \bf &=
    \begin{bmatrix}
        -\bx_t & -\bx_t
    \end{bmatrix}
    \label{eq:appendix:D:example:A:2}
\end{align}
\begin{figure}[t]
    \centering
    \begin{minipage}[b]{0.23\textwidth}
        \centering
        \begin{tikzpicture}[
            scale=1.0,
            every node/.style={transform shape},
            node distance=0.5cm and 2.5cm, 
            >=Stealth,
            leftnode/.style={circle, draw=oxfordred, fill=oxfordred, text=white, minimum size=7mm, font=\bfseries, inner sep=0pt},
            rightnode/.style={circle, draw=oxfordblue, fill=oxfordblue, text=white, minimum size=7mm, font=\bfseries, inner sep=0pt}
        ]
            \node[leftnode] (t0) at (0,-1.0) {$\theta_0$};
            \node[leftnode] (t1) at (0,-0.0) {$\theta_1$};
            \node[rightnode] (x) at (2.5, 0) {$\bx$};

            \draw[->] (t0) -- (x);
            \draw[->] (t1) -- (x);
        \end{tikzpicture}
        \par\vspace{5pt} %
        (A)
        \label{fig:graph:appendix:D:A}
    \end{minipage}
    \hfill
    \begin{minipage}[b]{0.23\textwidth}
        \centering
        \begin{tikzpicture}[
            scale=1.0,
            every node/.style={transform shape},
            node distance=0.5cm and 2.5cm, 
            >=Stealth,
            leftnode/.style={circle, draw=oxfordred, fill=oxfordred, text=white, minimum size=7mm, font=\bfseries, inner sep=0pt},
            rightnode/.style={circle, draw=oxfordblue, fill=oxfordblue, text=white, minimum size=7mm, font=\bfseries, inner sep=0pt}
        ]
            \node[leftnode] (t0) at (0,-1.0) {$\theta_0$};
            \node[leftnode] (t1) at (0,-0.0) {$\theta_1$};
            \node[leftnode] (t2) at (0, 1.0) {$\theta_2$};
            \node[rightnode] (x) at (2.5, 0) {$\bx$};

            \draw[->] (t0) -- (x);
            \draw[->] (t1) -- (x);
            \draw[->] (t2) -- (x);
        \end{tikzpicture}
        \par\vspace{5pt}
        (B)
    \end{minipage}
    \hfill
    \begin{minipage}[b]{0.23\textwidth}
        \centering
        \begin{tikzpicture}[
            scale=1.0,
            every node/.style={transform shape},
            node distance=0.5cm and 2.5cm, 
            >=Stealth,
            leftnode/.style={circle, draw=oxfordred, fill=oxfordred, text=white, minimum size=7mm, font=\bfseries, inner sep=0pt},
            rightnode/.style={circle, draw=oxfordblue, fill=oxfordblue, text=white, minimum size=7mm, font=\bfseries, inner sep=0pt}
        ]
            \node[leftnode] (t0) at (0,-1.0) {$\theta_0$};
            \node[leftnode] (t1) at (0,-0.0) {$\theta_1$};
            \node[leftnode] (t2) at (0, 1.0) {$\theta_2$};
            \node[rightnode] (x) at (2.5, 0) {$\bx$};

            \draw[->] (t0) -- (x);
            \draw[->] (t1) -- (x);
        \end{tikzpicture}
        \par\vspace{5pt}
        (C)
        \label{fig:environments:springs_c}
    \end{minipage}
    \hfill
    \begin{minipage}[b]{0.23\textwidth}
        \centering
        \begin{tikzpicture}[
            scale=1.0,
            every node/.style={transform shape},
            node distance=0.5cm and 2.5cm, 
            >=Stealth,
            leftnode/.style={circle, draw=oxfordred, fill=oxfordred, text=white, minimum size=7mm, font=\bfseries, inner sep=0pt},
            rightnode/.style={circle, draw=oxfordblue, fill=oxfordblue, text=white, minimum size=7mm, font=\bfseries, inner sep=0pt}
        ]
            \node[leftnode] (t0) at (0,-1.0) {$\theta_0$};
            \node[leftnode] (t1) at (0,0.0) {$\theta_1$};
            \node[rightnode] (x) at (2.5, 0) {$\bx$};

            \draw[->] (t0) -- (x);
            \draw[->] (t1) -- (x);
        \end{tikzpicture}
        \par\vspace{5pt}
        (D)
        \label{fig:environments:springs_d}
    \end{minipage}
    \caption{
    \new{The causal DAGs pertaining to examples A,B,C and D.} 
    }    \label{fig:causal_graph_D2}
\end{figure}
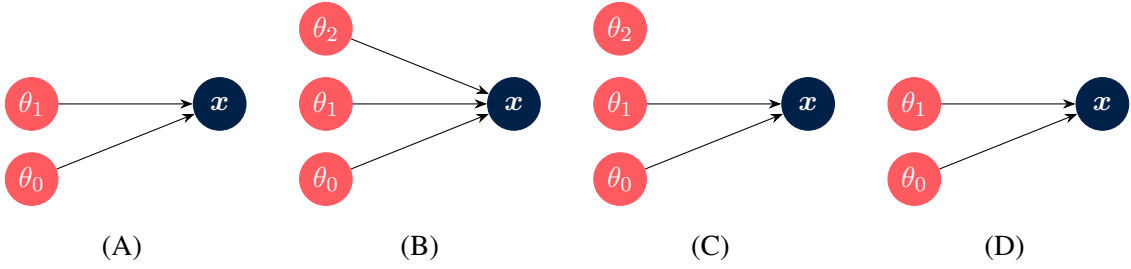

Equation \ref{eq:appendix:D:example:A:2} expresses the corresponding Jacobian of the transition model. As per our definition of causal graphs (non-zero Jacobian at any point), the graph is fully connected. The graph is depicted in figure 7A.  

From the Jacobian, it is evident that the assumption can never be satisfied with this system. Intuitively, because both springs have the same effect upon the system, their combined influence can be aggregated into one parameter. Another implication is that, given any observed trajectory, the ground-truth parameters are not uniquely recoverable. Satisfying the sufficient variability assumption by itself does not guarantee disambiguation of the influence of parameters given any specific trajectory, but it does guarantee that there exist states where the influence can be disambiguated.
\paragraph{Example B. Redundant Causal Edge}
\begin{align}
    \bx_{t+1}&= \bf(\bx_t, \bt) = \frac{\theta_0}{\theta_1}\bx_t + \theta_2
    \\
    \nabla_{\btheta} \bf &=
    \begin{bmatrix}
        \frac{1}{\theta_1}\bx_t & -\frac{\theta_0}{{\theta_1}^2} \bx_t & 1 
    \end{bmatrix}
\end{align}
Example B showcases the same effect in a slightly more complex setting. Many physical effects are functions of the ratios of natural variables. For example, elastic collisions are a function of the ratio of the masses of the colliding objects. In the same setting, using two parameters introduces linearly dependent columns, which violate the assumption. 
\paragraph{Example C. Non-Redundant Causal Edge}
\begin{align}
    \bx_{t+1}&= \bf(\bx_t, \bt) = \theta_0 \bx_t + \theta_1
    \\
    \nabla_{\btheta} \bf &=
    \begin{bmatrix}
        \bx_t & 1 & 0 
    \end{bmatrix}
\end{align}
In contrast, if the redundant parameter is replaced with a single parameter, the first two columns become linearly separable. Note that in this new example, the Jacobian with respect to $\theta_2$ is zero. Consequently, the parameter no longer exerts causal influence on $\bx$ and the assumption no longer requires that the column is linearly independent.

\newpage
\subsubsection{Sufficient Influence Assumption in Related Works}
\label{sec:appendix:D:related}

A notion of sufficient influence (or variability) is a recurring requirement across causal representation learning (CRL) methods, as it is fundamental for resolving ambiguity between causal influences. While this assumption is intuitive in physical systems, where parameters exert identifiable effects on system states,its interpretation becomes less direct in more abstract representation learning settings.

In the temporal CRL setting of \citet{lachapelle2024nonparametric}, the sufficient variability assumption is formalized as a requirement that the family of functions influencing each variable be linearly independent. In their framework, these functions correspond to components of the transition function, which closely parallels our setting, in which identifiability depends on the distinguishability of parameter-induced effects on state transitions. Linear dependence among these functions introduces redundancy, preventing parameters from being uniquely identified.
Recent work on hierarchical temporal CRL \citep{li2025towards}, employs a similar modified sufficient variability assumption.

Similar principles appear outside of temporal CRL. In iVAE, \citep{khemakhem2020variational}, identifiability is achieved by leveraging an observed auxiliary variable that modulates a latent distribution. There, sufficient influence is enforced by requiring independence between the sufficient statistics and the natural parameters of an assumed exponential family. As in our setting, this condition rules out redundant or indistinguishable latent effects that would otherwise break identifiability.

More broadly, related assumptions are commonplace across a wide range of CRL approaches. In interventional CRL, for example, methods typically require interventions to induce sufficiently distinct changes in the latent variables. The work of \citet{von2023nonparametric} formally assumes that interventions is sufficiently distinguishable such that only the correct causal graph can explain a set of observations. CITRIS \citep{lippe2022citris} requires the independence of intervention targets, preventing scenarios in which it is ambiguous which intervention produced a given outcome. Although expressed differently, these conditions serve a common purpose, preventing ambiguity in causal influence.

\newpage

\section{Dataset Details}
\label{appendix:dataset}
\begin{figure}[h]
\centering
\setlength{\tabcolsep}{2pt}%
\begingroup
  \setlength{\abovecaptionskip}{4pt}%
  \setlength{\belowcaptionskip}{0pt}%
  \begin{tabular}{@{}cccc@{}}
    \includegraphics[trim={0cm 0cm 0 1.0cm},clip,width=.24\linewidth]{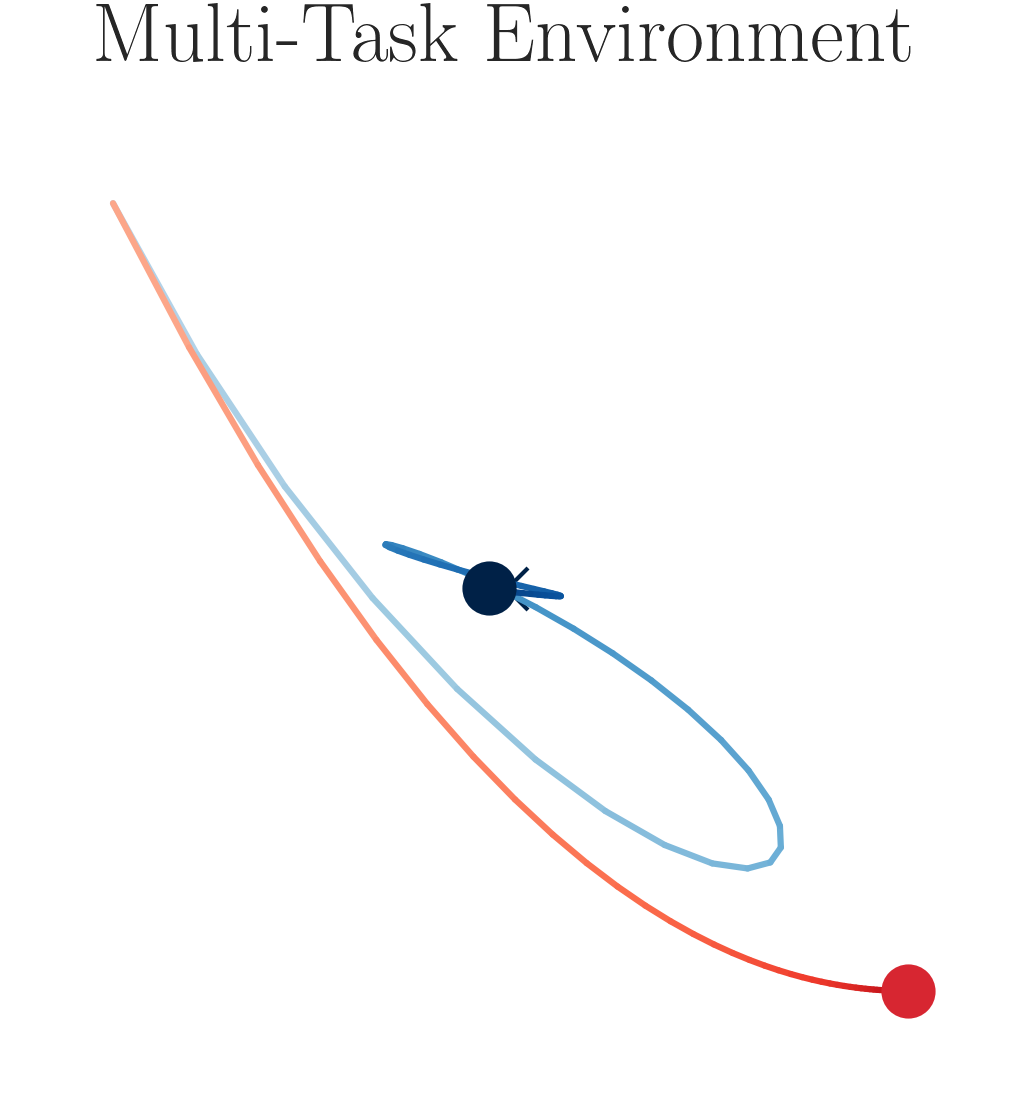} &
    \includegraphics[trim={0cm 0cm 0 1.0cm},clip,width=.24\linewidth]{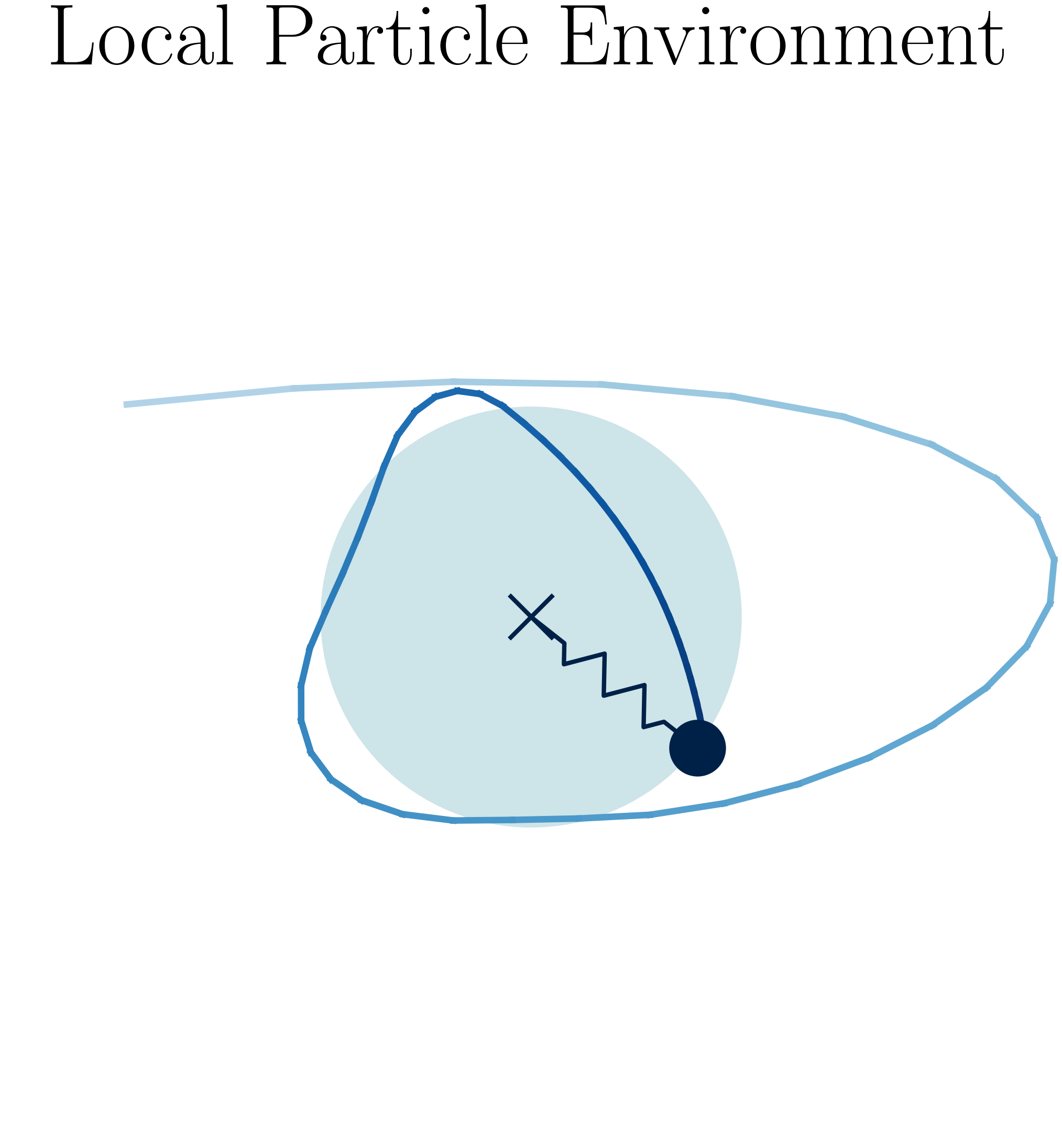} &
    \includegraphics[trim={0cm 0cm 0 1.0cm},clip,width=.24\linewidth]{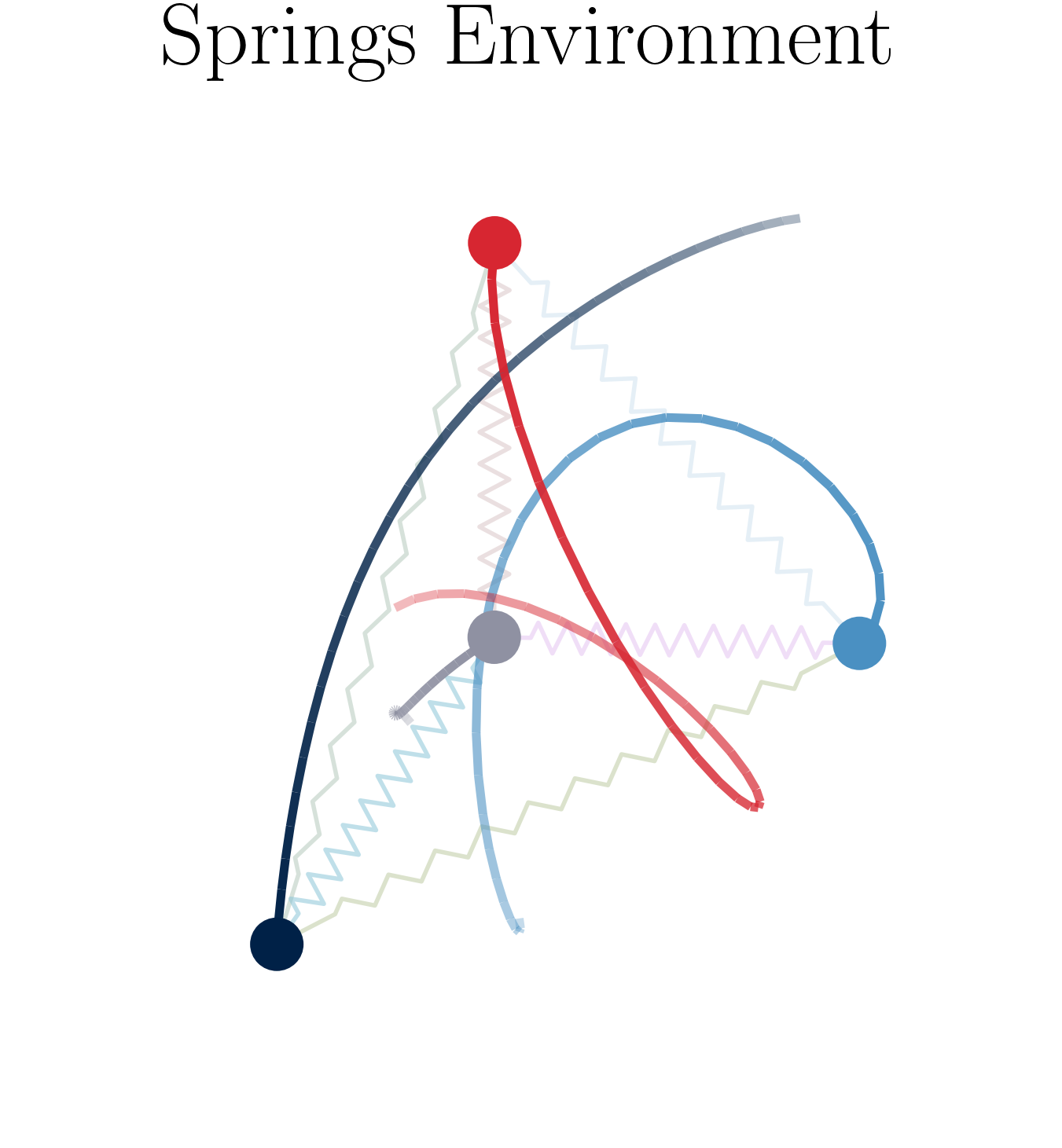} &
    \includegraphics[trim={0cm 0cm 0 1.0cm},clip,width=.24\linewidth]{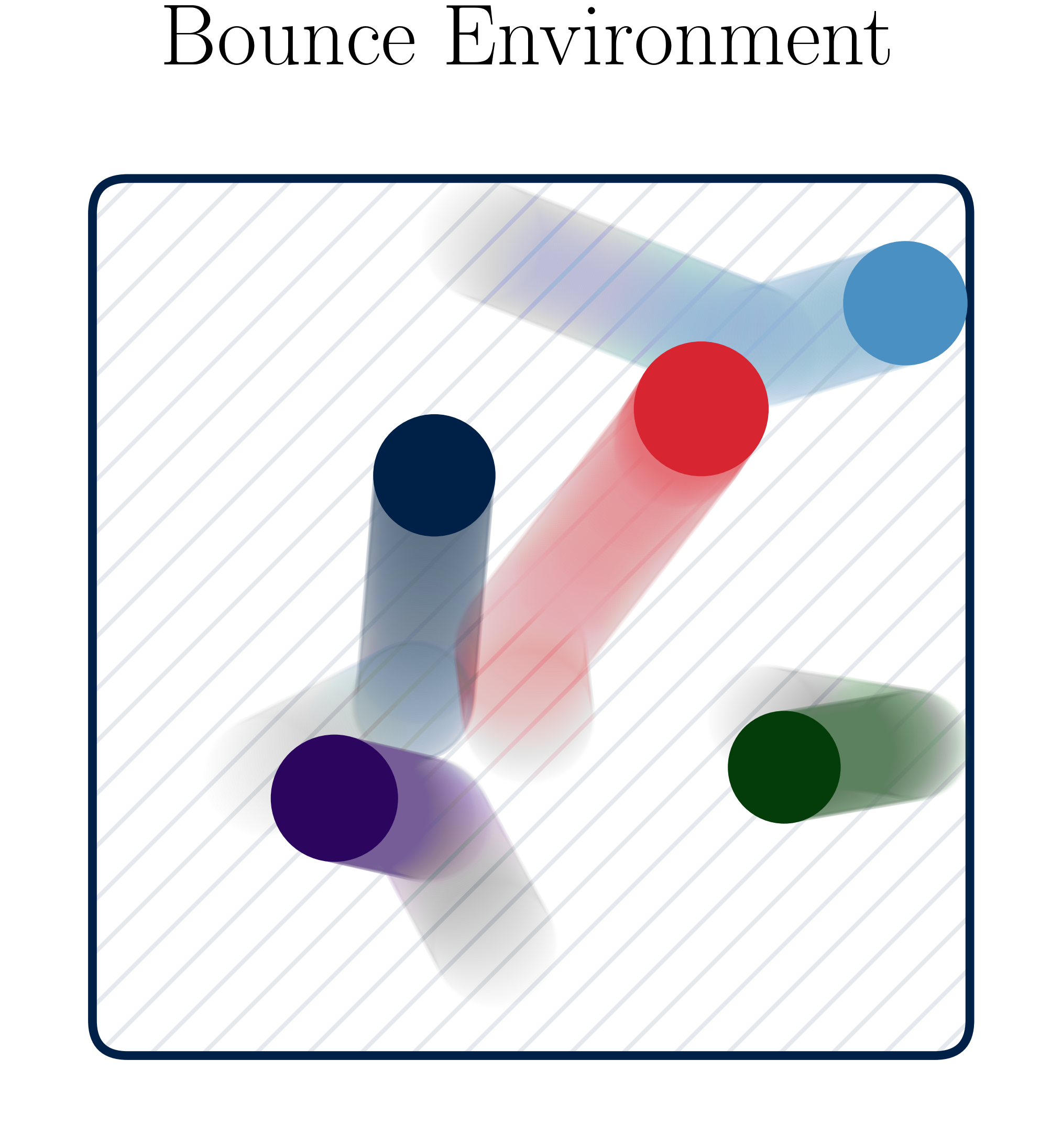} \\
  \end{tabular}
  \caption{Graphical representation of the four evaluation environments, in order from left to right. \textit{Left:} Dual Particle. \textit{Center-Left:} Local Particle. \textit{Center-Right:} Springs. \textit{Right:} Bounce.}
  \label{fig:environments_visualisation}
\endgroup
\end{figure}

\subsubsection*{Dual Particle}
Consider a point-mass that is connected to the origin by a spring (controlled by dynamical parameter, $k$) and under the influence of separate x-axis and y-axis damping (parameters $c_x$, $c_y)$. The state information is taken as a four-dimensional vector containing position and velocity information. As per the graphical criterion derived in section \ref{section:theory}, the dynamical factors are not disentangle-able using the system's causal graph. To achieve disentanglement, we modify the system to perform two counterfactual tasks: general prediction, and then counterfactually what would have happened if the spring were disabled. The resultant DAG is visualised in figure \ref{fig:particles_dag_three} and satisfies the graphical criterion globally. This environment also showcases an interesting edge case in the theory. The path-connectedness condition is not technically satisfied. Consider the environment as visualised in Figure \ref{fig:environments_visualisation}. The red particle represents the counterfactual trajectory with no spring connected. The entire state-space can be partitioned into four, depending upon the sign of the x and y velocity of the non-spring-connected particle. The path-connectedness assumption is violated, as there exists no position in state space or no (positive) damping parameters that would change the sign of the velocity components of the red particle. The assumption being violated is not catastrophic in this scenario, as it only affects the damping parameters, which the blue particle must also reuse, and those are well-behaved relative to the path-connectedness assumption. The result is a representation which is surjective, simply many-to-one. Within the set of experiments, this instantiates itself as a $V$ or inverted-$V$ learnt representation. In practice, we also regularise the state-to-state causal DAG, which isolates the dynamical behaviour of each particle, and forces the learnt representation to be diffeomorphism again. 

\subsubsection*{Local Particle}
The graphical criterion can also be met locally, with state-dependent local causal graphs. The local particle environment implements the same physics as the multi-task particle environment, but disables the spring within a ball around the origin. Consequently, locally within the ball the causal graph is the same as the counterfactual graph of the multi-task environment, and outside the same as the primary task. Interestingly, \textbf{neither} local graph is sufficient to disentangle all parameters by itself, but when used in conjunction they satisfy the local graphical criterion.

\subsubsection*{Springs}
The springs environment contains four-point particles connected by six springs, one for every possible pairing. The causal connections are quantised at the per-object level, meaning each token input to SPARTAN corresponds to an object, not a state dimension. The spring constants are distributed via the absolute value of a standard normal variable, but there is only a seventy per cent probability of a given spring being active at all. The global causal graph for the environment is consequently heavily populated, and can be seen in figure \ref{fig:particles_dag_three}. This environment is an example in which there are more dynamic factors of variation than direct observables, yet the theory guarantees disentanglement with respect to the global graph. Variations of this environment, usually with static spring constants, are commonplace across works in causal discovery and representation learning \citep{ACD,NRI,battaglia2016interaction,li2020causal}.

\subsubsection*{Bounce}
The bounce environment consists of five round balls constrained in a square box, observable through object-centric representations. Each ball can collide elastically with another ball and with the wall. The characteristics of each collision are governed by the mass of each object, which is the dynamical factor of variation, sampled from a non-zero mean normal distribution. For ease of observation, the radius of each ball is also proportional to its mass, creating a dynamic environment that is sensitive to its mass. As with the spring environment, variations of this environment are prevalent testing grounds for works within compositional world models, causal discovery, and representation learning \citep{spartan,battaglia2016interaction,chang2017compositional,jiang2024slot}, and were initially introduced in video form by \citet{van2018relational}. The environment is much more difficult to learn as the causal graphs occur discontinuously only upon collision events. The global graph for this environment is fully connected and provides zero disentanglement guarantees; however, as visualised in figure \ref{fig:particles_dag_three}, there exists a combinatorial cardinality of local causal graphs. Only a small subset of them is required to prove the theory guarantees disentanglement. 

\section{Extended Results}
\label{appendix:extended_results}
\subsection{MCC Metric for Disentanglement}
\label{appendix:mcc_metric}

The MCC disentanglement metric is computed by encoding all trajectories in a validation dataset into the learnt parameters, $\hat{\boldsymbol{\theta}}$, and fitting a small MLP onto every combination of marginalised predictions and marginalised ground truth parameters, $\theta_i \approx \text{MLP}_{ij}(\hat{\theta}_j)$. The predictions from this collection of MLPs are used to form a matrix of non-linear correlation coefficients, $R^2\in \mathbb{R}^{I\times J}$. The MCC metric is calculated as $\text{MCC}= \frac{1}{I}\sum_i^I \max_j(R^2_{i,j})$. The MCC metric is permutation, scale, and shift invariant in representations and accounts for any non-linearities. It does not determine whether a representation is a diffeomorphism, but it does enforce surjectivity. Consequently, an MCC of $\approx 1$ indicates that there exists at least one element in the learnt representation that perfectly specifies, functionally up to continuous, and differentiable surjection, every element in the ground-truth representation. The use of MCC to measure disentanglement is commonplace in causal representation learning literature \citep{lachapelle2022synergies,yao2022temporally,lachapelle2022partial,liidentifying,li2024disentangled,gresele2021independent,khemakhem2020variational}.

The metric is an approximation as exact values depend upon the size and convergence of the MLP used. If not enough data points are used to calculate the mapping, overfitting might also skew the results. Consequently, we opt to use a one-hidden-layer MLP with a hidden dimension of 32 and 5,000 sampled training points, as well as cross-validation with a $10\%$ and $90\%$ split to prevent overfitting and ensure that the model converges quickly.

\subsection{Examples of Learnt Representations}
\label{appendix:representations}

To illustrate the learnt representations, the empirical entanglement mapping can be plotted. We encode over a validation dataset of trajectories all trajectories into pairs of latent vectors and ground-truth dynamical factor vectors. Then in a grid by plotting the marginals we can represent the multidimensional mapping in a way that visually shows disentanglement. On the y axis we depict the ground truth parameters and on the x axis the learnt parameters. Surjective disentanglement (as measured by MCC) is evident when specifying one learnt parameter allows for a precise specification of a ground truth factor of variation. Each diagram is a representative sample from one of the trials in the main results.

\begin{figure}[h]
    \centering
    \includegraphics[width=1.0\linewidth]{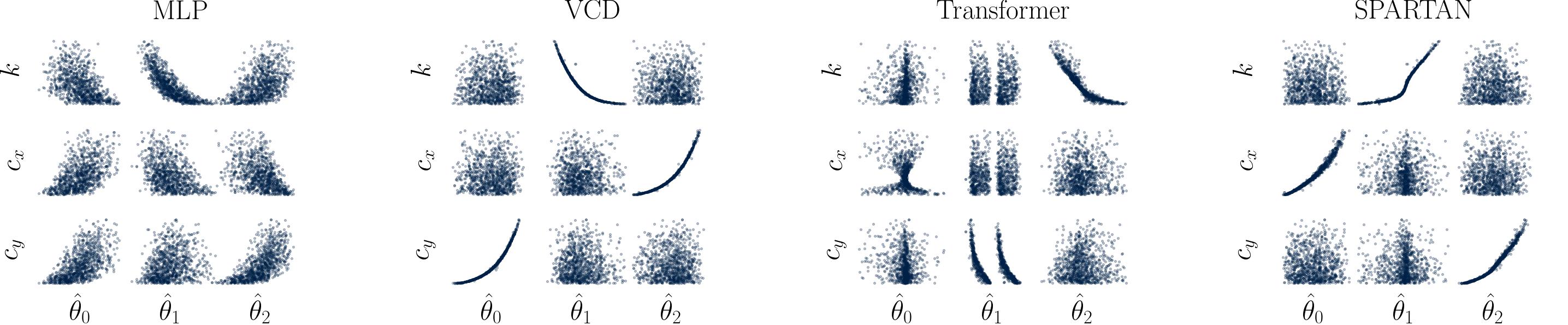}
    \caption{Representative samples of representations learnt on the Dual Particle environment. Each subplot plots the marginal between one ground-truth factor of variation on the vertical axis against a learnt parameter dimension on the horizontal axis. Increasingly sharp one-to-one mappings show disentanglement. The MLP learns a representation where parameters are correlated with the factors of variation but not disentangled. The Transformer learns a non-bijective distribution in which only the spring constant, $k$, exhibits clear diffeomorphic disentanglement. The multi-form nature of the Transformers distribution indicates entanglement with the state. Both SPARTAN and VCD show clear, sharp, and crisp disentanglement.}
    \label{fig:representations_two_particle}
\end{figure}
\begin{figure}[h]
    \centering
    \includegraphics[width=1.0\linewidth]{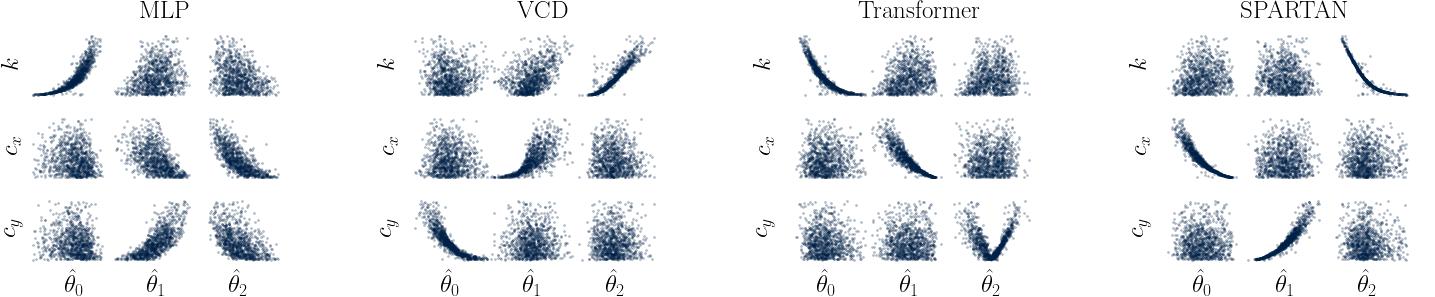}
    \caption{Representative samples of representations learnt on the Local Particle Environment. Each subplot plots the marginal between one ground-truth factor of variation on the vertical axis against a learnt parameter dimension on the horizontal axis. Increasingly sharp one-to-one mappings show disentanglement. Both the MLP and VCD baselines learn representations that are correlated with but not fully disentangled. In this instance, the global graphical criterion predicts partial disentanglement; the spring constant $k$ is guaranteed to be disentangled, but not the damping constants. The VCD model strongly disentangles $k$ and $c_y$, but $c_x$ remains entangled. 
    SPARTAN learns a diffeomorphism mapping as predicted by the theory, whilst a Transformer learns a surjective mapping. Both perform well on the MCC metric.}
    \label{fig:representations_local_particle}
\end{figure}
\begin{figure}[h]
    \centering
    \includegraphics[width=0.75\linewidth]{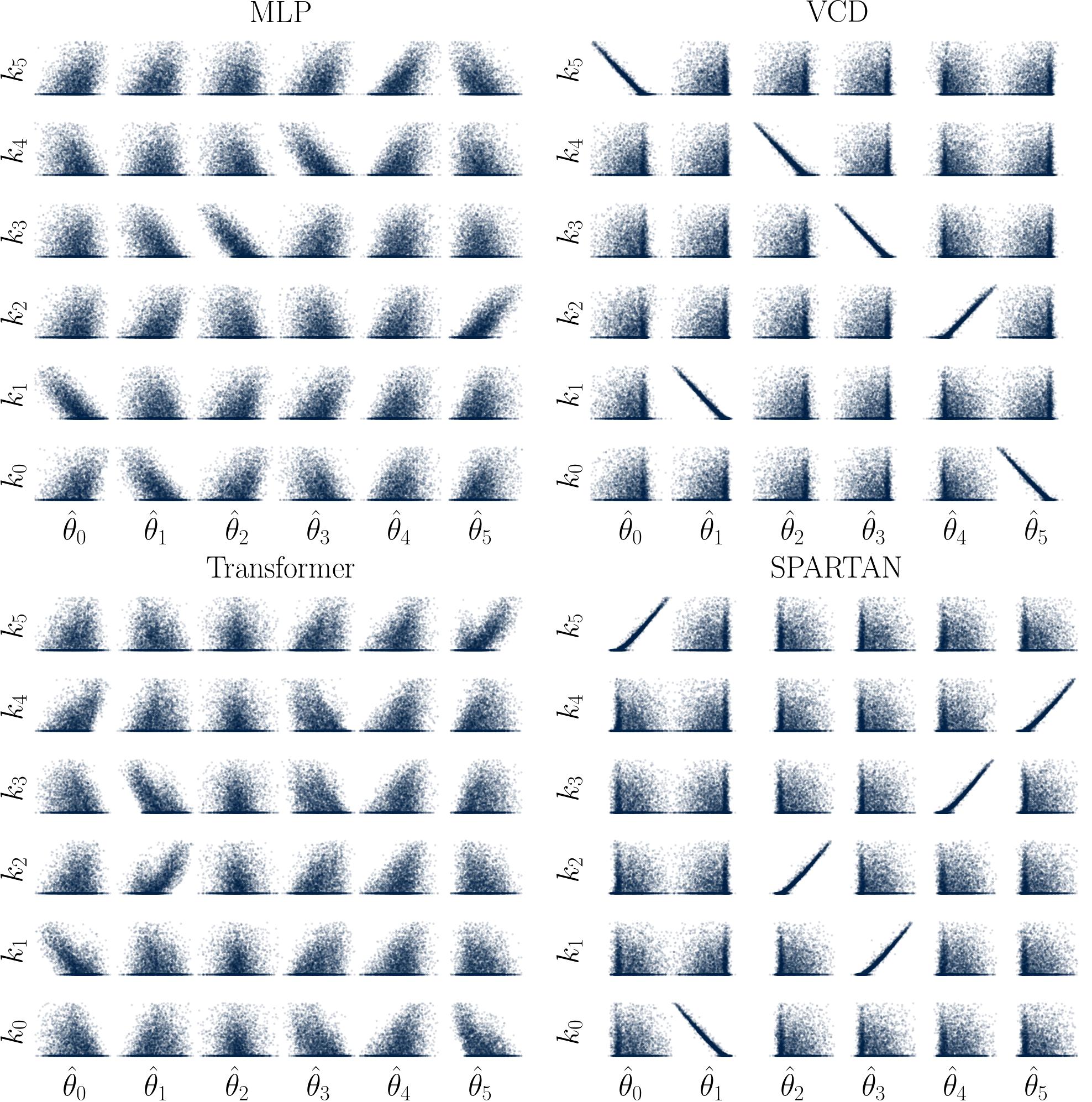}
    \caption{Representative samples of representations learnt in the Springs environment. Each subplot plots the marginal between one ground-truth factor of variation on the vertical axis against a learnt parameter dimension on the horizontal axis. Increasingly sharp diagonals in each subplot showcase disentanglement. VCD and SPARTAN both learn highly disentangled representations with evident structure that is missing from the MLP and Transformer baselines.}
    \label{fig:representations_springs}
\end{figure}
\begin{figure}[h]
    \centering
    \includegraphics[width=0.75\linewidth]{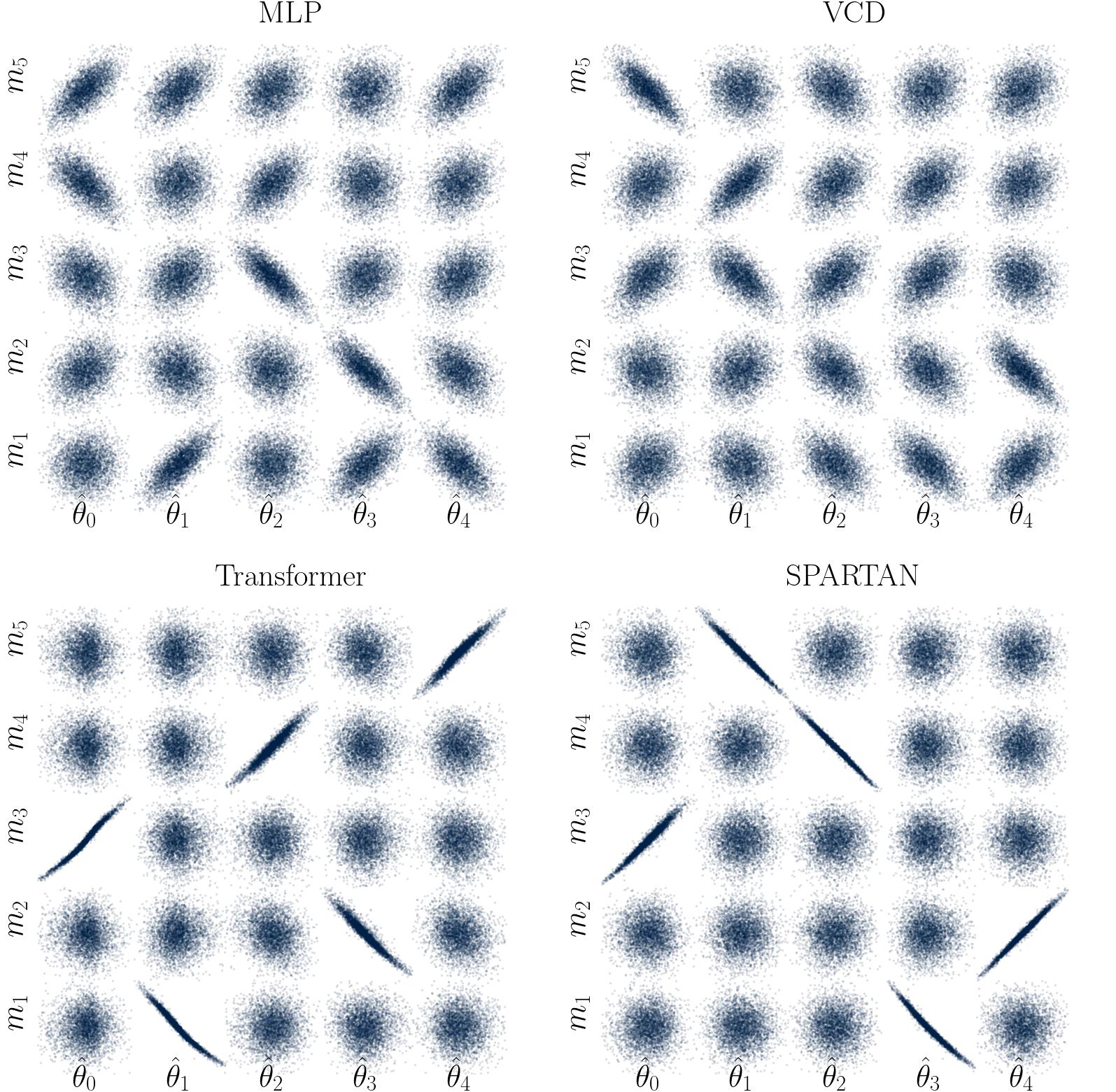}
    \caption{Representative samples of representations learnt in the Bounce environment. Each subplot plots the marginal between one ground-truth factor of variation on the vertical axis against a learnt parameter dimension on the horizontal axis. Increasingly sharp diagonals indicate disentanglement. As this environment requires adherence to local-causal graphs for disentanglement, both the MLP and VCD baselines fail to disentangle. Both SPARTAN and a Transformer selectively attend to parameters, and in this instance, the Transformer's softer bias is sufficient for both to disentangle strongly.}
    \label{fig:representations_bounce}
\end{figure}
\clearpage
\subsection{Examples of Learnt Causal Graphs}
\label{appendix:causal_graphs}
\begin{figure}[h]
    \centering
    \includegraphics[width=0.5\linewidth]{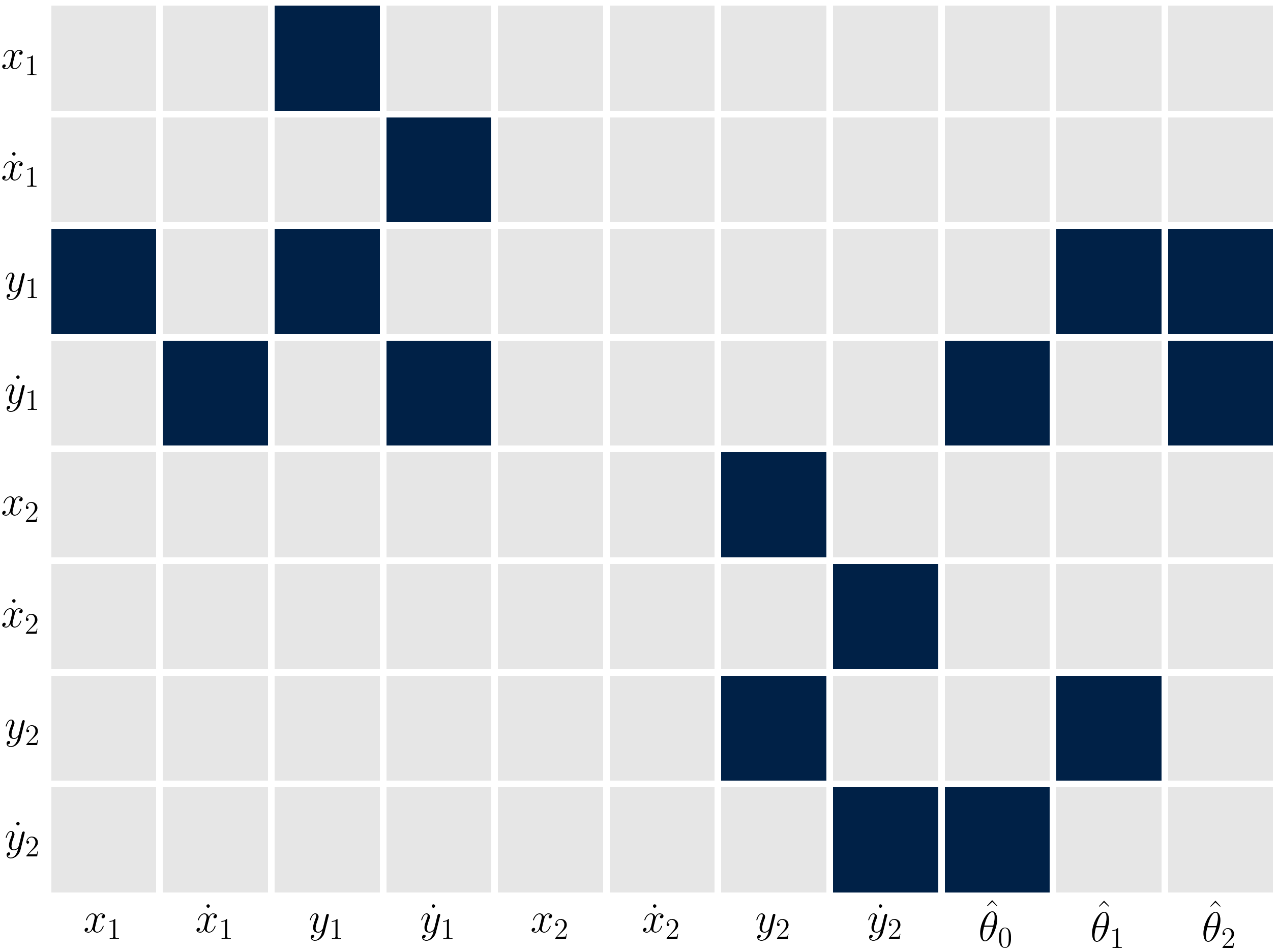}
    \caption{Representative \emph{global} causal graph learnt by the VCD architecture on the dual particle environment. State space tokenisation is used. The parameter influence graph (referred to as $\hat{\mathcal{G}}$ in the theory) is represented by the final three columns. The sparse representation satisfies the global graph criterion for disentanglement. Dark blue squares represent an edge probability of approximately $1$, light grey squares an edge probability of approximately $0$, and colours in between represent values in between.}
    \label{fig:causal_grapg_vcd_dual}
\end{figure}
\begin{figure}[h]
    \centering
    \includegraphics[width=0.5\linewidth]{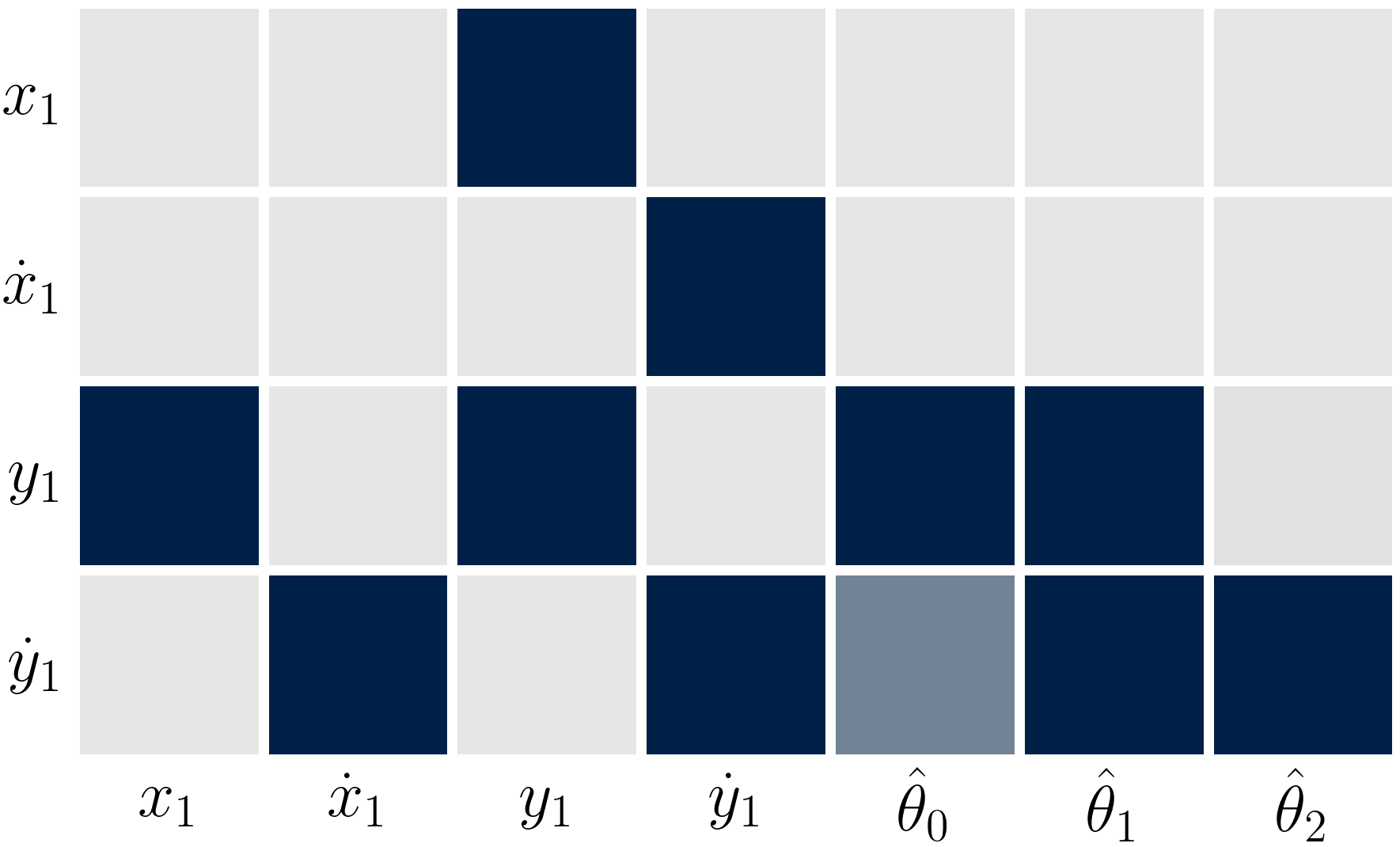}
    \caption{Representative \emph{global} causal graph learnt by the VCD architecture on the local particle environment. State space tokenisation is used. The parameter influence graph (referred to as $\hat{\mathcal{G}}$ in the theory) is represented by the final three columns. In contrast to the local-graphs learnt by SPARTAN, the global graph displayed does not satisfy the disentanglement condition. Dark blue squares represent an edge probability of approximately $1$, light grey squares an edge probability of approximately $0$, and colours in between represent values in between.}
    \label{fig:causal_grapg_vcd_dual}
\end{figure}
\begin{figure}[h]
    \centering
    \includegraphics[width=0.70\linewidth]{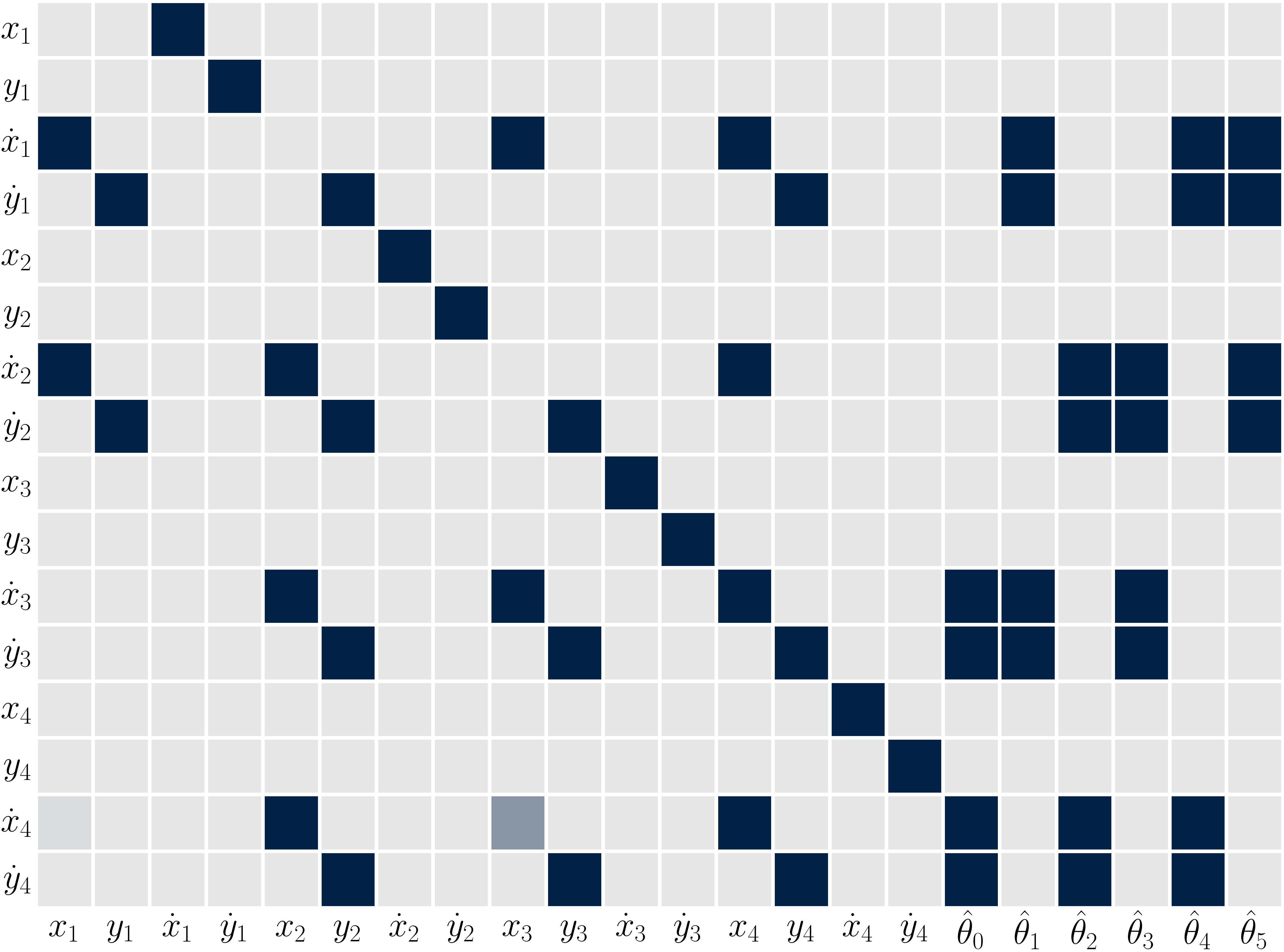}
    \caption{Representative \emph{global} causal graph learnt by the VCD architecture on the springs environment. State space tokenisation is used. The parameter influence graph (referred to as $\hat{\mathcal{G}}$ in the theory) is represented by the final six columns. The sparse representation satisfies the global graph criterion for disentanglement. Dark blue squares represent an edge probability of approximately $1$, light grey squares an edge probability of approximately $0$, and colours in between represent values in between.}
    \label{fig:causal_grapg_vcd_dual}
\end{figure}
\begin{figure}[h]
    \centering
    \includegraphics[width=0.5\linewidth]{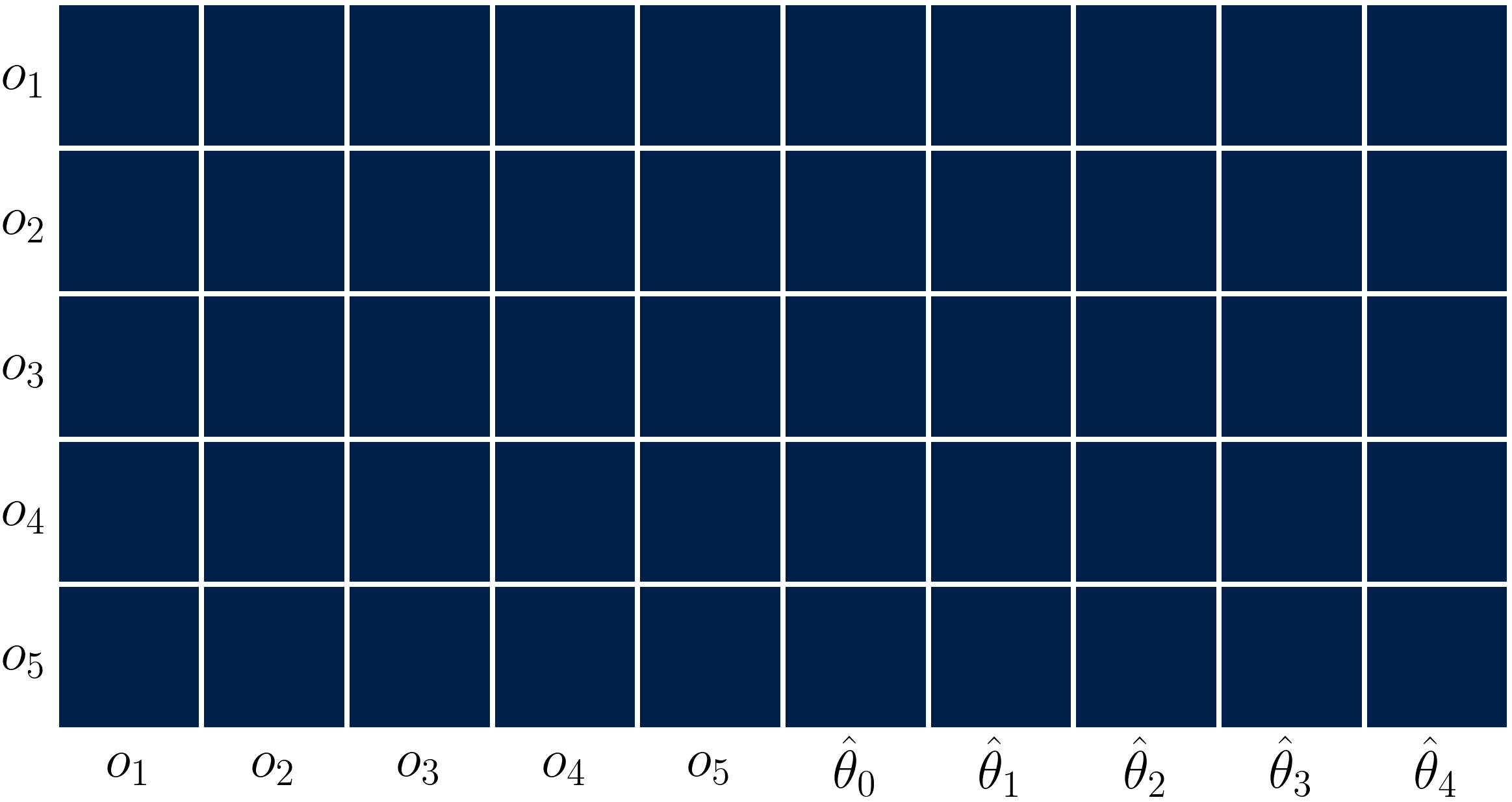}
    \caption{Representative \emph{global} causal graph learnt by the VCD architecture on the bounce environment. Object-level tokenisation is used. The parameter influence graph (referred to as $\hat{\mathcal{G}}$ in the theory) is represented by the final five columns. The graph is fully connected, as any pair of objects and any set of parameters \emph{can} interact at some point in the dataset, even though interactions are very sparse in practice. The \emph{local} sparsity is taken into account by the SPARTAN baseline. Dark blue squares represent an edge probability of approximately $1$, light grey squares an edge probability of approximately $0$, and colours in between represent values in between. 
    }
    \label{fig:causal_grapg_vcd_dual}
\end{figure}

\clearpage

\subsection{Ablation of the Logit Regularisation}

\label{appendix:attention_logit_ablation}
\begin{figure}[tb]
    \centering
    \includegraphics[width=0.8\linewidth]{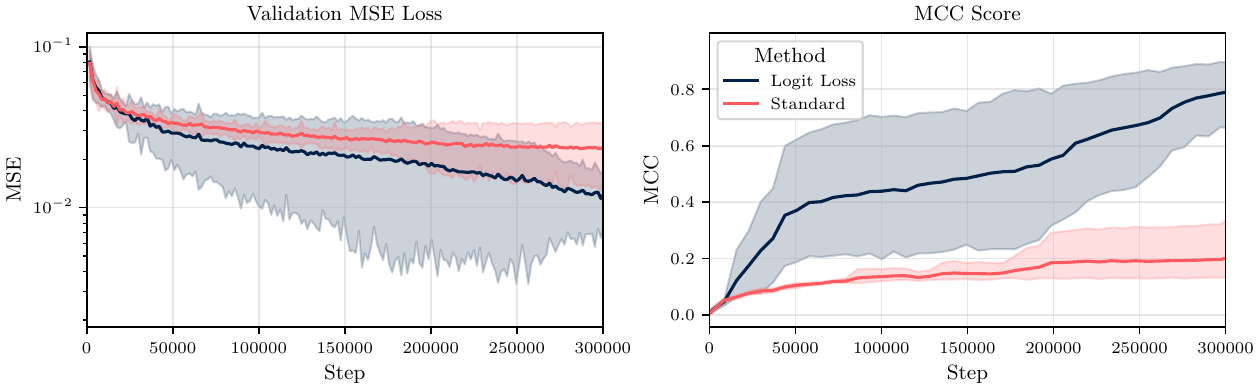}
    \caption{Comparison of the impact that the attention logit regularisation loss has on training a transformer in the bounce environment. Both lines consists of 8 trials completed with different random seeds, and the shaded region represents plus-minus one standard deviation.}
    \label{fig:logitlossregcomparison}
\end{figure}

We demonstrate the necessity of attention logit regularisation through an ablation study within the bounce environment, comparing performance with and without across eight random seeds. Because the encoder is bootstrapped to the decoder, its capacity to encode dynamical behaviour emerges only after the decoder has adequately modelled the system and identified the relevant factors of variation. This results in a two stage information adoption process akin to the dual optimisation procedure used with the SPARTAN and VCD experiments. Without the attention logit regularisation, the decoder struggles to incorporate the new information. We conjecture that this is due to vanishing gradients in the softmax of the transformer. The effect is visualised in figure \ref{fig:logitlossregcomparison}, which shows that the transformer can only disentangle in this setting when introducing the logit loss. We observed in experimental settings that undertuning or removing the logit loss hyperparameter, $\lambda_{\text{logit}}$, in the SPARTAN experimental runs results in the path loss plateauing during training, and consequently, the model does not sparsify sufficiently for representational disentanglement. 
\clearpage
\subsection{Disentanglement Induced by the ELBO Loss.}
\label{appendix:kl_div_disentanglement}

\begin{figure}[tb]
    \centering
    \includegraphics[width=1.0\linewidth]{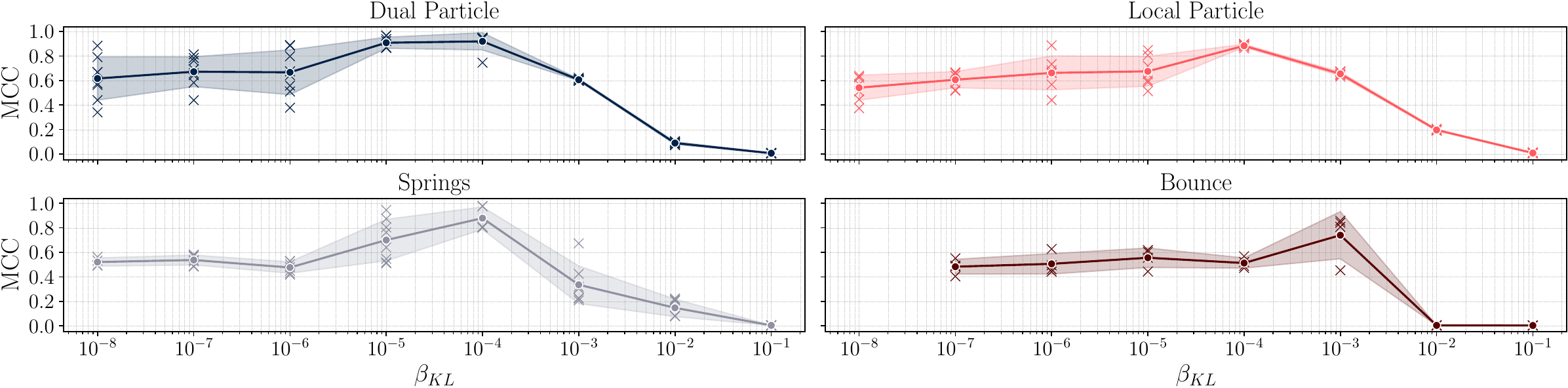}
    \caption{Comparison of the disentanglement achieved by tuning the weighting on the KL regularisation and using an MLP decoder. Each data point represents one trial over a random seed, and the lines represent the mean over all trials at that value. The shaded region covers plus-minus one standard deviation.}
    \label{fig:klbeta_sweep}
\end{figure}
During all experiments, the weighting on the KL regularisation $\beta_{KL}$ is set to $10^{-6}$. This is to ensure that latent regularisation does not contribute to the disentanglement of factors, an issue widely studied in the non-linear ICA and representation learning literature \citep{burgess2018understanding,kim2018disentangling,khemakhem2020variational,chen2018isolating,esmaeili2019structured,mathieu2019disentangling}. The impact of varying the KL regularisation parameter is shown in figure \ref{fig:klbeta_sweep}.

\end{document}